\documentclass[12pt,hidelinks]{article}
\usepackage{xr-hyper}
\makeatletter
\newcommand*{\addFileDependency}[1]{
  \typeout{(#1)}
  \@addtofilelist{#1}
  \IfFileExists{#1}{}{\typeout{No file #1.}}
}
\makeatother

\usepackage{amsfonts,amsmath,float,color,algorithm,amsthm,
  verbatim,subcaption,algorithmic,bm,booktabs,xcolor,enumerate,
  url,cases,indentfirst,graphicx,multicol,hyperref,multirow,amssymb}
\usepackage{tikz}
\usepackage{array}
\usepackage{arydshln}
\usepackage[T1]{fontenc}
\usepackage[letterpaper,top=1in,bottom=1in,left=1in,right=1in]{geometry}

\usepackage{csvsimple}
\usepackage{booktabs}
\graphicspath{{pictures/}}
\allowdisplaybreaks

\newtheorem{lemma}{Lemma}
\newtheorem{corollary}{Corollary}
\newtheorem{theorem}{Theorem}

\newcommand{\bcircle}{\tikz\draw[blue, thick] (0,0) circle (0.15cm);}
\newcommand{\bcircleplus}{
  \tikz{
    \draw[blue, thick] (0,0) circle (0.15cm);
    \draw[blue, thick] (0,-0.15cm) -- (0,0.15cm);
    \draw[blue, thick] (-0.15cm,0) -- (0.15cm,0);
  }
}
\newcommand{\gcircle}{
  \tikz{
    \filldraw[green] (0,0) circle (0.15cm);
    \draw[blue, thick] (0,0) circle (0.15cm);
  }
}
\newcommand{\gcircleplus}{
  \tikz{
    \filldraw[green] (0,0) circle (0.15cm);
    \draw[blue, thick] (0,0) circle (0.15cm);
    \draw[blue, thick] (0,-0.15cm) -- (0,0.15cm);
    \draw[blue, thick] (-0.15cm,0) -- (0.15cm,0);
  }
}

\newcommand{\rw}{\mathrm{rs}}

\newcommand{\dtm}{\mathrm{tg}}
\newcommand{\comdat}{\mathrm{c(d)}}
\newcommand{\comest}{\mathrm{c(e)}}

\newcommand{\maxs}{\vee}
\newcommand{\mins}{\wedge}
\newcommand{\dfr}{\text{df}}

\newcommand{\0}{\bm{0}}

\newcommand{\Z}{\bm{Z}}
\newcommand{\I}{\bm{I}}

\newcommand{\ttheta}{\bm{\theta}_0}

\newcommand{\bSigma}{\bm{\Sigma}}

\newcommand{\A}{\bm{A}}
\newcommand{\x}{\bm{x}}
\newcommand{\tx}{\tilde{\bm{x}}}
\newcommand{\bv}{\bm{v}}
\newcommand{\bmu}{\bm{\mu}}

\newcommand{\bdelta}{\bm{\delta}}
\newcommand{\dd}{\mathrm{d}}

\newcommand{\cS}{\mathcal{T}}
\newcommand{\cB}{\Xi}
\newcommand{\cO}{\mathcal{O}}
\newcommand{\cT}{\mathcal{T}}
\newcommand{\cBs}{\cB^{*}}
\newcommand{\cBd}{\cB_{\dtm}^{*}}
\newcommand{\cBr}{\cB_{\rw}^{*}}

\newcommand{\ys}{\bm{y}_{\cS}}
\newcommand{\yb}{\bm{y}_{\cB}}
\newcommand{\ybs}{\bm{y}_{\cBs}}

\newcommand{\X}{\bm{X}}
\newcommand{\Xs}{\bm{X}_{\cS}}

\newcommand{\Xb}{\bm{X}_{\cB}}
\newcommand{\Xbs}{\bm{X}_{\cBs}}

\newcommand{\Ns}{n_{\cS}}
\newcommand{\Nb}{n_{\cB}}
\newcommand{\Nbs}{n_{\cBs}}

\newcommand{\V}{\bm{V}}

\newcommand{\W}{\bm{W}}
\newcommand{\bH}{\bm{H}}
\newcommand{\diag}{\mathrm{diag}}
\newcommand{\mpi}{c_{\pi}}

\newcommand{\argmin}{\mathop{\arg\min}}

\newcommand{\sumcs}{\sum_{t\in\cS}}
\newcommand{\sumcb}{\sum_{e\in\cB}}

\newcommand{\sumcbs}{\sum_{e\in\cBs}}
\newcommand{\sumcbd}{\sum_{e\in\cBd}}
\newcommand{\sumcbr}{\sum_{e\in\cBr}}

\newcommand{\cL}{\mathcal{L}}
\newcommand{\cLp}{\mathcal{L}^{\pi}}
\newcommand{\cLd}{\mathcal{L}^{\mathrm{D}}}

\newcommand{\bbeta}{\bm{\beta}}
\newcommand{\tbeta}{\bm{\beta}_0}
\newcommand{\hbeta}{\hat{\bbeta}}
\newcommand{\bgamma}{\bm{\gamma}}
\newcommand{\hgamma}{\hat{\bgamma}}

\newcommand{\ttgamma}{\tilde{\gamma}}
\newcommand{\ep}{\bm{\varepsilon}}
\newcommand{\epS}{\bm{\varepsilon}_{\cS}}
\newcommand{\bmeta}{\bm{\eta}}

\newcommand{\rd}{r}

\newcommand{\Ze}{Z^0}
\newcommand{\Zg}{Z^{\gamma}}

\newcommand{\Leps}{L_{\varepsilon}}
\newcommand{\feps}{f_{\varepsilon}}
\newcommand{\fgam}{f_{\gamma}}
\newcommand{\feg}{f_{\varepsilon+\gamma}}
\newcommand{\fetg}{f_{\varepsilon+\ttgamma}}
\newcommand{\ftga}{f_{\ttgamma}}

\newcommand{\geps}{g_{\varepsilon}}
\newcommand{\getg}{g_{\varepsilon+\ttgamma}}

\newcommand{\mineg}{\sigma_{\min}}

\newcommand{\cvsigma}{\varsigma}
\newcommand{\deltaS}{\bm{\delta}_{\cS}}

\newcommand{\bu}{\bm{u}}
\newcommand{\z}{\bm{z}}

\newcommand{\M}{\bm{M}}

\newcommand{\Ds}{\mathcal{D}_{\cS}}

\newcommand{\tp}{^{\top}}

\newcommand{\cvd}{\rightsquigarrow}

\renewcommand{\Pr}{\mathbb{P}}
\newcommand{\Exp}{\mathbb{E}}
\newcommand{\Var}{\mathbb{V}}
\newcommand{\vc}{\mathrm{vec}}

\newcommand{\opt}{{\mathrm{opt}}}

\newcommand{\wtd}{{\textrm{\tiny W}}}%

\newcommand{\vw}{\sigma}
\newcommand{\vuw}{s}
\newcommand{\consp}{c_0}

\newcommand{\hbetaw}{\hbeta_{\rw}}
\newcommand{\hbetad}{\hbeta_{\dtm}}
\newcommand{\gammae}{\gamma}

\newcommand{\Nor}{\mathbb{N}}

\newtheorem{assumption}{Assumption}
\newtheorem{proposition}{Proposition}

\newtheorem{remark}{Remark}

\usepackage{natbib}

\begin{document}
\title{Robust Data Fusion via Subsampling}
\author{Jing Wang, HaiYing Wang, Kun Chen}
\maketitle
\centerline{\it Department of Statistics, University of Connecticut}

\begin{abstract}
  Data fusion and transfer learning are rapidly growing fields that enhance
  model performance for a target population by leveraging other related data
  sources or tasks. The challenges lie in the various potential heterogeneities
  between the target and external data, as well as various practical concerns
  that prevent a naïve data integration. We consider a realistic scenario where
  the target data is limited in size while the external data is large but
  contaminated with outliers; such data contamination, along with other
  computational and operational constraints, necessitates proper selection or
  subsampling of the external data for transfer learning.  To our knowledge,
  transfer learning and subsampling under data contamination have not been
  thoroughly investigated. We address this gap by studying various transfer
  learning methods with subsamples of the external data, accounting for outliers
  deviating from the underlying true model due to arbitrary mean shifts.  Two
  subsampling strategies are investigated: one aimed at reducing biases and the
  other at minimizing variances. Approaches to combine these strategies are also
  introduced to enhance the performance of the estimators. We provide
  non-asymptotic error bounds for the transfer learning estimators, clarifying
  the roles of sample sizes, signal strength, sampling rates, magnitude of
  outliers, and tail behaviors of model error distributions, among other
  factors. Extensive simulations show the superior performance of the proposed
  methods. Additionally, we apply our methods to analyze the risk of hard
  landings in A380 airplanes by utilizing data from other airplane types,
  demonstrating that robust transfer learning can improve estimation efficiency for
  relatively rare airplane types with the help of data from other types of
  airplanes.

  \-
\end{abstract}

\section{Introduction}\label{sec:intro}

Transfer learning and data fusion have gained significant attention in
statistics and machine learning in recent years
\citep{rosenstein2005transfer,chen2015data,li2022transfer}. A typical scenario
for transfer learning is that the target task has limited data but can be
potentially enhanced by a substantial amount of external data from one or many
sources. As such, transfer learning techniques can be particularly beneficial in
fields like engineering and medicine \citep{li2022transfer,jin2024transfer},
where collecting sufficient data for a specific task is costly or data is
collected from disparate sources and thus can be heterogeneous. For example, in
air safety studies, airplane hard landings can be analyzed with Quick Access
Recorder (QAR) data. However, QAR data for certain types of airplanes, such as
the Airbus A380, remains limited \citep{chen2021detection}. Borrowing
information from other airplane types, such as the Airbus A320 and the Boeing
737, can potentially improve the prediction of hard landings for the Airbus
A380. In suicide risk studies with electronic health records (EHRs), transfer
learning and data fusion can be used to build more accurate models
for a local healthcare provider with limited data, when a separate and more
comprehensive EHR or claims database is available \citep{XuChang2022}.

However, external data can be a double-edged sword: naively utilizing all its
information could lead to negative transfer, where the resulting model performs
worse than using only the target data \citep{rosenstein2005transfer}.  This is
mainly due to potential inconsistencies and heterogeneity between the target and
the external sources, and various practical concerns prevent complete data
integration. In the modeling of hard landings, negative transfer could occur if
some flights or airplanes, such as the Boeing 737, exhibit substantially
different landing behaviors compared to typical Airbus A380 flights. In a
suicide risk study, negative transfer may occur if the external source cohort, e.g.,
patients in a statewide database, are very different from the target, e.g.,
patients at a children's hospital.

The effective transfer of knowledge from external sources to the target relies
on exploring and utilizing their connections. One commonly adopted similarity
assumption is that the sources and the target share similar model parameters
\citep{li2022transfer,tian2023transfer,jin2024transfer}, and the difference between the
parameters has to be sufficiently small to ensure a positive transfer. More
recently, an alternative assumption is that the target distribution can be
approximated by a mixture of source distributions without necessarily being
similar to each other
\cite{Hu2018,Mohri2019}. While this mixing assumption is much weaker than the
similarity assumption, it is less intuitive and does not apply when there are
only one or a few large sources. In either of the two approaches, a source is
modeled as a whole without considering within-source heterogeneity, thus
limiting their capability of handling many real-world scenarios.

In this study, we focus on a prototypical scenario where the target data is
limited in size, while the external data is large but contaminated with
outliers. This setup is practical and novel, and neither the aforementioned
similarity assumption nor the mixing assumption may hold. The data
contamination, along with other computational and operational constraints,
necessitates a new strategy for effective transfer learning.

Since anomalies in the external data could lead to significant estimation biases
and variances, we argue that data selection techniques are crucial for ensuring
the success of transfer learning. To achieve this goal, we utilize ideas from
another rapidly developing field, subsampling \citep{Drineas:11,
  ma2015statistical, WangZhuMa2017, WangYangStufken2018, ai2019optimal}, to select
informative data for transfer learning.  To the best of our
knowledge, existing subsampling and transfer learning literature seldom
considers data with outliers. On the other hand, it is a natural scenario that
the target data and a subset of the external source data share the same or similar
underlying models, and the large external dataset as a whole can be quite
heterogeneous.

We address the above gap by studying robust transfer learning methods with
subsamples of the external source data, accounting for outliers deviating from
the underlying true model due to arbitrary mean shifts.  Our learning strategy
is summarized in Figure~\ref{fig:diag}. Two distinct subsampling strategies are
investigated: one aimed at reducing biases, referred to as ``Target-guided
selection'', and the other at minimizing variances, referred to as
``Leverage-based random sampling'' or simply ``random sampling.'' Our thorough
theoretical analysis quantifies the performance of
each transfer learning estimator in terms of bias and variance, as a function of
important structural parameters including sample sizes, magnitudes of outliers,
the sampling rate, tail behaviors of model errors, and others. Furthermore, we
consider methods to combine these strategies to enhance the performance of transfer,
corresponding to the ``Robust data fusion'' step that utilizes both target data
and sampled external data. 

\begin{figure}[H]
  \centering
  \includegraphics[width=\textwidth]{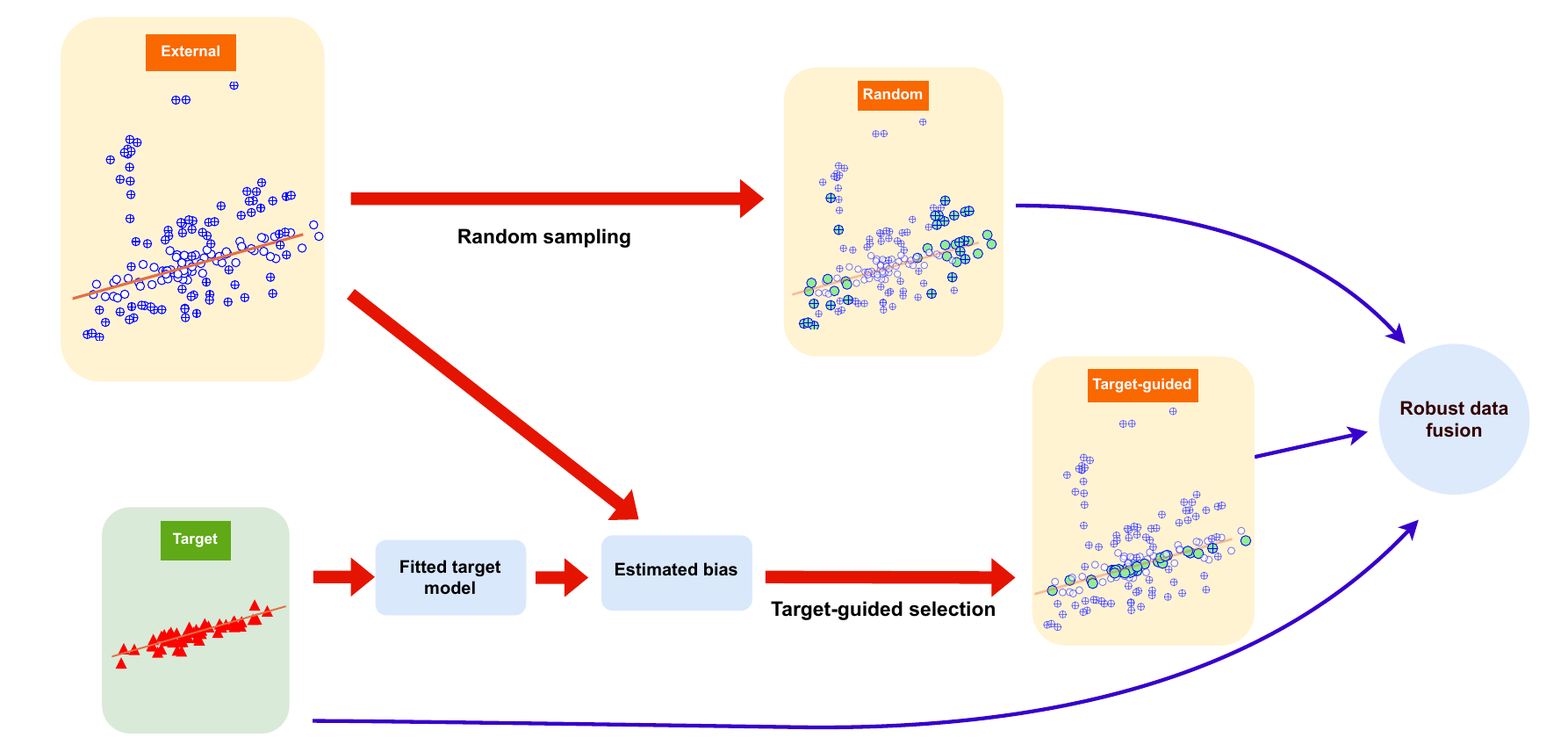}
  \caption{Diagram of robust transfer learning with subsampling.}
  \label{fig:diag}
\end{figure}

The remainder of this paper is structured as follows. Section~\ref{sec:model}
provides an overview of transfer learning and subsampling, along with a formal
definition of the problem setup. In Section~\ref{sec:rtrans}, we discuss random
sampling strategies for transfer learning and present their theoretical
properties. Section~\ref{sec:dtrans} then introduces a target-guided data
selection method, accompanied by theoretical analysis. Computational aspects,
including combined estimators and practical implementation issues, are addressed
in Section~\ref{sec:computation}. Section~\ref{sec:numeric} presents numerical
experiments using both simulated and real-world data to validate our
approach. We conclude the paper in Section~\ref{sec:conclusion}. Detailed proofs
of all the theoretical results are provided in the supplementary material.

\section{Transfer Learning under Source Contamination}\label{sec:model}

\subsection{Overview of Transfer Learning and Subsampling}

We start by providing a brief review of transfer learning and
subsampling. Transfer learning has garnered significant interest in the
statistical community. \cite{chen2015data} proposed a framework for transfer
learning in linear regression involving one target and one external data. They
introduced a fusion estimator that pools the target and external data, adding a
shrinkage penalty on the parameters of external data. \cite{li2022transfer}
investigated high-dimensional
linear regression with multiple external sources. Their algorithms include two
steps: first, train a fusion model with all available data; second, de-bias
using the target data with $\ell_1$ penalization. This methods was extended to
generalized linear models in \cite{tian2023transfer} and high-dimensional
quantile regression in \cite{li2022transfer}. All these works assumed a
similarity condition that
the target and the sources have models with similar parameters. Another more recent and 
potentially weaker assumption is that the target model can be well approximated by a
mixture of the source models, under which group distributional robust optimization
\citep{Hu2018} and agnostic federated learning methods \citep{Mohri2019} have been
developed for transfer learning.

Subsampling is another rapidly growing field. \cite{Drineas:11} proposed to use
uniform sampling to accelerate least-squares approximation. \cite{Ma2014} and
\cite{ma2015statistical} studied leverage score subsampling for linear
regression from the statistical aspect. \cite{WangZhuMa2017} proposed optimal 
subsampling probabilities for
logistic regression, and \cite{ai2019optimal} extended these methods to
generalized linear models. Besides random subsampling, deterministic selection
methods also are major focuses. \cite{WangYangStufken2018} proposed an
information-based optimal subdata selection scheme named IBOSS, using ideas from
optimal design of experiments. \cite{cheng2020information} extended IBOSS to
logistic regression,
and \cite{singh2023subdata} developed deterministic data selection with a large
number of variables.

Despite extensive research on subsampling algorithms, few studies consider
subsampling in the presence of outliers. We aim to address this gap within
the context of regression with potential mean shifts. 
\cite{she2011outlier} and \cite{SheChen2017} investigated robust regression with
outlier detection using penalized mean-shift
models, connecting mean-shift penalization to robust regression,
particularly Huber regression \citep{huber1983minimax} when using $\ell_1$
penalization. Theoretical properties of mean-shift models
were also studied in \cite{beyhum2020inference}. From the perspective of robust
regression, \cite{sun2020adaptive} provided theoretical results for Huber
regression, which is equivalent to linear regression with mean shifts and $\ell_1$
penalization.

\subsection{Problem Setup \& General Form of the Estimators}

We begin with defining some notations. We use $\|\bv\|_p$ and $\|\A\|_p$ to represent the
$p$-norm of a vector $\bv$ and a matrix $\A$, respectively. For a fixed matrix $\A$, we
use notation $\|\bv\|_{\A,2}=\sqrt{\bv\tp\A\bv}$ to represent the 2-norm according to
a matrix $\A$. %
We define $a\maxs b=\max\{a,b\}$ and $a\mins b=\min\{a,b\}$. For any two sequences $a_n$ and $b_n$, we use
$a_n=O(b_n)$ and $a_n\lesssim b_n$ to denote $|a_n/b_n|\le C$ for some positive constant
$C$ when $n$ is large enough; we use $a_n=o(b_n)$ to denote $|a_n/b_n|\to 0$ as
$n\to\infty$.

We consider the scenario that we have data from a target and an external
source. The target data, denoted as $\cS$, is with a relatively small sample size $\Ns$ from the population of interest; the external source
data, denoted as $\cB$, is with a relatively much larger sample size $\Nb$ from a relevant yet
different population, thus subjecting to data contamination with respect to the target data. 
Our main goal is properly utilize the source data to improve the estimation and
prediction for the target population. With some abuse of notations, we also use
$\cS = \{1,\ldots, \Ns\}$ to represent the target data index set and
$\cB=\{\Ns+1,\ldots, \Ns+\Nb\}$ to represent the external data index set. 

To focus on the main idea, consider a linear regression model on the target: %
\begin{equation}
y_t=\x_t\tp\tbeta+\varepsilon_t,\qquad t\in\cS,\label{eq:targetmodel}
\end{equation}
where $y_t$ is the response, $\x_t \in \mathbb{R}^{d}$ is the covariate vector, 
$\tbeta\in\mathbb{R}^d$ is the true coefficient vector of interest, and
$\varepsilon_t$'s are independent and identically distributed
(i.i.d.) random errors with mean zero. For the ``contaminated'' source, we
assume a mean-shift model, 
\begin{equation}\label{eq:sourcemodel}
y_e=\x_e\tp\tbeta+\gamma_e+\varepsilon_e,\qquad e\in\cB,
\end{equation}
where $\tbeta$ remains the same as in~\eqref{eq:targetmodel}, $\gamma_e$'s are
the subject-specific mean-shift parameters, and $\varepsilon_e$'s are i.i.d. random
errors with mean zero.
We introduce some notations to facilitate
the presentation. Let $\Xs=(\x_1\tp, \x_2\tp, \ldots, \x_t\tp, \ldots)\tp$ and
$\ys= (y_1, y_2, \ldots, y_t, \ldots)\tp$ for $t\in\cS$.  For the external data,
let
$\Xbs=( \ldots, {\x_e}\tp, \ldots)\tp$, $\ybs= (\ldots, y_e, \ldots)\tp$,
$\bgamma_{\cBs}=(\ldots, \gamma_e, \ldots)\tp$, and $\ep_{\cBs}=(\ldots, \varepsilon_e,
\ldots)\tp$, for $e\in\cBs$, where $\cBs$ is a subset of index set $\cB$. A
special case is when $\cBs=\cB$ where the matrices and vectors above represent
the full external data matrices and vectors.

A key aspect of the above setup is that the potential similarity and discrepancy
between the target model and the source model are characterized by the unknown
mean-shift terms ($\gamma_e$). The target and the source are more similar when,
for example, the true values of $\gamma_e$ are sparse or of small magnitude, and
vise versa. This is drastically different from the typical ``similarity
assumption'' in transfer learning where the target and the source are assumed to
have exactly the same model structure with similar yet ``controllably'' different
regression coefficient vectors. Our setup is much more realistic and flexible in that the
connection between the target and the source is reflected at the individual
observation level. This allows the presence of arbitrary anomalies in the
source that may easily violate the similarity assumption. The potential value of
the source data can be flexibly controlled by some structural or distributional
assumptions on the mean-shift parameters; we defer detailed technical
discussions to Sections~\ref{sec:rtrans} and~\ref{sec:dtrans}.

It is clear that naively using all data to fit the regression model
in~\eqref{eq:targetmodel} can fail miserably: while observations in the source
with none or small biases (as reflected in the magnitude of the mean-shift parameters
$\gamma_e$) can help reduce the estimation variance of $\tbeta$, observations with large
biases may lead to a substantial estimation bias that overrides any potential gain.

To enable transfer learning, we propose to consider subsampling of the source
data, coupled with a weighted and regularized fusion learning with combined
data. Specifically, we will investigate several different strategies of
sampling/selecting a subset of $\cB$, denoted as $\cBs$, that could potentially lead to a
better bias-variance trade-off and thus improve transfer learning.  With the
target data and the selected source data, we will construct sampling weights to
account for subsampling and adopt regularized estimation to alleviate the impact
of the potential mean shifts of the selected source data.

As such, the transfer
learning estimators of our interest take the following general form: 
\begin{equation}\label{eq:ftrans}
(\hbeta\tp,\hgamma_{\cBs}\tp)\tp=\arg\min_{\bbeta,\bgamma_{\cBs}}
\left\{\sumcs(y_t-\x_t\tp\bbeta)^2+\sumcbs w_e(y_e-\x_e\tp\bbeta-\gamma_e)^2
+\lambda\sum_{e\in\cBs} w_e P(\gamma_e)\right\},
\end{equation}
where $P(\cdot)$ is an even penalty
function and $\lambda$ is a tuning parameter. A special case
of~\eqref{eq:ftrans} is when $\cBs=\cB$ and $w_e=1,e\in\cB$; this estimator
fuses all the external source data with the target data, so we
refer it as the full-data transfer learning estimator.

The idea of mean-shift regularization has been adopted for outlier
detection and robust regression with a single data set \citep[e.g.,][]{she2011outlier}.
It has been shown that the resulting estimator is equivalent to a class
of robust estimators, including Huber regression
\citep{she2011outlier}. The following proposition shows that this connection remains for the transfer learning
estimators in~\eqref{eq:ftrans}. 

\begin{proposition}\label{pro:thr}
For any $\lambda>0$
and penalty function $P(\cdot)$, the $\hbeta$
in~\eqref{eq:ftrans} is the minimizer of 
\begin{equation}\label{eq:fhuber}
\sumcs(y_t-\x_t\tp\bbeta)^2
+\sumcbs w_e\mathcal{H}\left(y_e-\x_e\tp\bbeta;\lambda\right),
\end{equation}
where $\mathcal{H}(z;\lambda)=2\int_0^z\psi(u;\lambda)\dd
u=2\int_0^z\left\{u-\Theta(u;\lambda)\right\}\dd u$, and
$\Theta(z;\lambda)=\argmin_{\gammae\in\mathbb{R}}\{(z-\gamma)^2+\lambda P(\gamma)\}$.
\end{proposition}

Proposition~\ref{pro:thr} extends Proposition 4.1 of~\cite{she2011outlier} to a
more general scenario; the latter is a special
case when there is no target data set, $\cBs=\cB$, and $w_e^{*}=1$ for all
$e\in\cB$. The result shows that $\hbeta$
defined in~\eqref{eq:ftrans} can be obtained by
minimizing a robust loss function that involves
only $\bbeta$. For example, if $P(\gamma)=2|\gamma|$, the $\ell_1$ penalty, then
$\Theta(z;\lambda)=(z-\lambda)I(z>\lambda)+(z+\lambda)I(z<-\lambda)$ and
$\mathcal{H}(z;\lambda)=z^2I(|z|\le\lambda)+(2\lambda|z|-\lambda^2)I(|z|>\lambda)$. The
resulting $\Theta(z;\lambda)$ is the soft thresholding function and  $\mathcal{H}(z;\lambda)$ is Huber's loss widely
used in robust statistics. Another important special case is
$P(\gamma)=\gamma^2$, the $\ell_2$ penalty, for which
 $\Theta(z;\lambda)=z/(1+\lambda)$ and $\mathcal{H}(z;\lambda)=\lambda
z^2/(1+\lambda)$. This is related to shrinkage estimator and ridge
regression. Although Proposition~\ref{pro:thr} holds for general penalty
functions, we focus on the fundamental $\ell_1$ and $\ell_2$ penalties.
Specifically, we specify
$$
P(\gamma)=2\nu^{-1}|\gamma|^{\nu},
$$
where $\nu=1$ or $2$ corresponds to the $\ell_1$ or $\ell_2$ penalty,
respectively.

\subsection{Subsampling Strategies}

With the form of the transfer learning estimators defined in~\eqref{eq:ftrans}, the main 
puzzle is to determine $\cBs$, i.e., selecting data points from the external 
source $\cB$ to balance the bias-variance trade-off. Figure~\ref{fig:exam} illustrates 
three approaches of data selection: (a) leverage-based random sampling, (b) target-guided selection, 
and (c) the optimal subsampling method under the A-optimality criterion
(OSMAC) \citep{ai2019optimal}. 

\begin{figure}[H]
  \centering
  \begin{subfigure}{0.325\textwidth}
    \includegraphics[width=\textwidth]{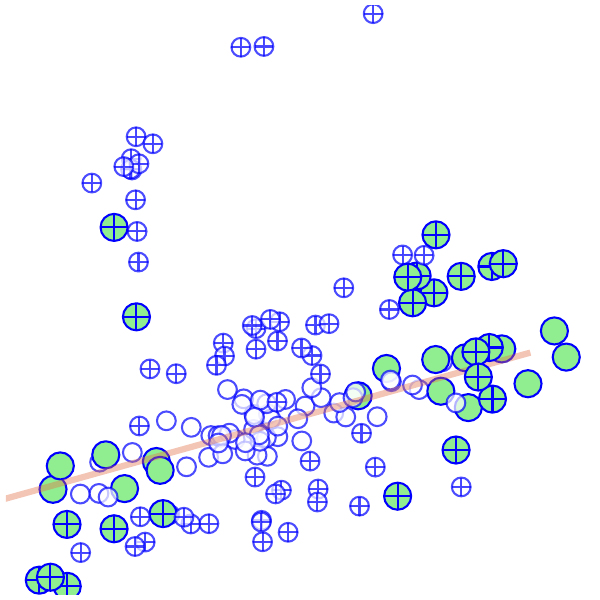}
    \caption{Leverage-based selection}
  \end{subfigure}
  \begin{subfigure}{0.325\textwidth}
    \includegraphics[width=\textwidth]{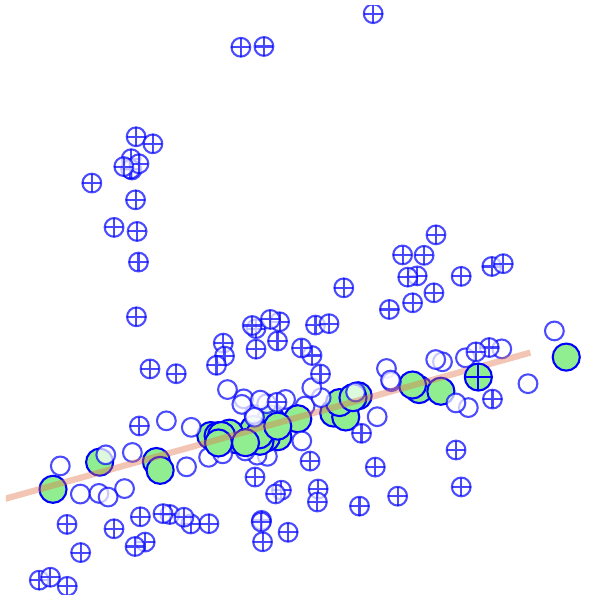}
    \caption{Target-guided selection}
  \end{subfigure}
  \begin{subfigure}{0.325\textwidth}
    \includegraphics[width=\textwidth]{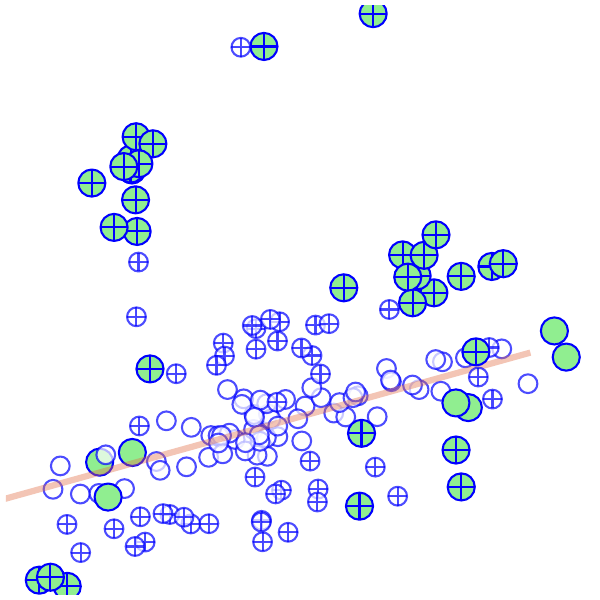}
    \caption{OSMAC selection}
  \end{subfigure}
  \caption{Illustration of data selection strategies. Unbiased observations ($\gamma=0$)
    that do not deviate away from the target data are represented by hollow
    circles (\protect\bcircle, \protect\gcircle) scattering around the true
    regression line (the red reference line). Biased observations
    ($\gamma\neq 0$) that deviate away from the target data are represented by
    the circles filled with ``+'' (\protect\bcircleplus, \protect\gcircleplus).
    Selected data points are colored in green (\protect\gcircle,
    \protect\gcircleplus).}
  \label{fig:exam}
\end{figure}
The leverage-based random sampling \citep{MaSun:2015,Ma+EtAl:2015}, as shown in
Figure~\ref{fig:exam} (a), relies
only on the covariates to construct subsampling probabilities, focusing on selecting 
the most influential data points to reduce estimation variance. However, in the presence
of mean-shift outliers, leverage-based sampling may select biased points%
; as such, it alone may not be effective in reducing bias for transfer learning.

The target-guided selection, as illustrated in Figure~\ref{fig:exam} (b), aims
to select source data with less biases utilizing the target data. Specifically,
a natural estimator of $\gamma_e$, the amount of mean shift or bias in a source
data point, is $y_e-\x_e\tp\hbeta_{\cS}$, where $\hbeta_{\cS}$ is the
least-squares estimator of $\tbeta$ from fitting the target data. The procedure
then uses the order statistics of $|y_e-\x_e\tp\hbeta_{\cS}|$ to select data
points with smaller biases.  However, the performance of this method relies on
the quality of the target-only estimator $\hbeta_{\cS}$, which may have a large
variance due to the small sample size of the target data. As such, the resulting
subsamples may not be effective in reducing the estimation variance.

Optimal subsampling is a widely used method to reduce the asymptotic
variance of the subsampling estimator \citep[see][]{Ma+EtAl:2015, WangZhuMa2017,
ai2019optimal}. Figure~\ref{fig:exam} (c) illustrates the selection by OSMAC
\citep{ai2019optimal}. Optimal subsampling probabilities are constructed under
the assumption of no outliers, and they are approximately proportional to
$|y_e-\x_e\tp\tbeta|$. This suggests that the method favors the selection of the
outliers and may jeopardize transfer learning. We have conducted a simulation
study to further illustrate this; see Section~\ref{sec:toy} in the supplementary material. The
results show that OSMAC performs even worse than uniform sampling estimator and
the target-only estimator.

It now becomes intuitive that while optimal subsampling such as OSMAC may fail 
miserably
with contaminated data, both random sampling and target-based selection are subject to a
bias-variance tradeoff. A thorough investigation of their behaviors sheds light on the
best strategies of conducting transfer learning with a contaminated and heterogeneous source.
In the sequel,
we shall thoroughly investigate the two general strategies, namely, random sampling and
target-guided selection, as well as their adaptive integration, in conjunction with
transfer learning.

\section{Transfer Learning with Random Sampling}\label{sec:rtrans}

In this section, we study random sampling for
transfer learning. We describe the random sampling procedure in
Section~\ref{sec:rand} and present the theoretical property of the transfer
learning estimator based on the selected data. We then discuss optimal sampling
probabilities, leveraging the insights from our theoretical analysis in
Section~\ref{sec:optpi}.

\subsection{Poisson Subsampling and the Resulting Estimator}
\label{sec:rand}
Non-uniform random sampling is widely used in the subsampling literature to
reduce variance, as it assigns different sampling probabilities to individual
data points, allowing for the selection of a more informative subset. In
alignment with this approach, we adopt non-uniform Poisson sampling, which is
frequently employed in subsampling due to its higher statistical efficiency
compared to sampling with replacement \citep[see][]{wang2019more,
wang2022sampling}.
Let our objective be to sample a nominal size of $r = \rho \Nb$ data points from
the set $\cB$ based on sampling probabilities $\{\pi_e\}_{e=1}^{\Nb}$, where $0
< \pi_e \leq 1$ and $\sumcb\pi_e = r$. The Poisson sampling algorithm proceeds
as follows:

\begin{algorithm}[H]%
  \caption{Poisson subsampling}
  \label{alg:pois}
  \begin{algorithmic}[1]
    \STATE For $e=1,...,\Nb$:
    \STATE Generate $u_e\sim U[0,1]$; 
    \IF {$u_e\le \pi_e$}
    \STATE include $(y_e,\x_e)$ and record $\pi_e$ in the subsample;
    \ENDIF
  \end{algorithmic}
\end{algorithm}

The actual subsample size $r^{*}$ resulted from Algorithm~\ref{alg:pois} is
random with $\Exp(r^{*})=r$. Denote the index set $\cBr=\{e\in\cB:\delta_e=1\}$,
where $\delta_e=I(u_e\le \pi_e)$. We propose a transfer learning estimator
through inverse probability weighting (IPW) based on $\cBr$ by taking $\cBs=\cBr$,
$w_e=\rho\pi_e^{-1}$ and $P(\gamma_e)=2\nu^{-1}|\gamma_e|^{\nu}$
in~\eqref{eq:ftrans}, which has the following form:  
\begin{align}\label{eq:rtrans}
(\hbetaw\tp,\hgamma_{\cBr}\tp)\tp
&=\argmin_{\bbeta,\bgamma_{\cBr}}
\left\{\sumcs(y_t-\x_t\tp\bbeta)^2
+\sum_{e\in\cBr}\frac{\rho(y_e-\x_e\tp\bbeta-\gamma_e)^2}{\pi_e}
+\lambda\sum_{e\in\cBr}\frac{2\rho|\gamma_e|^{\nu}}{\nu\pi_e}\right\}.
\end{align}
The IPW is a widely used technique in
subsampling literature \citep[e.g.,][]{ai2019optimal,wang2022sampling}. Here, we
also use weighted penalties for robust estimation. From
Proposition~\ref{pro:thr}, we know that $\hbetaw$ in~\eqref{eq:rtrans}
is the minimizer of
\begin{equation*}
\sumcs(y_t-\x_t\tp\bbeta)^2+\sum_{e\in\cBr}\frac{
\rho\mathcal{H}(y_e-\x_e\tp\bbeta;\lambda)}{\pi_e},
\end{equation*}
where $\mathcal{H}(z;\lambda)=z^2I(|t|\le\lambda)
+(2\lambda|z|-\lambda^2)I(|z|>\lambda)$ for the $\ell_1$ penalty and
$\mathcal{H}(z;\lambda)=\lambda z^2/(1+\lambda)$ for the $\ell_2$ penalty. Therefore, our IPW transfer learning estimator
is the minimizer of an objective function that 
combines a least-squares
objective function on $\cS$ and a weighted robust objective function on $\cBr$. %

\subsection{Theoretical Analysis}\label{sec:rtheory}
In this section, we investigate the theoretical aspects of random sampling, starting with 
presenting some general assumptions. %
\begin{assumption}\label{asm:xsubg}
The covariate $\x$ in both $\cS$ and
$\cB$ satisfy that $\Pr(|\bu\tp\tx|\geq z)\le 2e^{-z^2\|\bu\|_2^2/C_0^2}$ for all
$z\in\mathbb{R}$ and $\bu\in\mathbb{R}^d$, where $\tx=\bSigma^{-\frac{1}{2}}\x$,
$\bSigma=\Exp(\x\x\tp)$, and $C_0>0$ is a constant.
\end{assumption}
\begin{assumption}\label{asm:eps}
The error term $\varepsilon_t,t\in\cS$ are i.i.d. satisfying
$\Exp(\varepsilon_t)=0$, $\Var(\varepsilon_t)=\sigma_{\cS}^2$, and
$\Exp(|\varepsilon_t|^3)<\infty$; and the error term $\varepsilon_e,e\in\cB$ are
i.i.d. with a symmetric density $\feps(x)$, satisfying
$\Exp(|\varepsilon_e|^{1+\alpha_1})<\infty$ for some $0<\alpha_1\le 1$.    
\end{assumption}
\begin{assumption}\label{asm:bpi}
  The subsampling probabilities satisfy that
$\min_{e\in\cB}\pi_e\geq\mpi\rho$ for some constant $\mpi>0$.
\end{assumption}

Assumption~\ref{asm:xsubg} requires that the covariate distribution is
sub-Gaussian. A similar assumption is also used
in~\cite{sun2020adaptive} for theoretical analysis of Huber regression. In our
framework, we allow the covariates in the target data and external data to
follow different distributions.
Assumption~\ref{asm:eps} permits different error distributions for the target and external
data. When $\alpha_1=1$, the error terms $\varepsilon_e$ possess the
second-order moment, which forms a standard condition for establishing
asymptotic results for least-squares estimators. In contrast, when $\alpha_1<1$,
the error terms may not have a finite second-order moment, implying that
$\varepsilon_e$ exhibit heavy-tailed behavior, in contrast to the random errors
in $\cS$. This allows the external data to have larger variation than the
  source data even when there are no outliers.
Assumption~\ref{asm:bpi} is widely adopted in subsampling literature 
\citep[e.g.,][]{yu2020optimal, wang2022sampling, ZhangZuoWangSun2024}.

We first define some notation to ease the presentation on non-asymptotic bounds
of $\hbetaw$:
\begin{align}
&\vw_{\varepsilon,\alpha_1}^{\wtd}=
  \frac{1}{\Nb}\sumcb\Exp\left(\frac{|\varepsilon_e|^{1+\alpha_1}}{\pi_e}\Big|\x_e\right), \label{eq:vew}\\
&\vw_{\gammae,\alpha_1}^{\wtd}=\frac{1}{\Nb}\sum_{e\in\cO}\frac{|\gamma_e|^{1+\alpha_1}}{\pi_e},\text{
  and}\label{eq:vgw}\\
&\vuw_{\gammae,\alpha_2}=\frac{1}{\Nb}\sum_{e\in\cO}|\gamma_e|^{1-\alpha_2},\label{eq:vg}
\end{align}
where $\cO=\{e\in\cB:\gamma_e\neq 0\}$.
We now present a non-asymptotic bound for $\hbetaw$ with $\ell_1$
penalization in the following theorem.

\begin{theorem}\label{thm:rtrans}
Under Assumptions~\ref{asm:xsubg}-\ref{asm:bpi}, for any $\varsigma>0$ and
$0<\alpha_1,\alpha_2\le 1$, the estimator $\hbetaw$ 
in~\eqref{eq:rtrans} with $\nu=1$ and
$\lambda=
\left\{(\mpi r)/\varsigma\right\}^{\frac{1}{1+\alpha_1}}$
satisfies that
\begin{equation*}
\|\hbetaw-\tbeta\|_{\bSigma,2}\lesssim \frac{\Ns}{\Ns+r}I_{\cS}
+\frac{r}{\Ns+r}I_{\cB}+\frac{r}{\Ns+r}I_{\gammae},
\end{equation*}
with probability at least $1-Ce^{-\varsigma}$, if
\begin{align}
  &\Ns\gtrsim e^{2\varsigma}, \label{eq:1}\\
  &r\gtrsim \vw_{\gammae,\alpha_1}^{\wtd}\varsigma, \text{ and}\label{eq:2}\\
  &(\Ns+r)r^{\frac{-\alpha_1-\alpha_2}{1+\alpha_1}}
    \gtrsim \vuw_{\gammae,\alpha_2}\varsigma^{\frac{1-\alpha_2}{1+\alpha_1}},\label{eq:3}  
\end{align}
where 
\begin{align*}
I_{\cS}=\sigma_{\cS}\sqrt{d}\left(\frac{\varsigma}{\Ns}\right)^{\frac{1}{2}},
I_{\cB}=(\vw_{\varepsilon,\alpha_1}^{\wtd}\maxs 1)\sqrt{d}
\left(\frac{\varsigma}{r}\right)^{\frac{\alpha_1}{1+\alpha_1}},\text{
  and}\\
  I_{\gammae}
  :=I_{\gammae}^{\textrm{var}} + I_{\gammae}^{\textrm{bias}}
  :=(\vw_{\gammae,\alpha_1}^{\wtd}\maxs 1)\sqrt{d}\left(
\frac{\varsigma}{r}\right)^{\frac{\alpha_1}{1+\alpha_1}}
  +\vuw_{\gammae,\alpha_2}
  \left(\frac{\mpi r}{\varsigma}\right)^{\frac{\alpha_2}{1+\alpha_1}},
\end{align*}
and $C$ is a constant that does not depend on $\Ns$ or $\Nb$.
\end{theorem}

Theorem~\ref{thm:rtrans} applies to general forms of bias terms. However, the
specific structure of these bias terms influences both the conditions and
corresponding bounds in the theorem. To facilitate our discussion of
Theorem~\ref{thm:rtrans}, we introduce Assumption~\ref{asm:gamma}, which
describes a mixed structure for the bias terms.

\begin{assumption}\label{asm:gamma}
The mean-shift parameters $\gamma_e$ are i.i.d. with the same distribution of
$\gamma=\tilde{\gamma}I(u\le p_{\gamma})$, where $\tilde{\gamma}$ is a random
variable, $u\sim Unif[0,1]$ is independent of $\tilde{\gamma}$, and $0<p_{\gamma}\le 1$ bounds
the probability of $\gamma$ being non-zero.    
\end{assumption}

In Assumption~\ref{asm:gamma}, the bias term $\gamma$ is modeled as a mixture of
zero and a nonzero random variable $\tilde{\gamma}$, with the mixture rate
controlled by $p_{\gamma}$. If $p_{\gamma} < 1$, there is a positive probability
that $\gamma_e = 0$ for some $e \in \mathcal{B}$; if $p_{\gamma} = 1$, all data
points in $\mathcal{B}$ deviate from the model for $\mathcal{S}$.

Conditions~\eqref{eq:1}, \eqref{eq:2}, and \eqref{eq:3} in
Theorem~\ref{thm:rtrans} require the sample sizes to be large enough for the
bound to hold with the exponential tail probability rate.
Condition~\eqref{eq:1} is a restriction for $\Ns$.
 Under Assumptions~\ref{asm:bpi} and~\ref{asm:gamma}, Condition~\eqref{eq:2}
 reduces to
 $r\gtrsim p_{\gamma}\Exp(|\tilde{\gamma}|^{1+\alpha_1})\varsigma$. This
 indicates that the required subsample size from the external data is affected
 by the both the outlier proportion $p_{\gamma}$ and the average magnitude of the
outliers $\Exp(|\tilde{\gamma}|^{1+\alpha_1})$. 
 Condition~\eqref{eq:3} puts a restriction on both $\Ns$ and $r$. When
 $\alpha_2=1$, $\vuw_{\gammae,\alpha_2}$ becomes the sample proportion of outliers in the
 external data, and Condition~\eqref{eq:3} reduces to
 $(\Ns+r)r^{-1}\gtrsim p_{\gamma}$ under Assumption~\ref{asm:gamma}. This is
 always satisfied since $p_{\gamma}<1$.

The error bound of $\hbetaw$ in Theorem~\ref{thm:rtrans} can be decomposed
into three sources. The $I_{\cS}$ in the first term can be regarded as the
uncertainty due to the randomness of the target data. The convergence rate of this
term is of $1/\sqrt{\Ns}$, which is the regular least-squares convergence
rate. The $I_{\cB}$ in the second term can be regarded as the uncertainty due to randomness of
the external data. If $\vw_{\varepsilon,\alpha_1}^{\wtd}$ is bounded in
probability, the convergence rate of $I_{\cB}$ is
$1/r^{\frac{\alpha_1}{1+\alpha_1}}$, which is the same as the
convergence rate of Huber regression estimator discussed in~\cite{sun2020adaptive}.
When the model error for the external data $\varepsilon_{\cB}$ is light-tailed with a finite second moment, 
we can set $\alpha_1=1$, and the rate becomes $1/\sqrt{r}$. When $\varepsilon_{\cB}$'s
distribution has a heavier tail, the convergence rate is slower than the
regular root rate, but $I_{\cB}$ still converges to 0. This means our method with
$\ell_1$ penalty is able to transfer useful information from outlier
contaminated external data with heavy-tailed model errors.

The $I_{\gammae}$ in the third term is the error due to the biases in the
external data, and it contains two parts $I_{\gammae}^{\textrm{var}}$ and
$I_{\gammae}^{\textrm{bias}}$. The first part $I_{\gammae}^{\textrm{var}}$ can
be regarded as the contribution of $\gamma_e$ to the variance, and it converges
to zero at the same rate as $I_{\cB}$ if $\vw_{\gammae,\alpha_1}^{\wtd}$ is
bounded in probability.  The second part $I_{\gammae}^{\textrm{bias}}$ can be
regraded as the contribution of $\gamma_e$ to the bias, and it may not converge
to zero without strong assumption on $\gamma_e$.  Note that under
Assumption~\ref{asm:gamma}, $\vuw_{\gammae,\alpha_2}\approx
p_{\gamma}\Exp(|\tilde{\gamma}|^{1-\alpha_2})$. If the outliers in the external data
are sparse, i.e., $p_{\gamma}$ is small, we can take $\alpha_2=1$ so that
$\vuw_{\gammae,\alpha_2}$ is not affected by the magnitude of the outliers at
all. In this case, $I_{\gammae}^{\textrm{bias}}\approx p_{\gamma}(\mpi
r/\varsigma)^{\frac{1}{1+\alpha_1}}$, and we need
$p_{\gamma}=o\{(\varsigma/r)^{\frac{1}{1+\alpha_1}}\}$ for it to go toward zero.  If
the outliers is not very sparse and $p_{\gamma}$ is bounded away from zero, then we
  have $I_{\gammae}^{\textrm{bias}}\approx\Exp(|\tilde{\gamma}|^{1-\alpha_2})(\mpi
r/\varsigma)^{\frac{\alpha_2}{1+\alpha_1}}$ and need
$\Exp(|\tilde{\gamma}|^{1-\alpha_2})=o\{(\varsigma/r)^{\frac{\alpha_2}{1+\alpha_1}}\}$
in order to have $I_{\gammae}^{\textrm{bias}}$ converge to zero.  This requires
the magnitudes of the biases to be sufficiently small if there are too many
outliers in the external data.

Our non-asymptotic error bound in Theorem~\ref{thm:rtrans} is more general than
existing results in the literature. For example, when $\alpha_2=1$, the error
bound can be rewritten as
\begin{equation}\label{eq:beyhumbound}
\|\hbetaw-\tbeta\|_{\bSigma,2}\lesssim\frac{\Ns}{\Ns+r}\Ns^{-\frac{1}{2}}
+\frac{r}{\Ns+r}(r^{-\frac{\alpha_1}{1+\alpha_1}}\maxs\lambda p_{\gamma}),
\end{equation}
because $s_{\gammae,\alpha_2=1}=p_{\gamma}$ and $\lambda=\{(\mpi
r)/\varsigma\}^{\frac{1}{1+\alpha_1}}$. Taking $\Ns=0$ and $r=\Nb$,
the estimator $\hbetaw$ reduces to an $\ell_1$-penalized mean-shift regression estimator,
and in this case, our error bound is similar
to the asymptotic error bound in \cite{beyhum2020inference}, which is
$O(\Nb^{-\frac{\alpha_1}{1+\alpha_1}} \maxs\lambda p_{\gamma})$. Our result
holds even if $\varepsilon_{\cB}$ is very heavy tailed while \cite{beyhum2020inference} requires
  $\varepsilon_{\cB}$ to be light-tailed. In addition, \cite{beyhum2020inference} does
  not cover the scenario that the outliers are not sparse.

Theorem~\ref{thm:rtrans} also indicates the
substantial difference of our study from the existing investigations in the
subsampling literature.
To illustrate this, we simplify the bound in Theorem~\ref{thm:rtrans} as
\begin{align}\label{eq:simpr}
  \|\hbetaw-\tbeta\|_{\bSigma,2}
  \lesssim\frac{\sqrt{\Ns\varsigma}}{\Ns+r}
  +\frac{r^{\frac{1}{1+\alpha_1}}\varsigma^{\frac{\alpha_1}{1+\alpha_{1}}}}{\Ns+r}
  +\frac{r^{1+\frac{\alpha_2}{1+\alpha_1}}}{(\Ns+r)\varsigma^{\frac{1}{1+\alpha_1}}}.
\end{align}
Note that the bound is a monotonically decreasing function of $\Ns$ but not a
monotonic function of $r$.  Therefore, increasing the subsample size of the
contaminated external data may not always increase the estimation efficiency of
$\hbetaw$, which is different from the results in existing subsampling
literature. For the transfer learning scenario we are considering, the subsample
size provides a bias-variance trade-off, and using the entire external data may
not result in the best performance of $\hbetaw$. This justifies the
importance of data selection in transfer learning.

Our analysis in Theorem~\ref{thm:rtrans} focuses specifically on the $\ell_1$
penalization. A similar error decomposition holds for the estimator
$\hbetaw$ with the $\ell_2$ penalization. The corresponding error bound is
expressed in the following theorem.

\begin{theorem}\label{thm:rtrans2}
Under Assumptions~\ref{asm:xsubg}-\ref{asm:bpi}, if
$\vw_{\varepsilon,2}^{\wtd}<\infty$, then for any $\lambda>0$, the estimator defined
in~\eqref{eq:rtrans} with $\nu=2$ satisfies that
\begin{equation*}
\|\hbetaw-\tbeta\|_{\bSigma,2}\lesssim
\frac{\Ns}{\Ns+c_{\lambda}r} I_{\cS}+\frac{c_{\lambda}r}{\Ns+c_{\lambda}r}I_{\cB}
+\frac{c_{\lambda}r}{\Ns+c_{\lambda}r}I_{\gammae},
\end{equation*}
with probability at least $1-Ce^{-\varsigma}$, when $\Ns\gtrsim e^{2\varsigma}$
and $\sqrt{r}\gtrsim \vw_{\gammae,2}^{\wtd}e^{\varsigma}$, 
where $c_{\lambda}=\lambda/(1+\lambda)$,
\begin{align*}
I_{\cS}=\sigma_{\cS}\sqrt{d}\left(\frac{\varsigma}{\Ns}\right)^{\frac{1}{2}},
I_{\cB}=\sqrt{\vw_{\varepsilon,1}^{\wtd}}\sqrt{d}
\left(\frac{\varsigma}{r}\right)^{\frac{1}{2}},
\text{ and }\\
I_{\gammae}:=I_{\gammae}^{\textrm{var}}+I_{\gammae}^{\textrm{bias}}:=\sqrt{\vw_{\gammae,1}^{\wtd}}\sqrt{d}
\left(\frac{\varsigma}{r}\right)^{\frac{1}{2}}+\vuw_{\gammae,0}.
\end{align*}
\end{theorem}

The error bound for $\hbetaw$ with $\ell_2$ penalization is comparable to
the one with $\ell_1$ penalization described in
Theorem~\ref{thm:rtrans}. Similar to the $\ell_1$ case, the error for
$\hbetaw$ with $\ell_2$ penalization can be decomposed into three
components: $I_{\cS}$ represents the contribution of randomness from the target
data; $I_{\cB}$ captures the randomness from the external data; and
$I_{\gammae}$ is associated with the bias terms.

While the interpretation of these error components remains consistent between
the $\ell_1$ and $\ell_2$ penalization cases, there are notable differences.
The first key difference is that $\hbetaw$ with $\ell_1$ penalization is
applicable in more general settings. The moment conditions required in
Theorem~\ref{thm:rtrans} are weaker than those in
Theorem~\ref{thm:rtrans2}. Specifically, for $\ell_1$ penalization, the error
term $\varepsilon_{\cB}$ can be heavy-tailed, whereas $\ell_2$ penalization
requires $\varepsilon_{\cB}$ to be light-tailed enough to have a finite second
moment. As a result, $\ell_1$ penalization can handle external data with
contamination and heavy-tailed errors, while $\ell_2$ penalization may struggle
under such circumstances due to its stricter distributional assumptions.
Another key difference lies in the term related to the bias
$I_{\gammae}^{\textrm{bias}}$. In Theorem~\ref{thm:rtrans2}, this term becomes
$\vuw_{\gammae,0}=\Nb^{-1}\sum_{e\in\cO}|\gamma_e|\approx p_{\gamma}
\mathbb{E}(|\tilde{\gamma}|)$. To ensure this term goes toward zero, both
the proportion of the outliers $p_{\gamma}$ and the magnitude of the bias terms must remain
sufficiently small, regardless of how small $p_{\gamma}$ is. This limitation may impact
the performance of $\hbetaw$ with $\ell_2$ penalization in cases where
$\cB$ contains sparse outliers with large magnitudes.

\subsection{Optimal Sampling Probabilities under Data Contamination}\label{sec:optpi}

A key step of the transfer learning estimator based on $\cBr$ is determining the
optimal $\pi_e$. In subsampling literature, this is typically achieved by
minimizing the asymptotic variance of the subsample estimator, assuming it is
asymptotically unbiased. However, as highlighted in Section~\ref{sec:model},
directly applying this strategy is unsuitable in our setting due to the presence
of outliers.  In this section, we determine optimal sampling
probabilities for $\cB$ under our data contamination scenario. We first present
a corollary that implies the asymptotic normality
of the weighted estimator, $\hbetaw$.

\begin{corollary}\label{cor:norm}
Assume that $\x_e$ for $e\in\cB$ and $\x_t$ for $t\in\cS$ have the same
distribution. For any $\varsigma>0$, $\hbetaw$ has the following
properties: 
\begin{enumerate}[(I)]
\item Under the conditions in Theorem~\ref{thm:rtrans}, if
  $\vw_{\varepsilon,1}^{\wtd}<\infty$, then the estimator $\hbetaw$ with $\ell_1$
  penalization and $\lambda=\{(\mpi r)/\varsigma\}^{\frac{1}{2}}$ satisfies
\begin{align*}
&\left\|\bSigma^{\frac{1}{2}}(\hbetaw-\tbeta)
-\frac{1}{N_{f}}\bSigma^{-\frac{1}{2}}\left\{\sumcs\varepsilon_t\x_t
+\sumcb\frac{\rho\delta_e}{\pi_e}\psi_1(\varepsilon_e;\lambda)\tx_e\right\}\right\|_2
\le\frac{r}{N_{f}}\check{I}_{\bgamma_e}
+\frac{C(d+\varsigma)}{N_{f}},
\end{align*}
with probability at least $1-Ce^{-\varsigma}$,
where $N_{f}=\Ns+r$ is the fused sample size,
$\psi_1(z;\lambda)=zI(|z|\le\lambda)+\lambda I(z>\lambda)-\lambda I(z<-\lambda)$,
$\check{I}_{\gammae}= I_{\gammae} + rN_{f}^{-1}(I_{\gammae})^2
+(d+\varsigma)r^2N_{f}^{-3}(I_{\gammae})^3$, and $C$ is a constant that
does not depend on $\Ns$ and $\Nb$. 
\item Under the conditions in Theorem~\ref{thm:rtrans2}, the estimator
$\hbetaw$ with $\ell_2$ penalization satisfies
\begin{align*}
&\left\|\bSigma^{\frac{1}{2}}(\hbetaw-\tbeta)
-\frac{1}{N_{f\lambda}}\bSigma^{-\frac{1}{2}}\left\{\sumcs\varepsilon_t\x_t
+\sumcb\frac{\rho\delta_e}{\pi_e}\psi_2(\varepsilon_e;\lambda)\x_e\right\}\right\|_2
\le\frac{c_{\lambda}r}{N_{f\lambda}}I_{\gammae}+\frac{C(d+\varsigma)}{N_{f\lambda}},
\end{align*}
with probability at least $1-Ce^{-\varsigma}$, where
$\psi_2(z,\lambda)=\lambda z/(1+\lambda)$, $c_{\lambda}=\lambda/(1+\lambda)$, and $N_{f\lambda}=\Ns+c_{\lambda}r$ is
the fused sample size adjusted by $\lambda$.
\end{enumerate}
\end{corollary}

\begin{remark}\label{rm:3}
Corollary~\ref{cor:norm} implies the asymptotic normality of
$\hbetaw$ under some conditions. Denote
$\Exp[\{\rho\psi(\varepsilon_{\cB};\lambda)\}^2/\pi]=\sigma_{\cB}^2$. 
Under some mild regularity conditions, 
\begin{equation*}
  \sigma^{-1}_{\cS}\bSigma^{-\frac{1}{2}}\frac{1}{\sqrt{\Ns}}
  \sumcs\varepsilon_t\x_t \cvd\Nor(\0,\I_d),\text{ and }
  \sigma^{-1}_{\cB}\bSigma^{-\frac{1}{2}}\frac{1}{\sqrt{r}}
  \sumcb\frac{\rho\delta_e}{\pi_e}\psi_\nu(\varepsilon_e;\lambda)\x_e
  \cvd\Nor(\0,\I_d),
\end{equation*}
where $\cvd$ means weak convergence and $\nu=1,2$. This shows that
\begin{equation*}
\hbetaw \approx \tbeta
+\frac{\sqrt{\Ns}}{N}\Nor(\0,\sigma^2_{\cS}\I_d)
+\frac{\sqrt{r}}{N}\Nor(\0,\sigma^2_{\cB}\I_d)
+\frac{r}{N}\check{I}_{\gammae},
\end{equation*}
where $N=N_{f}$ for the $\ell_1$ penalty and $N=N_{f\lambda}$ for the $\ell_2$
penalty. 
Thus, if $\check{I}_{\gammae}=o_{\mathrm{P}}(1/\sqrt{N})$, then
$\sqrt{N}(\hbeta-\tbeta)$ is asymptotic normal. 
\end{remark}

As shown in Theorem~\ref{thm:rtrans} and Theorem~\ref{thm:rtrans2}, subsampling
probabilities $\{\pi_e\}_{e=1}^{\Nb}$ affect the error bound by affecting
$I_{\cB}$ and $I_{\gammae}^{\textrm{var}}$. These two terms can be considered
as the variance part of $\hbetaw$. In traditional subsampling literature, we
can choose $\{\pi_e\}_{e=1}^{\Nb}$ to minimize the asymptotic variances to
obtain optimal subsamples. However, this approach may result in subsamples with
too many outliers, because the $\{\pi_e\}_{e=1}^{\Nb}$ that minimize
$I_{\gammae}^{\textrm{var}}$ are proportional to
$|\gamma_e|^{(1+\alpha_1)/2}$. Therefore, we propose to find optimal subsampling
probabilities through minimizing $I_{\cB}$, which is the variance component
contributed by the randomness of the external data. The optimal sampling
probabilities $\{\pi_e\}_{e=1}^{\Nb}$ are stated in the following proposition.

\begin{proposition}\label{pro:optpi}
The optimal subsampling probabilities $\{\pi_e\}_{e=1}^{\Nb}$ with expected
subsample size $r$ that minimize
\begin{equation*}
\Exp\left\{\left\|\sumcb\frac{\rho\delta_e}{\pi_e}\psi_{\nu}(\varepsilon_e;\lambda)\bSigma^{-\frac{1}{2}}\x_e
\right\|^2\Big|\Xb\right\}
\quad \text{for } \nu=1,2,
\end{equation*}
are
\begin{equation*}
\pi_e^{\opt}=\frac{r\|\bSigma^{-\frac{1}{2}}\x_e\|_2\mins H}{
\sum_{i=1}^{\Nb}\|\bSigma^{-\frac{1}{2}}\x_i\|_2\mins H},
\text{ }e=1,2,\ldots,\Nb,
\end{equation*}
where $H$ is a threshold value that ensures
$\pi_e^{\opt}\le 1$ and $\sumcb\pi_e^{\opt}=r$. 
\end{proposition}

Our optimal sampling probabilities in Proposition~\ref{pro:optpi}
are the same for $\ell_1$ and $\ell_2$ penalization. This is
because penalization impacts only $I_{\cB}$ by modifying the thresholding
function $\psi_{\nu}(\varepsilon; \lambda)$, which is applied to the random error
terms $\varepsilon$. However, since our goal is to determine the 
optimal subsampling probabilities based on the design matrix $\X_{\cB}$, the
function $\psi_{\nu}(\varepsilon_e; \lambda)$ does not influence the optimal
$\{\pi_e\}_{e=1}^{\Nb}$.

The optimal probabilities in Proposition~\ref{pro:optpi} do not dependent on the
responses. This type of sampling probabilities are useful in the scenario that
responses are difficult or expensive to obtain %
\citep{Zhang2020optimal}. Our
purpose of using response-free sampling probabilities is different. We use
covariates only to build optimal $\{\pi_e\}_{e=1}^{\Nb}$ to avoid including
outliers.

In practice, $\bSigma=\Exp(\x_e\x_e\tp)$ is unknown. We propose to use
$\sumcb\x_e\x_e\tp/\Nb$ to approximate it. The resulting approximate optimal
subsampling probabilities are
\begin{equation*}
  \hat{\pi}_e^{\opt}=\frac{r\|(\Xb\tp\Xb)^{-\frac{1}{2}}\x_e\|_2\mins H_a}{\sum_{i=1}^{\Nb}
    \|(\Xb\tp\Xb)^{-\frac{1}{2}}\x_i\|_2\mins H_a},e=1,2,\ldots,\Nb,
\end{equation*}
where $H_a$ is an approximated value of $H$.  Note that
$\|(\Xb\tp\Xb)^{-\frac{1}{2}}\x_e\|$'s in $\hat{\pi}_e^{\opt}$ are the square
roots of leverage scores, which are widely used in subsampling. Our results
shows that leverage scores can be used to
minimize $I_{\cB}$. Another interesting observation
is that a special case of Proposition~\ref{pro:optpi} implies the optimality of
leverage-based subsampling for Huber regression. Specifically, if there is no
$\cS$ and we only consider optimal sampling for the mean-shift model with Huber
regression, then leverage-scores produce the optimal response-free
subsampling probabilities.

\section{Transfer Learning with Target-guided Data Selection}\label{sec:dtrans}

In this section, we study target-guided data selection. We first
state the procedure of data selection and propose a transfer learning estimator
based on the selected data in Section~\ref{sec:det}. Next, we present theoretical analysis
of the proposed estimator and discuss its implications in Section~\ref{sec:dtheory}.

\subsection{Target-guided Data Selection and the Estimator}
\label{sec:det}

We propose a target-guided selection approach aimed at selecting
data points in $\cB$ where $\gamma_e$ is either zero or close to zero, i.e.,
data points that share a similar underlying model with $(y_t, \x_t)$. Given that
the target data $\cS$ is unbiased, we use the magnitude $|y_e -
\x_e\tp\hbeta_{\cS}|$ to estimate the bias term $\gamma_e$ in the external data.
Heuristically, $|y_e - \x_e\tp\hbeta_{\cS}|$ serves as an estimator of $|y_e -
\x_e\tp\tbeta| = |\varepsilon_e + \gamma_e|$. Since the residuals
$\varepsilon_e$ are expected to have a distribution centered around zero,
smaller values of $|y_e - \x_e\tp\tbeta|$ typically correspond to smaller values
of $\gamma_e$, i.e., the ``good'' data points.  Following this intuition, by
sorting the data in an increasing order of $|y_e - \x_e\tp\hbeta_{\cS}|$, we should
observe a concentration of the ``good'' data points at the beginning of the sorted
list. %
The target-guided selection is formalized in Algorithm~\ref{alg:det}.

\begin{algorithm}[H]%
  \caption{Target-guided selection}
  \label{alg:det}
  \begin{algorithmic}[1]
    \STATE Compute $\hbeta_{\cS}:=\argmin_{\bbeta}\sumcs(y_t-\x_t\tp\bbeta)^2$;
    \STATE Compute $|y_e-\x_e\tp\hbeta_{\cS}|$, for $e=1,2,\ldots,\Nb$;
    \STATE Include the $r$ data
    points with the smallest values of $|y_e-\x_e\tp\hbeta_{\cS}|$ in the
    subsample.
  \end{algorithmic}
\end{algorithm}

Algorithm~\ref{alg:det} does not introduce additional randomness; the selected
data points are deterministic with given target and external datasets. The use
of extreme statistics for data selection is common in the subsampling
literature, such as~\cite{WangYangStufken2018} and
~\cite{cheng2020information}, which are
driven by optimal experimental design principles, aiming to maximize the
information within the subsample. However, our motivation for using order statistics
is different: we seek to minimize bias
in the subsample by carefully selecting data points with minimal external bias.

Let the target-guided subsample be
$\cBd=\{e\in\cB:|y_e-\x_e\tp\hbeta_{\cS}|\le|y-\x\tp\hbeta_{\cS}|_{(r)}\}$,
where $|y-\x\tp\hbeta_{\cS}|_{(r)}$ is the $r$-th order statistics of
$|y_e-\x_e\tp\hbeta_{\cS}|$ for $e\in\cB$. We
define a transfer learning estimator by taking $\cBs=\cBd$, $w_e=1$ and
$P(\gamma_e)=2\nu^{-1}|\gamma_e|^{\nu}$ in~\eqref{eq:ftrans}, which results in the following estimator:
\begin{equation}\label{eq:dtrans}
(\hbetad\tp,\hgamma_{\cBd}\tp)\tp
=\mathop{\arg\min}_{\bbeta,\bgamma_{\cBd}}\left\{\sumcs(y_t-\x_t\tp\bbeta)^2
+\sumcbd (y_e-\x_e\tp\bbeta-\gamma_e)^2
+\lambda\sumcbd\frac{2|\gamma_e|^{\nu}}{\nu}\right\}.
\end{equation}
Again, we know from
Proposition~\ref{pro:thr} that $\hbetad$ in~\eqref{eq:dtrans} minimizes
\begin{equation*}
\sumcs(y_t-\x_t\tp\bbeta)^2+\sumcbs
w_e\mathcal{H}(y_e-\x_e\tp\bbeta;\lambda)
=\sumcs(y_t-\x_t\tp\bbeta)^2+\sum_{e\in\cBd}
\mathcal{H}(y_e-\x_e\tp\bbeta;\lambda).
\end{equation*}

\subsection{Theoretical Analysis}\label{sec:dtheory}

We investigate the theoretical aspects of $\hbetad$ in this section,
beginning with an additional assumption.

\begin{assumption}\label{asm:epsLip}
The density function of $\varepsilon$ in $\cB$, $\feps(z)$, is
Liptchitz with a Liptchitz constant $\Leps$. In addition, for any $z\in\mathbb{R}$, there
exists an $\alpha_{\varepsilon}>0$ and a bounded function $H(z)$ such that $\lim_{w\to
  0}w^{-\alpha_{\varepsilon}}\{\feps(z+w)-\feps(z)\}=H(z)$.
\end{assumption}

Assumption~\ref{asm:epsLip} imposes a smoothness condition on the density of
$\varepsilon$ for the external data, which is essential for deriving non-asymptotic bounds for
estimators based on target-guided selection. This assumption is satisfied by
many symmetric distributions, including the normal distribution and the
$t$-distribution.

Recall that we select data points in $\cB$ corresponding to the $r$ smallest
$\lvert y - \x\tp\hbeta_{\cS} \rvert_{(e)}$'s, which are the order statistics
of $\lvert y_e - \x_e\tp\hbeta_{\cS} \rvert$'s given $\Ds$, for
$e=1,2,\ldots,\Nb$. Here, $\Ds$ represents $\{(y_t,\x_t)\}_{t\in\cS}$. We first
analyze the properties of the bias terms
(mean-shift parameters) for the selected data points. Denote the bias terms in
the selected data points corresponding to the order statistics $\lvert y -
\x\tp\hbeta_{\cS} \rvert_{(i)}$ as $\gamma_{[i]}$, for $i = 1, 2, \ldots,
r$.  Here, $\gamma_{[i]}$'s are the concomitants of the order statistics
\citep[see,][]{yang1977general,chu1999study}. We will use tools for concomitants
and order statistics in extreme statistics to derive the properties of
$\gamma_{[i]}$'s. We begin by presenting the following result.

\begin{proposition}\label{pro:nobias}
Under Assumptions~\ref{asm:xsubg}-\ref{asm:eps} and~\ref{asm:gamma}-\ref{asm:epsLip}, let $p_{\gamma}$ be upper
bounded away from 1 and $\tilde{\gamma}$ be lower bounded away from 0, e.g.,
$\tilde{\gamma}\geq \gamma_b>0$ almost surely for some $\gamma_b$.
For every $\varsigma>0$, if $\Ns\gtrsim e^{2\cvsigma}$, then
\begin{align*}
\Pr\left(\gamma_{[i]}=0, \; 1\le i\le r\right)\geq 1-Ce^{-\varsigma}-\check{p},
\end{align*}
where 
\begin{equation*}
\check{p}=\frac{Cr^2}{\Nb^{\alpha_{\varepsilon}\mins 1}}+Crp_{\gamma}\Exp\{\feps(\ttgamma)\}
+\frac{Crp_{\gamma}\sqrt{\varsigma}}{\sqrt{\Ns}},
\end{equation*}
and $C$ is a constant that does not depend on the sample sizes.
\end{proposition}

Proposition~\ref{pro:nobias} shows that if $p_{\gamma}<1$, meaning that there
are data points with no biases, then the probability that
Algorithm~\ref{alg:det} selects $r$ data points with no biases approaches one.

The $\check{p}$ in Proposition~\ref{pro:nobias} contains three terms. The first term is
the gap between the distribution of order statistics and their limiting
distribution \citep[see][]{falk1993mises}. This term converges to zero if
$r^2=o(\Nb^{\alpha_{\varepsilon}\mins 1})$, which can be satisfied when the
external data are much larger than the data to be transferred. Here, the
$\alpha_\varepsilon$ is determined by the
smoothness of $f_{\varepsilon}(z)$, the density function of $\varepsilon$. For
example, if $f_{\varepsilon}(z)$ has a bounded derivative, then
Assumption~\ref{asm:epsLip} holds with $\alpha_{\varepsilon}=1$, which gives
the fastest convergence rate of $r^2/\Nb$.

The second term in $\check{p}$ is affected by the subsample size, the tail
behavior of $\varepsilon$, and the magnitude of biases in $\cB$. We illustrate
this by assuming $\varepsilon\sim t_{\nu}$, the $t$-distribution with degree of
freedom $\nu$. When $\feps(\cdot)$ is the
density function of $t_{\nu}$, we have that
$\Exp\{\feps(\ttgamma)\}\lesssim 1/\gamma_b^{\nu}$ because of
$\ttgamma\geq\gamma_b$ assumed in Assumption~\ref{asm:gamma}.  Therefore, the
second term $rp_{\gamma}\Exp\{\feps(\ttgamma)\}\lesssim rp_{\gamma}/\gamma_b^{\nu}$, indicating
that the sample size $r$ that can be taken from the external data is
determined by the tail of $\varepsilon$ (reflected by $\nu$) and the magnitude of $\gamma_b$. When
$p_{\gamma}\to 0$, we can take $r=O(\gamma_b^{\nu})$ to let this term converge to 0. If
$\varepsilon$ is lighter-tailed ($\nu$ is larger) and $\gamma_b$ is larger, then
$r$ can be larger, meaning that we can take more data points with no
biases. This is intuitive because if $\varepsilon$'s are small and
$\gamma$'s are large in $\cB$, it will be easier to identify data points with
biases.

The third term in $\check{p}$ is contributed by the target-guided
selection. Here, $rp_{\gamma}\sqrt{\varsigma}/\sqrt{\Ns}$ can be viewed as the level of uncertainty caused by target-guided
selection, and it shows that using the target data 
benefits the quality of the data selection. If we
use uniform sampling to take $r$ data points from $\cB$ without consulting the
target data, the probability of selecting at least one biased data point
is of rate $1-(1-p_{\gamma})^r\approx rp_{\gamma}$, instead of $rp_{\gamma}\sqrt{\varsigma}/\sqrt{\Ns}$. As
long as $\varsigma<\Ns$,
the target-guided data selection
may have a lower probability to include any biased data points. This term also shows that
the size of the target data limits the size of unbiased subsample we are able to
take from $\cB$.
For example, if $p_{\gamma}\nrightarrow0$ and we take $\varsigma=0.5\log\Ns$, we can only
take at most $r=o(\sqrt{\Ns/\log\Ns})$ data points from $\cB$ without
including any biased data points.
This shows that the randomness from the target data affects
the data selection, indicating that available information in the target data
limits the information one can borrow from the outlier contaminated external data.

Proposition~\ref{pro:nobias} does not rely on any specific distribution assumptions
on the error term $\varepsilon$ and the covariate $\x$ in $\cB$. If we assume
normality for $\varepsilon$ and $\x$ in $\cB$, we can
derive non-asymptotic bounds for $\hbetad$.
We introduce the following function to ease the presentation of the results:
\begin{equation}\label{eq:qfunc}
Q(u)=\frac{2\sqrt{F^{-1}_{\chi_1^2}(u)}\exp\{-\frac{1}{2}F_{\chi_1^2}^{-1}(u)\}}{\sqrt{2\pi}u},
\end{equation}
where $F_{\chi_1^2}^{-1}(u)$ is the quantile function of the chi-square
distribution with degrees of freedom one. Here, $Q(\cdot)$ is a decreasing function,
and we will come back to it with more details after the following theorem.

\begin{theorem}\label{thm:dtrans}
Under Assumptions~\ref{asm:xsubg}-\ref{asm:eps}, and~\ref{asm:gamma} 
and assume that $\varepsilon$ and $\x$ in the
external data are normally distributed. Let 
$\tilde{\gamma}\geq\gamma_b>0$ be lower bounded away from 0. 
For every $\varsigma>0$ and $0<\consp<8$, if $p_{\gamma}<\frac{1}{1+\consp}$, $\Ns\gtrsim e^{2\cvsigma}$,
$r\le\{1-(1+\consp)p_{\gamma}\}\Nb$,
$\gamma_b\gtrsim\sqrt{\log\Nb}$,
and there exists a constant
$0<\alpha_3<1$ such that $r\geq 2^{\frac{2}{\alpha_3}}$ and
$r^{\frac{\alpha_3}{2}}\Ns\gtrsim\cvsigma\log\Nb$,
then the following hold with probability at least
$1-Crp_{\gamma}e^{-\frac{\gamma_b^2}{4\sigma_{\varepsilon}^2}}-Ce^{-C[r^{1-\alpha_3}\mins\{r/(\log\Nb)^2\}]}
-Ce^{-C\{(\consp^2 p_{\gamma}\Nb)\mins\log\Nb\}}
-C\cvsigma^{-1}$ for some constant $C$:
\begin{enumerate}[(I)]
\item The estimator defined in~\eqref{eq:dtrans} with $\nu=1$ satisfies that
\begin{equation}\label{eq:dtransb1}
\|\hbetad-\tbeta\|_{\bSigma,2}
\lesssim\frac{\Ns}{\Ns+r}I_{\cS}+\frac{r}{\Ns+r}I_{\cB}^{\dtm},
\end{equation}
if $\lambda\gtrsim\sqrt{\log\Nb}$, where
\begin{align}\label{eq:dtransbound}
\begin{split}
&I_{\cS}=\sigma_{\cS}\sqrt{d}\left(\frac{\varsigma}{\Ns}\right)^{\frac{1}{2}},\\
&I_{\cB}^{\dtm} = I_1^{\dtm}+I_2^{\dtm},\\
&I_1^{\dtm}=\sigma_{\cS}\sqrt{d}\left(\frac{\varsigma}{\Ns}\right)^{\frac{1}{2}}
Q\left\{\left(1-\frac{1}{\log\Nb}\right)\frac{r}{\Nb}\right\}, \quad\text{and}\\
&I_2^{\dtm}=\frac{\sqrt{\cvsigma}\{(\cvsigma^2\log\Nb)\maxs(p_{\gamma}\Nb)\}}{\sqrt{r}}.
\end{split}
\end{align}
\item The estimator defined in~\eqref{eq:dtrans} with $\nu=2$ satisfies that $\forall\lambda>0$,
\begin{equation}\label{eq:dtransb2}
\|\hbetad-\tbeta\|_{\bSigma,2}
\lesssim\frac{\Ns}{\Ns+c_{\lambda}r}I_{\cS}+\frac{c_{\lambda}r}{\Ns+c_{\lambda}r}I_{\cB}^{\dtm},
\end{equation}
where $c_{\lambda}=\lambda/(1+\lambda)$, and $I_{\cS}$ and $I_{\cB}^{\dtm}$ have
the same definitions as in~\eqref{eq:dtransbound}.
\end{enumerate}
\end{theorem}

The bounds \eqref{eq:dtransb1} and~\eqref{eq:dtransb2} in Theorem~\ref{thm:dtrans} decreases
as the target sample size $\Ns$ increases, so increasing $\Ns$ always benefits
the performance of $\hbetad$. If $\Nb$ is 
fixed, then $I_{\cB}^{\dtm}$ is a decreasing function of $r$ (remember that $Q(\cdot)$ is decreasing);
this implies that
if we are always able to include ``good points'' into the sample, 
a larger subsample leads to a better performance. This is different from the results for
random sampling in Theorem~\ref{thm:rtrans}, where the bias term
$I_{\gamma}^{\textrm{bias}}$ is an increasing function of $r$ regardless of
$\Nb$. However, the bound does not necessarily decrease
as $r$ increases if $\Nb$ also increases.
For example, if $p_{\gamma}$ is bounded away from 0, then
$I_2^{\dtm}\approx\sqrt{\cvsigma}(p_{\gamma}\Nb/\sqrt{r})$, which may increase
as $r$ increases since $\Nb$ may increase as well. Therefore, 
it is possible that a subsample estimator outperforms the full data
estimator.

When $p_{\gamma}$ is small enough, e.g., $p_{\gamma}=O(\cvsigma^2\log\Nb/\Nb)$, we have
$I_2^{\dtm}\approx\cvsigma^{\frac{3}{2}}(\log\Nb/\sqrt{r})$. In this case, if
$r$ is sufficiently large, e.g., $r\gtrsim\cvsigma\Ns(\log\Nb)^2$,
then $I_1^{\dtm}$ becomes the dominating term in
$I_{\cB}^{\dtm}$, and the function $Q(u)$ plays an important role in determining
the performance of $\hbetad$. This function is continuous, strictly decreasing,
and satisfies $Q(0)=1$ and $Q(1)=0$ (as shown in Lemma~\ref{lem:chisq} in the
  supplementary material). Therefore, when $r$ is much larger than $n_{\cS}$, e.g.,
  $r\gtrsim\cvsigma\Ns(\log\Nb)^{2+\delta}$, for some $\delta>0$, then $I_2^{\dtm}$
  is small compared with $I_1^{\dtm}$ and therefore can be ignored. We have that 
\begin{equation*}
\frac{\Ns}{\Ns+r}I_{\cS}+\frac{r}{\Ns+r}I_1^{\dtm}\lesssim
\sigma_{\cS}\sqrt{d}\left(\frac{\cvsigma}{\Ns}\right)^{\frac{1}{2}},
\end{equation*}
where the right-hand side is the error bound for $\hbeta_{\cS}$. 
This means that $\hbetad$ will never be worse than $\hbeta_{\cS}$.

The function $Q(\cdot)$ gives a closer description on the connection between the
performance of $\hbetad$ and the sampling rate $r/\Nb$. We plot this
function in Figure~\ref{fig:quantile} to facilitate the discussion.
When the
sampling rate $r/\Nb$ is small, $u=(1-1/\log\Nb)(r/\Nb)$ is close to zero and
$Q(u)$ is close to one; the error bound of $\hbetad$ is close to that of
$\hbeta_{\cS}$. A heuristic interpretation is the selected
``good points'' are all very close to the hyperplane
$y-\x\tp\hbeta_{\cS}$ for a small sampling rate, and the estimator is highly affected by the
uncertainty of the target data $\cS$. When the sampling rate $r/\Nb$ increase, as shown in
Figure~\ref{fig:quantile}, the values of $Q(\cdot)$ drops very fast to zero when
$r/\Nb$ goes to one. In this case, the term $I_1^{\dtm}$ is much smaller than
the bound for $\hbeta_{\cS}$. We state this in the following corollary.

\begin{figure}[t]
  \centering
  \includegraphics[width=0.4\textwidth]{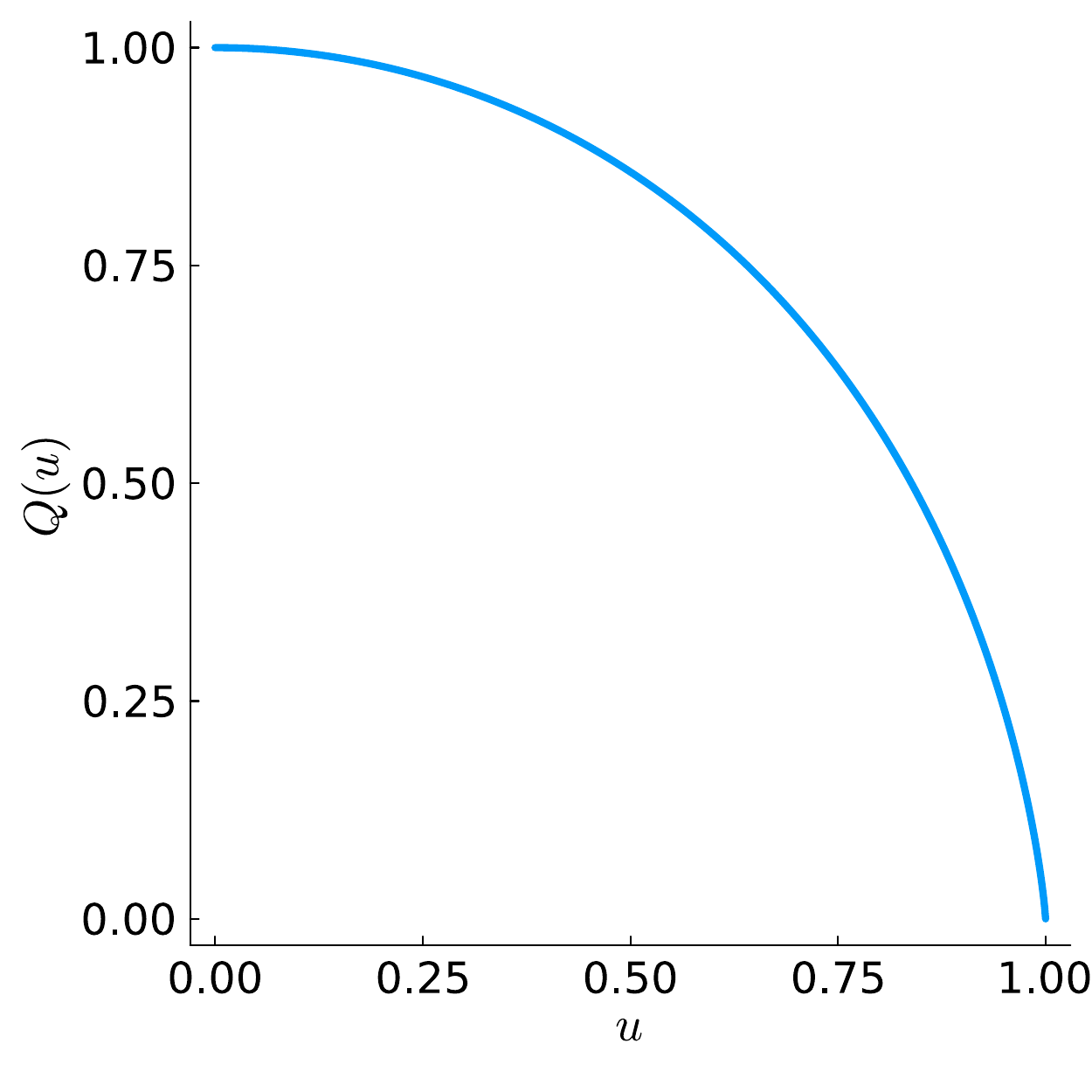}
  \caption{Plot of $Q(u)$}
  \label{fig:quantile}
\end{figure}

\begin{corollary}\label{cor:orderimprove}
Under the same assumptions in Theorem~\ref{thm:dtrans}, if there exists $q>1$ and $C>0$, such that
$(\Nb-r+r/\log\Nb)^{1-\frac{1}{q}}\le C$, then with the same probability bound as in Theorem~\ref{thm:dtrans},
\begin{equation*}
I_1^{\dtm}\lesssim
\left(\frac{\cvsigma}{\Ns}\right)^{\frac{1}{2}}
\left(\frac{1}{\Nb}\right)^{1-\frac{1}{q}}.
\end{equation*}
\end{corollary}

The condition in Corollary~\ref{cor:orderimprove} essentially assumes that $r$ is at the same order of $\Nb$.
In this case, $\hbetad$ can have an order improvement over $\hbeta_{\cS}$. A
heuristic interpretation is more unbiased points that are away from the
hyperplane $y-\x\tp\hbeta_{\cS}$ are included.
These data points contains additional information to the target data $\cS$ and
thus
benefits the performance of $\hbetad$.

Our results in Theorem~\ref{thm:rtrans} and Theorem~\ref{thm:dtrans} offer
further insights into other 
transfer learning models in the literature. For instance,~\cite{li2022transfer}
examines a linear model on external data with a parameter shift, expressed as
$\yb=\Xb(\tbeta+\bdelta)+\bm\varepsilon_{\cB}$, where
$\bm\varepsilon_{\cB}$ is the error vector.  If we define
$\bdelta=(\Xb\tp\Xb)^{-1}\Xb\tp\bgamma_{\cB}$, then the mean-shift model in
~\eqref{eq:sourcemodel} reduces to the parameter shift model. Theorem 1 of
~\cite{li2022transfer} imposes a condition on their method's performance,
requiring that $\|\bdelta\|_{\tau}=o(\sqrt{\log d/\Ns})$ with $\tau\in[0,1]$.
This constraint essentially
necessitates that the mean magnitude of $\gamma_e$ is not excessively large in the
mean-shift model, which is similar to the constraints in~\eqref{eq:2}
and~\eqref{eq:3} in Theorem~\ref{thm:rtrans}. However, for $\hbetad$, we
do not need to put any assumptions on the magnitude of $\gamma$, if
there is a positive proportion of data points in $\cB$ that is unbiased. This is
because intuitively, as long as we have enough external data and a proportion of
identifiable unbiased points, target-guided data selection is always able to
select some unbiased points with a large probability. Therefore, the magnitude
of $\gamma$ will not affect $\hbetad$, because $\hbetad$ are constructed based
on unbiased data points with a large probability. This is different from the
method in~\cite{li2022transfer} and the random subsampling estimator in
Section~\ref{sec:rtrans} as they put restrictions on the mean magnitude of
$\gamma_e$ and thus fail if there are data points with extreme large $\gamma_e$.

\section{Combined Estimators \& Computational issues}\label{sec:computation}

In Sections~\ref{sec:rtrans} and~\ref{sec:dtrans}, we introduced two data
selection strategies: Poisson random sampling and target-guided selection. Here,
we compare the random subsampling estimation with the target-guided selection
estimator.

Heuristically, the first two terms in~\eqref{eq:simpr} for random subsampling
can be considered as the
``variance part'', and the last term can be considered as the ``bias part''.
While, the error bound contributed by $\cB$ in the target-guided selection can
all be considered as ``variance part''. We
see that the bias term for random subsampling is larger compared with
target-guided selection if we select data points with no biases.  For random
sampling, there is no restriction on the subsample size. The
variances of $\hbetaw$ can be made small if we take enough external data
points. However, this may result in a large bias.

On the other hand, 
although the target-guided approach can select data points with no biases, the selection
may be more reliable when $r$ is small. However, when sampling rate is small,
the target-guided selection may result in large variance at the same level as
$\hbeta_{\cS}$. %

In summary, random sampling is tailored to reduce variance but may not
be effective in mitigating biases. Conversely, target-guided
selection focuses on reducing bias but may struggle with variance minimization.
Given these complementary strengths and weaknesses, a natural idea is to combine
the two approaches. By leveraging both methods, we may strike a balance
between bias and variance, which has the potential to yield more accurate
estimators.

In this section, we propose two methods for combining these selection strategies
and address several computational considerations to support their practical
implementation.

\subsection{Combined Estimators}\label{sec:comb}

We propose two methods for combining the data selection
approaches. Specifically, let $\cBr$ denote the subsample obtained through
random sampling and $\cBd$ the subsample obtained via target-guided
selection. The relative sizes of these subsamples can be determined based on
prior knowledge of the external dataset $\cB$.  For example, if the external
dataset exhibits significant bias relative to the target data, a larger
proportion of observations can be allocated to $\cBd$ to reduce bias. On the
other hand, if the external dataset is only mildly biased, assigning a larger
sample size to $\cBr$ would be more effective in mitigating variance.

One approach to combine the two data selection approaches is to merge all data
points in the two subsamples and compute a transfer learning estimator. We refer
to this method as data combining. The resulting estimator incorporates
information from three data sources and can be viewed as a special case
of~\eqref{eq:ftrans} with $\cBs = \cBs_{\mathrm{c}}=\cBr \cup \cBd$. Specifically, the combined
data estimator, $(\hbeta_{\comdat}\tp, \hgamma_{\cBs_{\mathrm{c}}}\tp)\tp$, is the
minimizer of the following objective function:
\begin{equation}
\label{eq:datcomb}
\sumcs(y_t-\x_t\tp\bbeta)^2
+\sum_{e\in\cBs_{\mathrm{c}}}w_e(y_e-\x_e\tp\bbeta-\gamma_e)^2
+\lambda\sum_{e\in\cBs_{\mathrm{c}}}\frac{2w_e|\gamma_e|^{\nu}}{\nu},
\end{equation}
where $w_e=\rho\pi_e^{-1}$ for $e\in\cBr$ and $w_e=1$ for $e\in\cBd$.

Another way to combine the Poisson sampling and the target-guided selection 
approaches is to aggregate the estimator $\hbetaw$ in~\eqref{eq:rtrans} and
the estimator $\hbetad$ in~\eqref{eq:dtrans}. We
refer this approach as estimator combining, which is defined as:
\begin{equation}\label{eq:estcomb1}
\hbeta_{\comest}=(2\V_{\cS}+\V_{\cBr}+\V_{\cBd})^{-1}
\left\{(\V_{\cS}+\V_{\cBr})\hbetaw
+(\V_{\cS}+\V_{\cBd})\hbetad\right\},
\end{equation}
where $\V_{\cS}=\sumcs\x_t\x_t\tp$, 
$\V_{\cBr}=\rho\sumcbr\x_e\x_e\tp/\pi_e$, and $\V_{\cBd}=\sumcbd\x_e\x_e\tp$.
The idea of combining estimators is widely utilized in various domains,
including distributed computing \citep{LinXie2011}, streaming data analysis
\citep{schifano2016online}, and subsampling~\cite{yu2020optimal}.
In our case, the combined estimator $\hbeta_{\comest}$ uses the target dataset
$\cT$ twice, thereby assigning greater weight to data from $\cS$. This is
equivalent to assigning comparatively smaller weights to data from the subsamples
$\cBd$ and $\cBr$.

We provide further insights on the relationship between $\hbeta_{\comdat}$
in~\eqref{eq:datcomb} and $\hbeta_{\comest}$ in~\eqref{eq:estcomb1}, by defining
another objective function:
\begin{equation}
\label{eq:estcomb3}
\sumcs(y_t-\x_t\tp\bbeta)^2
+\frac{1}{2}\sum_{e\in\cBs_{\mathrm{c}}}w_e(y_e-\x_e\tp\bbeta-\gamma_e)^2
+\frac{\lambda}{2}\sum_{e\in\cBs_{\mathrm{c}}}\frac{2w_e|\gamma_e|^{\nu}}{\nu}.
\end{equation}
Compared
with~\eqref{eq:datcomb}, the only difference is that~\eqref{eq:estcomb3} assigns
less weights to the external data points in
$\cBs$. Since we have computed $(\hbetaw\tp,\hgamma_{\cBr}\tp)\tp$ based
on $\cBr$ and $(\hbetad\tp,\hgamma_{\cBd}\tp)\tp$ based on $\cBd$, we
can use $\hgamma_{\dtm}^{*}$ and $\hgamma_{\rw}^{*}$ to replace the
$\gamma_e$'s in~\eqref{eq:estcomb3}, which results in
\begin{equation}\label{eq:estcomb2}
\sumcs(y_t-\x_t\tp\bbeta)^2
+\frac{1}{2}\sumcbr\frac{\rho(y_e-\x_e\tp\bbeta-\hat{\gamma}_{\cBr,e})^2}{\pi_e}
+\frac{1}{2}\sumcbd
(y_e-\x_e\tp\bbeta-\hat{\gamma}_{\cBd,e})^2,
\end{equation}
with $\hat{\gamma}_{\cBr,e}$ and $\hat{\gamma}_{\cBd,e}$ being elements of
$\hgamma_{\cBr}$ and $\hgamma_{\cBd}$, respectively. The combined
estimator $\hbeta_{\comest}$ in~\eqref{eq:estcomb1} is the minimizer
of~\eqref{eq:estcomb2}. Therefore, $\hbeta_{\comest}$ is a simplified combining
estimator that puts more weights on $\cS$.

\subsection{Computational Issues}\label{sec:compissue}

In this section, we discuss the practical implementation for the proposed
transfer learning method. As discussed in previous sections, the proposed
transfer learning estimators can be considered as a special case of the general
estimator in \eqref{eq:ftrans} with specific subsample set $\cBs$ and weights
$\{w_e\}_{e\in\cBs}$.  Penalized estimators similar to \eqref{eq:ftrans}
usually involve coordinate-descend algorithm, which conducts iteration
calculation with both $\hbeta$ and $\hgamma_{\cBs}$ alternatively. Our method can
avoid iterations on $\bbeta$ by utilizing the special structure of mean-shift
models. The following Proposition~\ref{pro:thr2} provides an equivalent
representation of the estimator in \eqref{eq:ftrans} that provides us a unified
framework for the implementation of the proposed estimators.
It shows that
only $\hgamma_{\cBs}$ requires iterative calculations with \eqref{eq:itgamma},
while $\hbeta$ can be calculated with~\eqref{eq:itbeta} after obtaining
$\hgamma_{\cBs}$.

\begin{proposition}\label{pro:thr2}
For any $\lambda>0$
and penalty function $P(\cdot)$, the estimator
in~\eqref{eq:ftrans} matches the estimator obtained at convergence of
iterating 
\begin{equation}
\label{eq:itgamma}
\hgamma_{\cBs,\lambda}^{(k+1)}=\Theta\left(\W_{\cBs}^{-\frac{1}{2}}\bm{z}^{(k)};\lambda\right),
\bm{z}^{(k)}=(\I-\bH_{\cBs}^w)\W_{\cBs}^{\frac{1}{2}}\ybs
+\bH_{\cBs}^w\W_{\cBs}^{\frac{1}{2}}\hgamma_{\cBs,\lambda}^{(k)}-\bH_{\cBs,\cS}^w\ys,
\end{equation}
where $\W_{\cBs}=\diag\{\ldots,w_e,\ldots\}$, for $e\in\cBs$,
$\bH_{\cBs}^w=\W_{\cBs}^{\frac{1}{2}}\Xbs\V^{-1}(\W_{\cBs}^{\frac{1}{2}}\Xbs)\tp$, and
$\bH_{\cBs,\cS}^w=\W_{\cBs}^{\frac{1}{2}}\Xbs\V^{-1}\Xs\tp$; and then calculating 
\begin{equation}
\label{eq:itbeta}
\hbeta^{(k+1)}=\V^{-1}\Xs\tp\ys+\V^{-1}\Xbs\tp\W_{\cBs}(\ybs-\hgamma_{\cBs,\lambda}^{(k+1)}),
\end{equation}
where $\V=\Xs\tp\Xs+\Xbs\tp\W_{\cBs}\Xbs$.
\end{proposition}
Another major issue for practical implementation of a penalized estimator is to
determine the tuning parameter $\lambda$, which is usually obtained through
cross validation. However, cross validation is not appropriate in our case,
because our major goal is to improve the estimation on the target data set,
which is often of a small sample size. Data splitting on the target data may
result in very small training and/or validating sets. This may cause inaccurate
estimation and even become infeasible for the target-guided estimator with
high-dimensional data.
Therefore, we prefer a tuning method that can be implemented without data
splitting. We adopt the idea in~\cite{she2011outlier} and suggest tuning
$\lambda$ with information criteria such as Akaike information criterion (AIC)
and Bayesian information criterion (BIC) as illustrated below.

Since we have two sources of data, the target data and the external data, we
combine the residual sum of squares (RSS) of the two sources to define
information criteria.  The RSS for the source data $\cS$ is
$\text{RSS}_{\cS}=\|\ys-\Xs\hbeta\|_2^2$, where $\hbeta$ is the estimator based
on both data.
For the external data $\cBs$, a reduced linear model can be obtained as
\begin{equation*}
  (\I-\bm{P}_{\Xbs})\ybs=(\I-\bm{P}_{\Xbs})\bgamma_{\cBs}+(\I-\bm{P}_{\Xbs})\ep_{\cBs},
\end{equation*}
where $\bm{P}_{\Xbs}$ is the projection matrix of $\Xbs$. Therefore, we define the
RSS for the external data $\cBs$ as the RSS form of the reduced linear model,
namely, $\text{RSS}_{\cBs}=\|(\I-\bm{P}_{\Xbs})(\ybs-\hgamma_{\cBs,\lambda})\|_2^2$.
With the combines RSS being defined as
$\text{RSS}=\text{RSS}_{\cS}+\text{RSS}_{\cBs}$, the AIC and BIC are given as
follows:
\begin{align}\label{eq:aic}
\text{AIC}=m\log\left(\frac{\text{RSS}}{m}\right)+2\dfr,\quad\text{ and }\quad
\text{BIC}=m\log\left(\frac{\text{RSS}}{m}\right)+\dfr\left\{\log(m)+1\right\},
\end{align}
where $m=\Ns+\Nb-d$, and $\dfr$ is the degrees of freedom of the model. For
$\ell_1$ penalty, $\dfr=\#\{e:\gamma_e\neq 0\}+1$, where $\#$ denotes the
cardinal number of the set. For $\ell_2$ penalty,
\begin{equation*}
\dfr=\frac{\Nbs}{1+\lambda}+\frac{\lambda
  d}{1+\lambda}+\sum_{i=1}^d \frac{q_i}{\lambda+(1+\lambda)q_i},
\end{equation*}
where $q_i,i=1,2,\ldots,d$ are the eigenvalues of the matrix
$\V_{\cS}^{\frac{1}{2}}\V_{\cBs}^{-1}\V_{\cS}^{\frac{1}{2}}$. Detailed calculations of these
degrees of freedom are presented in the supplementary material.

We present a practical Algorithm~\ref{alg:ftran} to implement the proposed
transfer learning estimators.
\begin{algorithm}[H]%
  \caption{Iterative algorithm for transfer learning estimators}
  \label{alg:ftran}
  \begin{algorithmic}[1]
    \REQUIRE Target data $(\ys,\Xs)$, external data $(\yb,\Xb)$, tuning parameter set
    $\{\lambda_i\}_{i=1}^n$, starting value $\bgamma_{\cBs}^{(0)}$, tolerance
    value $\epsilon$, and max iteration $K$; 
    \STATE Obtain a subsample $(\ybs,\Xbs)$ using Algorithm~\ref{alg:pois} or
    \ref{alg:det}, and record the corresponding weights $\{w_e\}_{e\in\cBs}$.
    \STATE Compute $\V$, $\bH_{\cBs}^w$,
    $\bH_{\cBs}^w\W_{\cBs}^{\frac{1}{2}}\ybs$ and $\bH_{\cBs,\cS}^w\ys$ according to Proposition~\ref{pro:thr2}; 
    \FOR{$i=1,\ldots,n$}
    \FOR{$k=0,\ldots,K$}
    \STATE Compute $\bgamma_{\cBs,\lambda_i}^{(k+1)}$ using~\eqref{eq:itgamma};
    \IF{$\|\bgamma_{\cBs,\lambda_i}^{(k+1)}-\bgamma_{\cBs,\lambda_i}^{(k)}\|_{\infty}<\epsilon$}
    \STATE Break;
    \ENDIF
    \ENDFOR
    \STATE %
    Compute $\hbeta_{\lambda_i}$ according to \eqref{eq:itbeta};
    \STATE Calculate 
    $\text{AIC}_{\lambda_i}$ or $\text{BIC}_{\lambda_i}$ using \eqref{eq:aic};
    \STATE Find  $m = \arg\min_{i=1,..., n}\text{AIC}_{\lambda_i}$ or $m = \arg\min_{i=1,..., n}\text{BIC}_{\lambda_i}$;
    \ENDFOR
    \ENSURE  
    $\hbeta_{\lambda_m}$ and $\hgamma_{\cBs,\lambda_m}^{(k+1)}$.
  \end{algorithmic}
\end{algorithm}

\section{Numerical Experiments}\label{sec:numeric}

\subsection{Simulation}\label{sec:simu}

We conduct simulation to investigate the performances of the proposed
methods. We generate independent target data of size $\Ns=150$ from a linear
model:
\begin{equation*}
y_t=\mu_0+\z_t\tp\ttheta+\varepsilon_t
=\x_t\tp\tbeta+\varepsilon_t, t=1,2,\ldots,\Ns,
\end{equation*}
where $\mu_0=1$ is an intercept, $\ttheta$ is a vector of ones with dimension
$100$, and $\varepsilon_t\sim\Nor(0,1)$. Large external data of size
$\Nb=20000$ are generated from a linear model with mean-shifts:
\begin{equation*}
  y_e=\mu_0+\z_e\tp\ttheta+\gamma_e+\varepsilon_e
  =\x_e\tp\tbeta+\gamma_e+\varepsilon_e, e=1,2,\ldots,\Nb,
\end{equation*}
where $\mu_0$, $\ttheta$, and the distribution of $\varepsilon_e$ are the same
as in the target data model.  We use the same covariate distribution $\z$ for
the target data and the external data. We consider both light-tailed
distribution $\z\sim\Nor(\0,\bSigma)$ and heavy-tailed distribution
$\z\sim T_3(\bSigma)$, where the $i,j$-th element of $\bSigma$ is
$\bSigma_{ij}=0.5^{|i-j|}$, $1\le i,j,\le 100$.

We consider three distinct cases characterized by different bias term
($\gamma_e$) distributions, detailed below:
\begin{enumerate}[]
\item{\textbf{Case 1. Sparse Bias Terms (SP):}} Only a subset
  of the large dataset is biased. Specifically, 70\% of the data points are
  randomly selected as biased, and their bias terms are generated from
  $\gamma_e \sim 2 + \Gamma(1, 1)$, where
  $\Gamma(1,1)$ represents gamma distribution with shape and scale parameters
  equal to 1.

\item{\textbf{Case 2. Heavy-Tailed Bias Terms (HT):}} All data points in
  the external data are biased, with bias terms following a heavy-tailed
  distribution: $\gamma_e \sim |t_2|$. This scenario represents situations where
  outliers exhibit large deviations.

\item{\textbf{Case 3. High-Leverage Bias Terms (HL):}}
  All data points are biased, with bias terms related to the leverage
  scores. Specifically, 
  $\gamma_e = c \xi_e \x_e\tp (\Xb\tp \Xb)^{-1} \x_e$, where
  $\xi_e \sim \Gamma(1, 1)$ and $c = \Nb / (d-1)$. This
  scenario accounts for high-leverage outliers, as the bias terms are influenced
  by the covariates' leverage scores.
\end{enumerate}

We repeat the simulation for $K=500$ times and evaluate the performances of the
proposed methods using an empirical trimmed mean square error (eMSE). For a set
of positive values $\{v_k:k=1,\ldots,K\}$, the trimmed mean with trimming
proportion $\alpha$ is defined as
\begin{equation*}
\text{Trim}\left(\{v_i:k=1,\ldots K\}\right)=\frac{1}{(1-2\alpha)K}
\sum_{k=\lfloor\alpha K\rfloor}^{\lfloor(1-\alpha)K\rfloor}v_{(i)},
\end{equation*}
where $\lfloor\cdot\rfloor$ is the floor function and
$v_{(i)},i=1,\ldots,K$ are the order statistics of $v_1,\ldots,v_K$. 
Our eMSE is defined as
\begin{equation*}
\mathrm{eMSE}(\hbeta)=\text{Trim}\left(\left\{\|\hbeta_k-\tbeta\|_2^2:k=1,\ldots,
  K\right\}\right),
\end{equation*}
where $\hbeta_k$ is the estimate in
the $k$-th replication. We use $\alpha=0.1$ for evaluations. Besides eMSE, we
also calculate the
empirical biases (squared) and empirical trimmed variances: 
\begin{equation*}
\mathrm{eBias}^2=\left\|\bar{\hbeta}-\tbeta\right\|_2^2,
\quad\text{ and }\quad
\mathrm{eVar}=\text{Trim}\left(\left\{
\left\|\hbeta_k-\bar{\hbeta}\right\|_2^2:k=1,\ldots,K\right\}\right),
\end{equation*}
respectively, where $\bar{\hbeta}=K^{-1}\sum_{k=1}^K\hbeta_k$.
We set the sampling
rates to be $\rho=0.0075$, $0.03$, $0.12$, and $0.48$, which correspond to
nominal sample sizes of $\Ns,4\Ns,16\Ns$, and $64\Ns$, respectively. For combined
estimators, equal expected number of observations are taken from the
target-guided selection and the leverage-based random sampling.

\paragraph{$\ell_1$ penalty}

We first present results using $\ell_1$ penalty. Figure~\ref{fig:mseL1} presents
the eMSE. 

\begin{figure}[H]
  \centering 
  \begin{subfigure}{0.255\textwidth}
    \includegraphics[width=\textwidth]{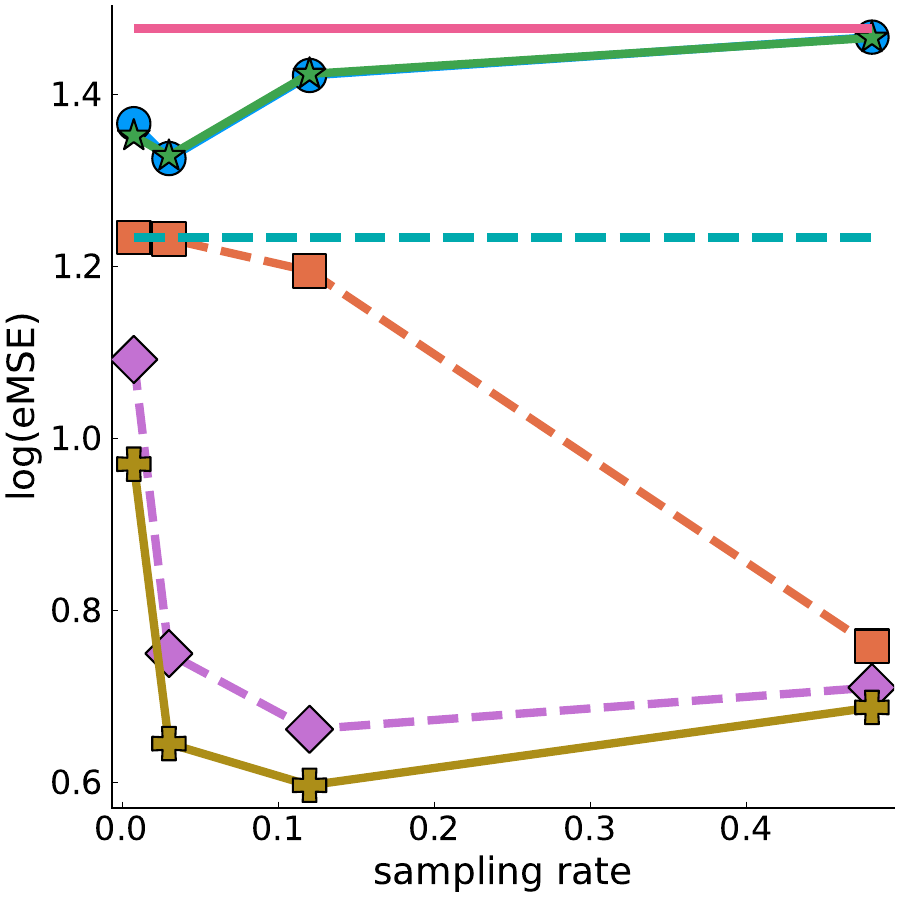}
    \caption{SP (light-tailed $\x$)}
    \label{sfig:a}
  \end{subfigure}
  \begin{subfigure}{0.255\textwidth}
    \includegraphics[width=\textwidth]{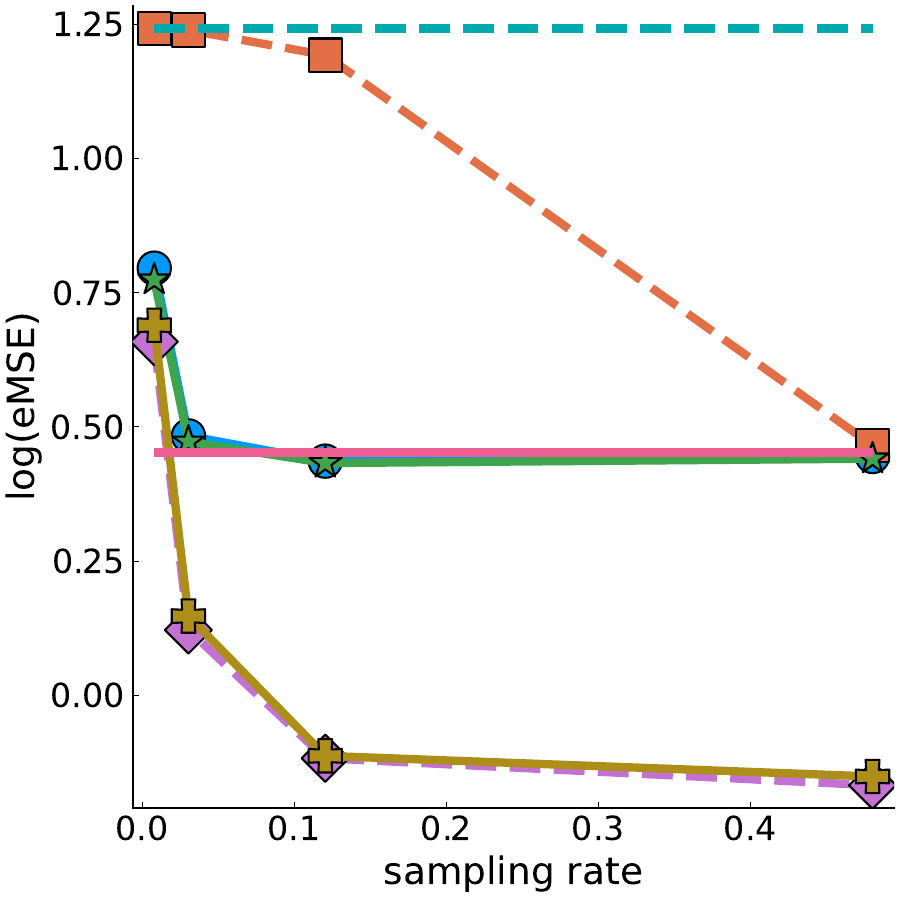}
    \caption{HT (light-tailed $\x$)}
    \label{sfig:b}
  \end{subfigure}
  \begin{subfigure}{0.255\textwidth}
    \includegraphics[width=\textwidth]{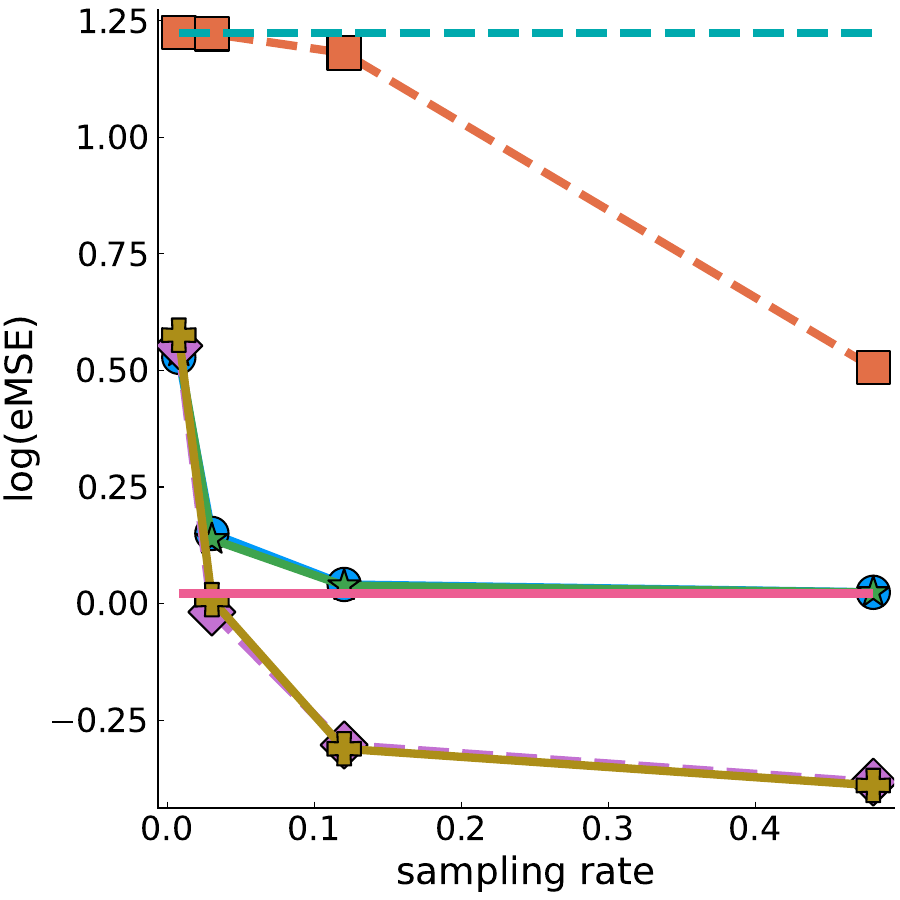}
    \caption{HL (light-tailed $\x$)}
    \label{sfig:c}
  \end{subfigure}
  \begin{subfigure}{0.12\textwidth}
    \includegraphics[width=\textwidth]{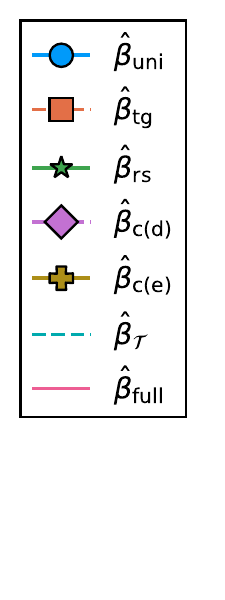}
  \end{subfigure}
  \begin{subfigure}{0.255\textwidth}
    \includegraphics[width=\textwidth]{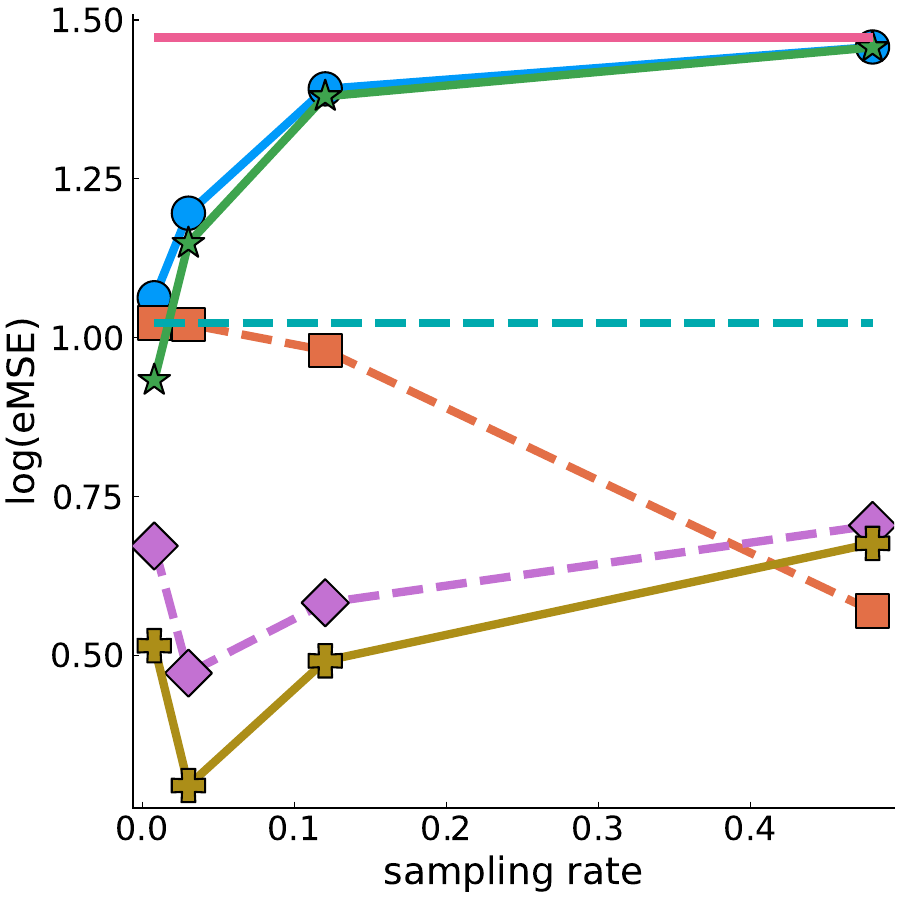}
    \caption{SP (heavy-tailed $\x$)}
    \label{sfig:d}
  \end{subfigure}
  \begin{subfigure}{0.255\textwidth}
     \includegraphics[width=\textwidth]{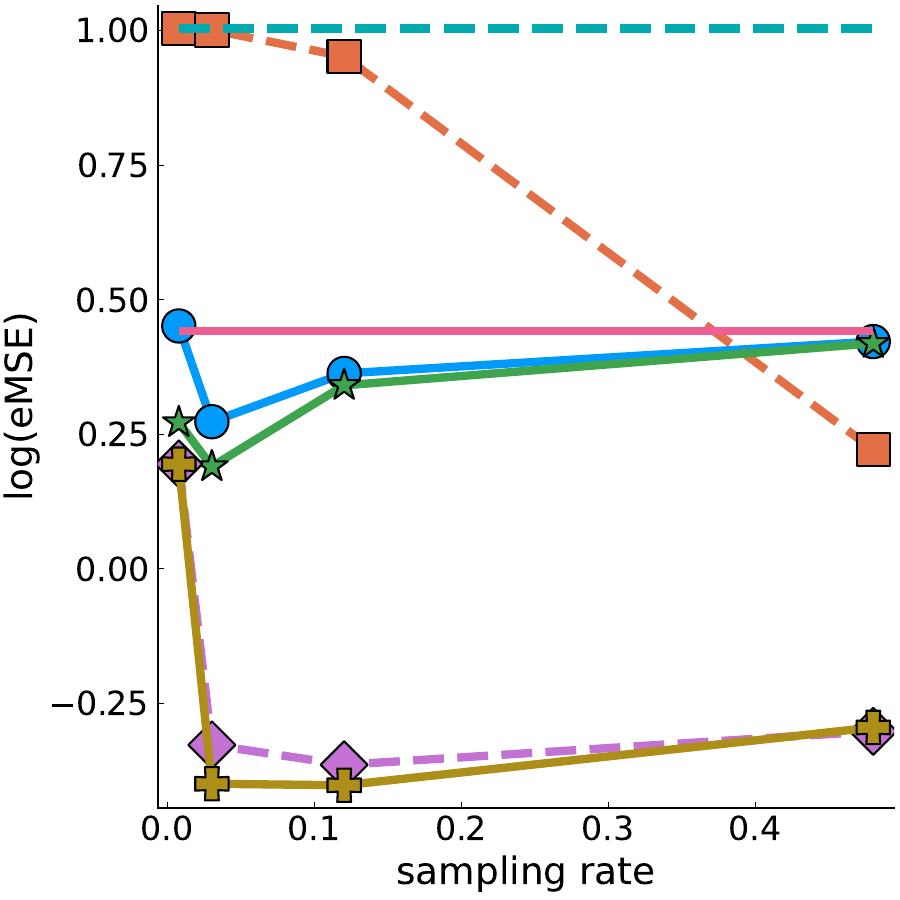}
    \caption{HT (heavy-tailed $\x$)}
    \label{sfig:e}
  \end{subfigure}
  \begin{subfigure}{0.255\textwidth}
    \includegraphics[width=\textwidth]{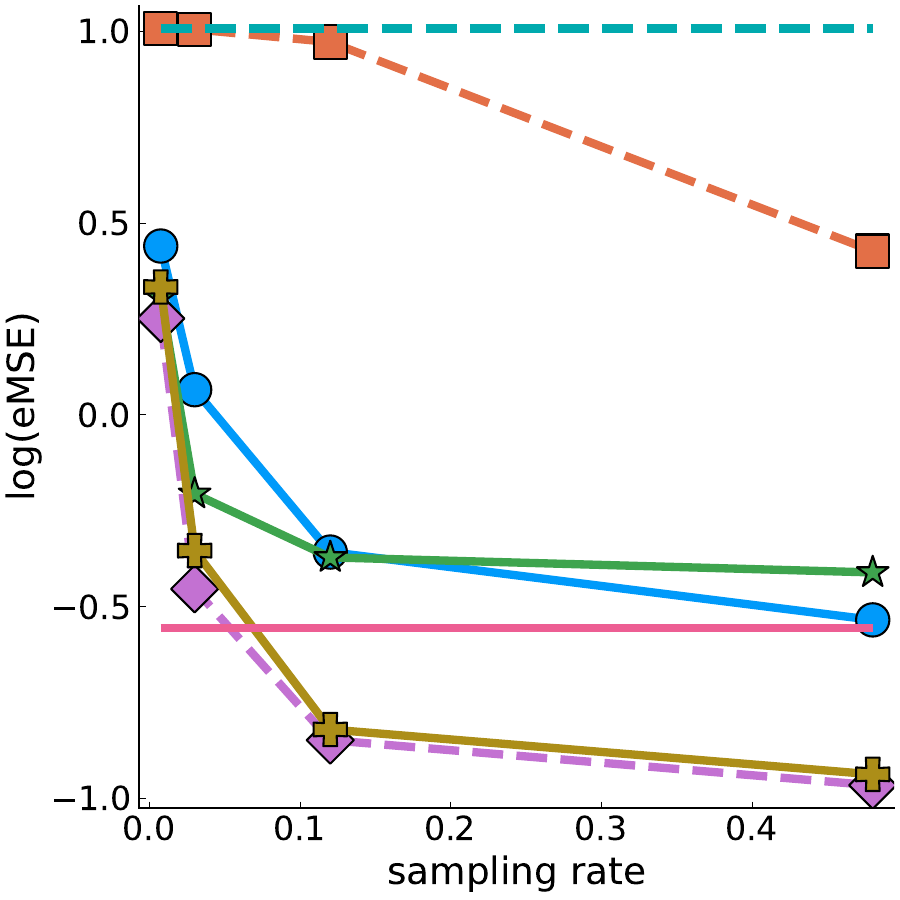}
    \caption{HL (heavy-tailed $\x$)}
    \label{sfig:f}
  \end{subfigure}
  \begin{subfigure}{0.12\textwidth}
    \includegraphics[width=\textwidth]{leg3.pdf}
  \end{subfigure}
  \caption{The eMSE of transfer learning estimators with $\ell_1$ penalty. The
    first row is for light-tailed $\x$ and the second row is for heavy-tailed
    $\x$. The dashed line represents the least squares estimator
    ($\hbeta_{\cS}$) with only target data $\cS$. The solid line represents the
    transfer learning estimator ($\hbeta_{\mathrm{full}}$) with full data
    $\cS\cup\cB$.}
  \label{fig:mseL1}
\end{figure}

Figure~\ref{fig:mseL1} shows that most transfer learning estimators outperform
$\hbeta_{\cS}$, although negative transfer is possible. Among all cases in our
simulations, the two combined estimators perform the best, suggesting that they
benefit from both target-guided selection and optimal subsampling.

In subfigures (\subref{sfig:a}) and (\subref{sfig:d}) for Case ``SP,'' bias
terms are relatively large, and transfer learning estimators using full data
($\hbeta_{\mathrm{full}}$) and random sampling ($\hbeta_{\mathrm{uni}}$ and
$\hbetaw$) show a negative transfer effect; they perform worse than
$\hbeta_{\cS}$. In this case, $\hbetad$ avoids negative transfer. It performs
very similarly to $\hbeta_{\cS}$ when the sampling rate is low and demonstrates
a more significant positive transfer effect as the subsample size from external
data increases.

For Cases ``HT'' and ``HL,'' there is no negative transfer, and $\hbetaw$
performs relatively better, while $\hbetad$ may be less efficient. This is
because $\hbetad$ behaves similarly to $\hbeta_{\cS}$ with small external
subsamples and is subject to high variance.

Notably, the eMSE does not always decrease with larger subsample sizes, which
contrasts with the common trend in subsampling literature that larger
subsamples typically yield better performance and the full data estimator
performs the best. For instance, in subfigures (\subref{sfig:a}) and
(\subref{sfig:d}) for Case ``SP'', the eMSE of $\hbeta_{\comest}$ and
$\hbeta_{\comdat}$ initially decreases but later increases, outperforming the
full data estimator. The random subsampling estimators also surpass the full
data estimator in this case.  This observation aligns with our theoretical
insight. As discussed in Section~\ref{sec:rtrans}, the error bound of $\hbetaw$
is not monotonically decreasing with sample size $r$, indicating a bias-variance
trade-off in data selection for transfer learning.  The
decreasing-then-increasing trend is more pronounced for combined estimators, as
they are specifically designed to optimize this trade-off. This phenomenon
highlights the importance of data selection in mitigating negative transfer.

To further illustrate the bias-variance trade-off, we present the bias and
variance in Figure~\ref{fig:bias-var}. Here, we only present the results for
Case ``HL'', which correspond to Figures~\ref{fig:mseL1}(\subref{sfig:c}) and
\ref{fig:mseL1}(\subref{sfig:f}). We leave results for other cases in the
supplementary material to save space.

\begin{figure}[H]
  \centering
  \begin{subfigure}{0.255\textwidth}
    \includegraphics[width=\textwidth]{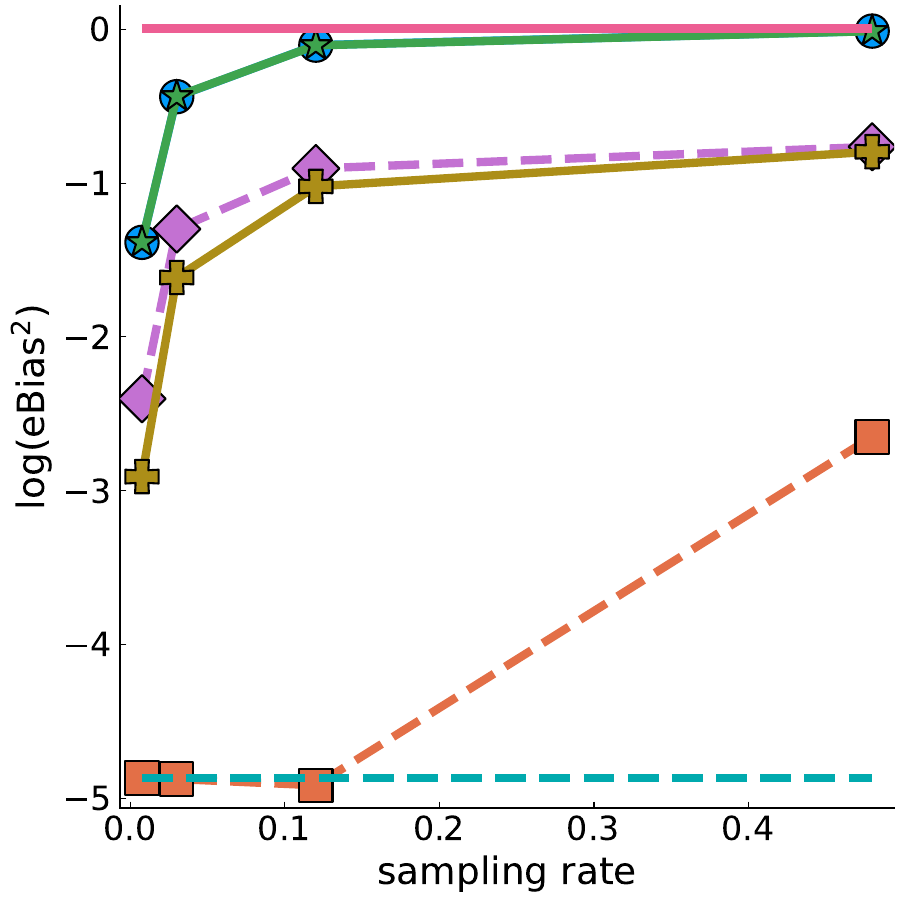}
    \caption*{eBias$^2$ of Figure~\ref{fig:mseL1}(\subref{sfig:c})}
  \end{subfigure}
  \begin{subfigure}{0.255\textwidth}
    \includegraphics[width=\textwidth]{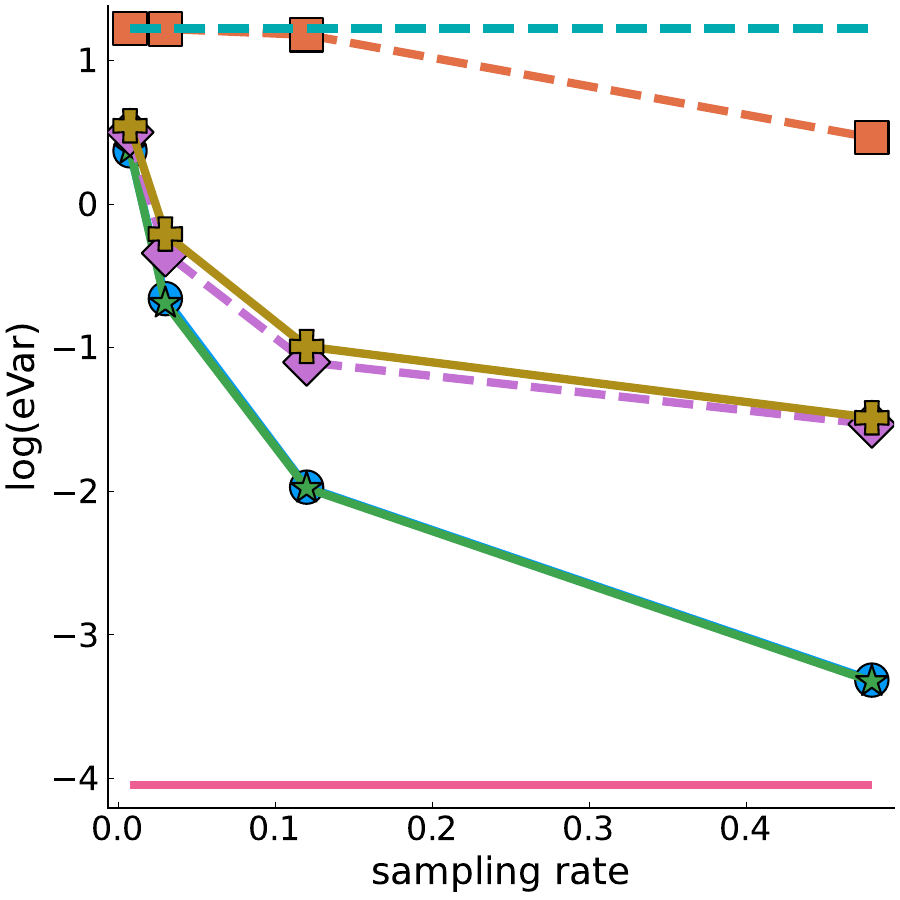}
    \caption*{eVar of Figure~\ref{fig:mseL1}(\subref{sfig:c})}
  \end{subfigure}
  \begin{subfigure}{0.12\textwidth}
    \includegraphics[width=\textwidth]{leg3.pdf}
  \end{subfigure}\\
  \begin{subfigure}{0.255\textwidth}
    \includegraphics[width=\textwidth]{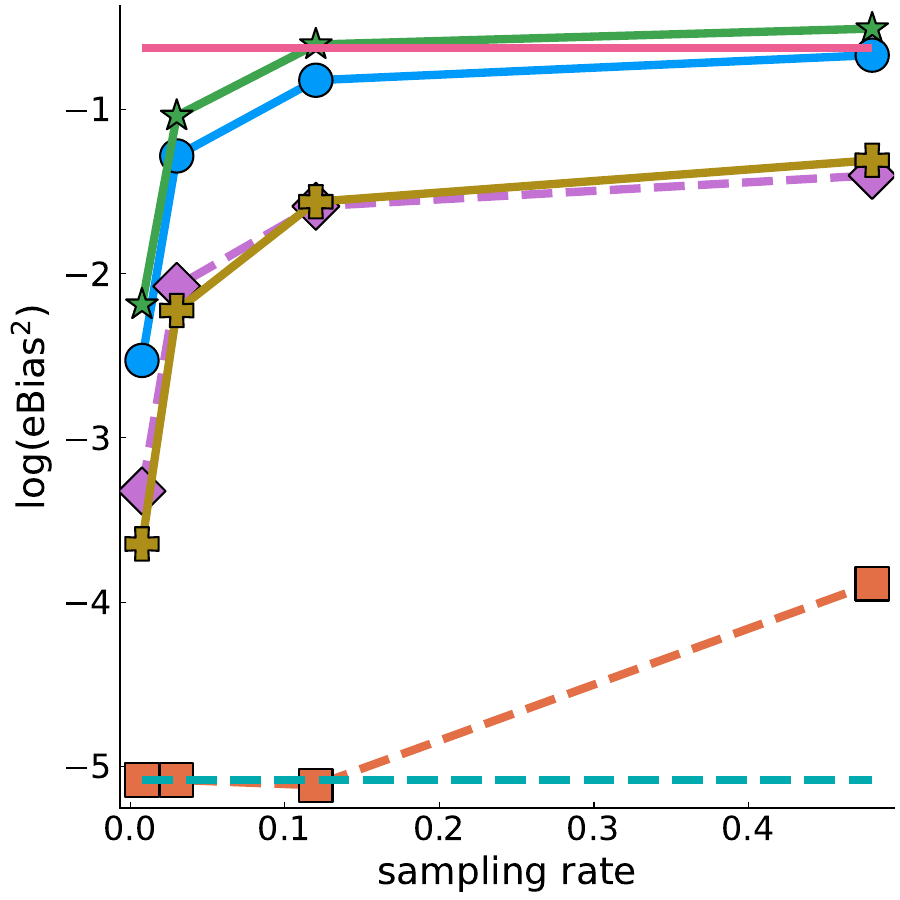}
    \caption*{eBias$^2$ of Figure~\ref{fig:mseL1}(\subref{sfig:f})}
  \end{subfigure}
    \begin{subfigure}{0.255\textwidth}
    \includegraphics[width=\textwidth]{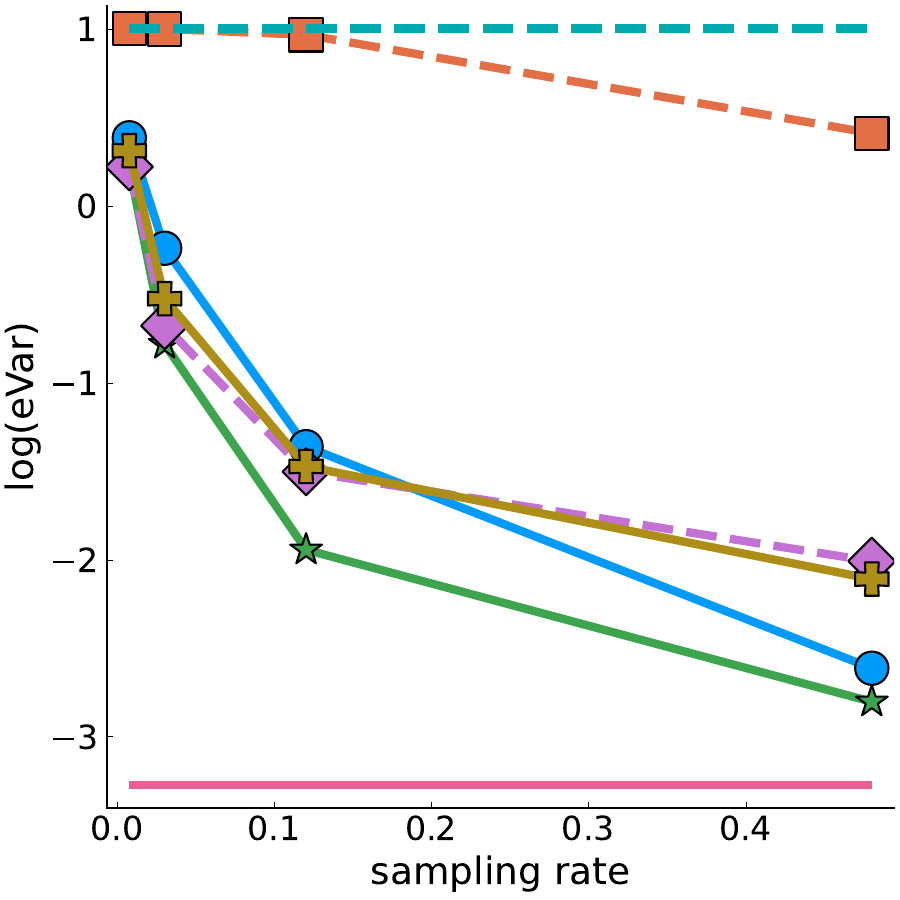}
    \caption*{eVar of Figure~\ref{fig:mseL1}(\subref{sfig:f})}
  \end{subfigure}
  \begin{subfigure}{0.12\textwidth}
    \includegraphics[width=\textwidth]{leg3.pdf}
  \end{subfigure}
  \caption{
    The eBias$^2$ and eVar of transfer learning estimators with $\ell_1$
    penalty for Case ``HL''.}
  \label{fig:bias-var}
\end{figure}

From Figure~\ref{fig:bias-var}, we observe that the estimator $\hbetad$ has the
smallest bias among all sampling estimators, whereas random sampling generally
results in large biases, with leverage-based random sampling exhibiting the
highest bias. The biases of the two combined estimators lie between those of
$\hbetad$ and $\hbetaw$, which aligns with our expectations. In terms of
variance, leverage-based subsampling results in the smallest variance, while
$\hbetad$ has relatively large variance. Similarly, the variances of the
combined estimators fall between those of $\hbetad$ and $\hbetaw$.  These
findings suggest that combined estimators outperform target-guided selection and
leverage sampling by better balancing bias and variance.

\paragraph{$\ell_2$ penalty}%

We now present simulations with the $\ell_2$ penalty, using the same settings as
for $\ell_1$. Figure~\ref{fig:corL2} shows the results for Case ``HL'' with
heavy-tailed covariates $\x$. The results for other cases and light-tailed $\x$
are similar to those of the $\ell_1$ penalty and are provided in the supplementary material.

\begin{figure}[htp]
  \centering 
    \begin{subfigure}{0.255\textwidth}
    \includegraphics[width=\textwidth]{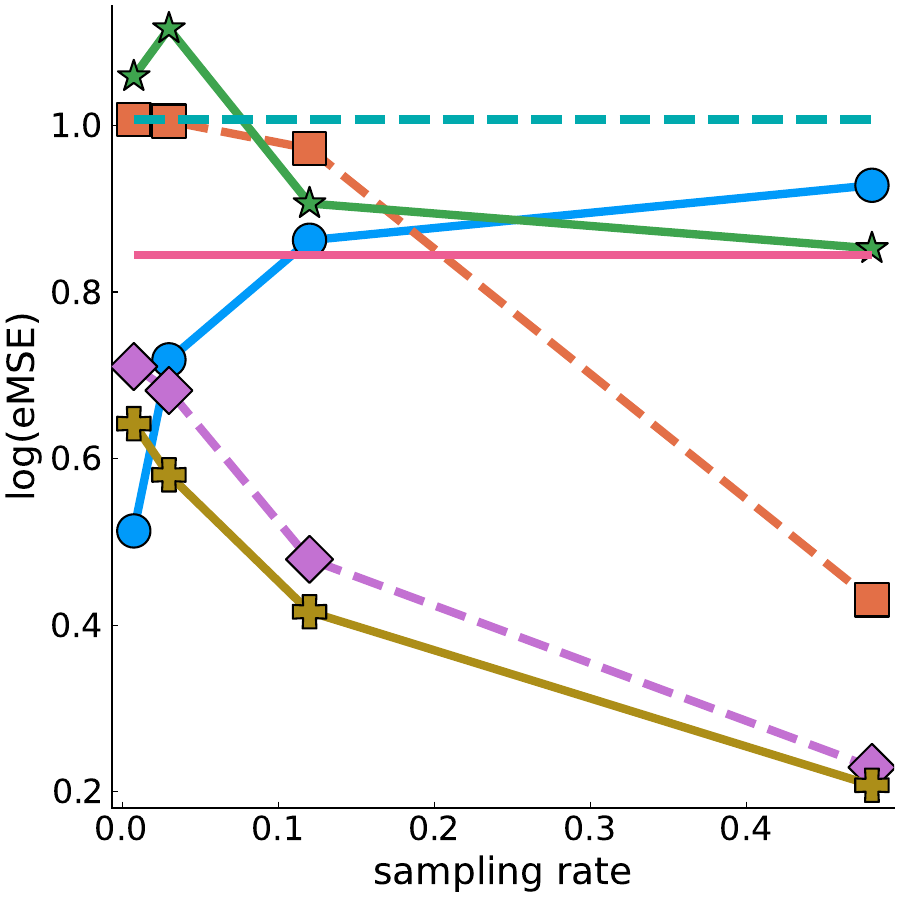}
    \caption*{eMSE}
  \end{subfigure}
    \begin{subfigure}{0.255\textwidth}
    \includegraphics[width=\textwidth]{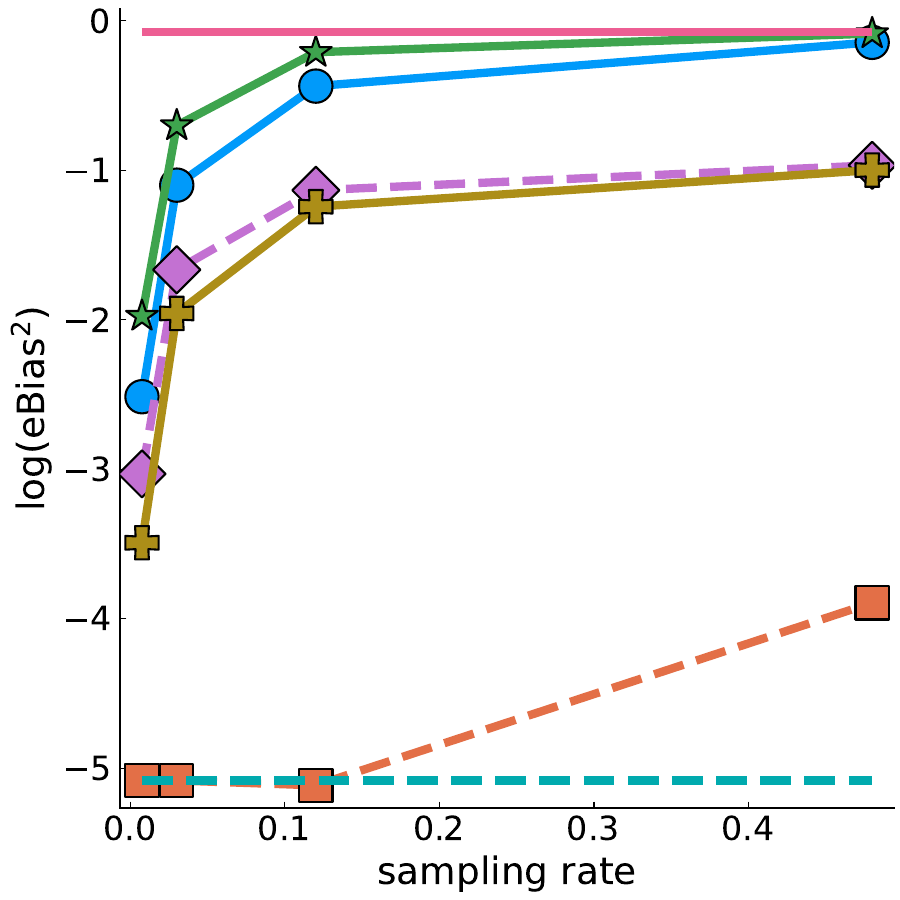}
    \caption*{eBias$^2$}
  \end{subfigure}
  \begin{subfigure}{0.255\textwidth}
    \includegraphics[width=\textwidth]{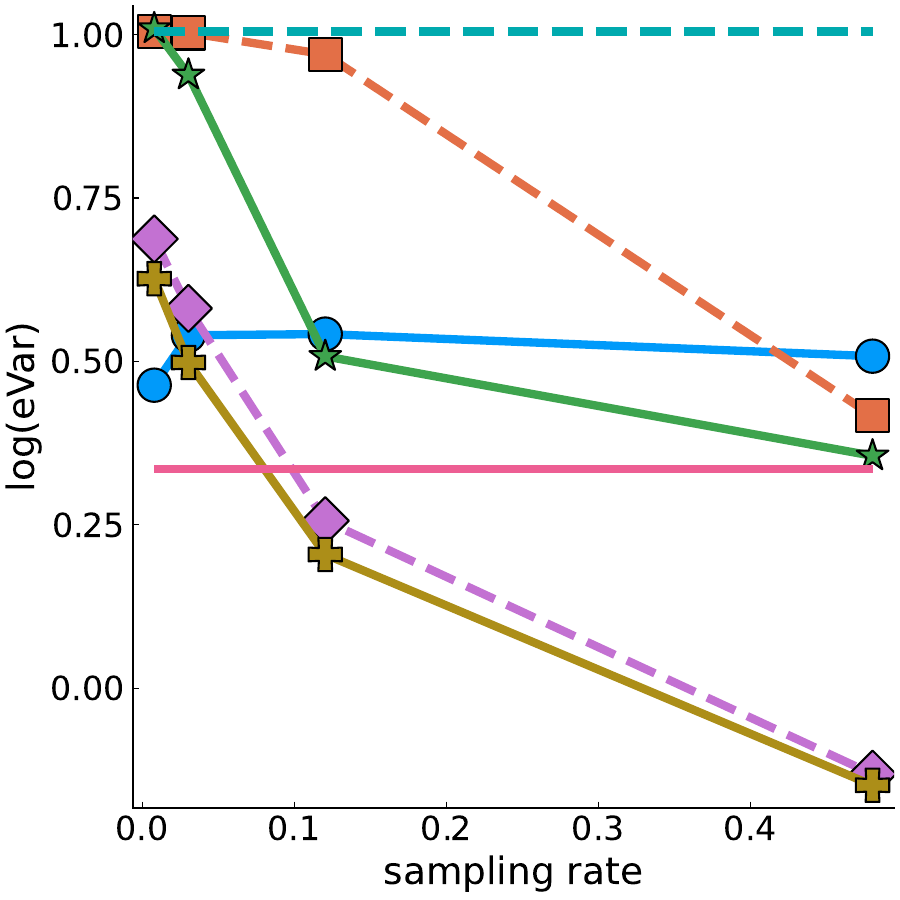}
    \caption*{eVar}
  \end{subfigure}
  \begin{subfigure}{0.12\textwidth}
    \includegraphics[width=\textwidth]{leg3.pdf}
  \end{subfigure}
  \caption{The eMSE, eBias$^2$, and eVar of transfer learning estimators with
    $\ell_2$ penalty for Case ``HL'' with heavy-tailed covariates $\x$. The
    horizontal dashed line represents the least squares estimator
    ($\hbeta_{\cS}$) with only target data $\cS$. The horizontal solid line
    represents the transfer learning estimator ($\hbeta_{\mathrm{full}}$) with
    full data $\cS\cup\cB$.}
  \label{fig:corL2}
\end{figure}

Figure~\ref{fig:corL2} shows that $\hbeta_{\rw}$ with the $\ell_2$ penalty is
unstable for the Case ``HL,'' where bias terms are proportional to the leverage
scores of heavy-tailed covariates. In contrast, this instability does not occur
with the $\ell_1$ penalty, showing its superiority over the $\ell_2$ penalty in
cases where biases are related to heavy-tailed covariates.
To directly compare the $\ell_1$ and $\ell_2$ penalties, we present the results
from the two penalties side-by-side in Figure~\ref{fig:corL1L2} for the Case
``HL'' with heavy-tailed covariates $\x$. For clearer presentation, we only
plot $\hbetad$, $\hbetaw$, $\hbeta_{\comdat}$, and $\hbeta_{\mathrm{full}}$. It
is evident that estimators with the $\ell_1$ penalty are more efficient, except
for $\hbetad$. Note that $\hbetad$ does not use external data points with large
outliers with a large probability, whereas $\hbetaw$ almost always
includes external data points with large outliers. This suggests that the
$\ell_1$ penalty is more effective than the $\ell_2$ penalty in handling large
outliers in transfer learning. However, we should clarify that the $\ell_2$
penalty has an advantage over the $\ell_1$ penalty in terms of computational
simplicity.

\begin{figure}[htp]
  \centering
  \begin{subfigure}{0.490\textwidth}
    \includegraphics[width=\textwidth]{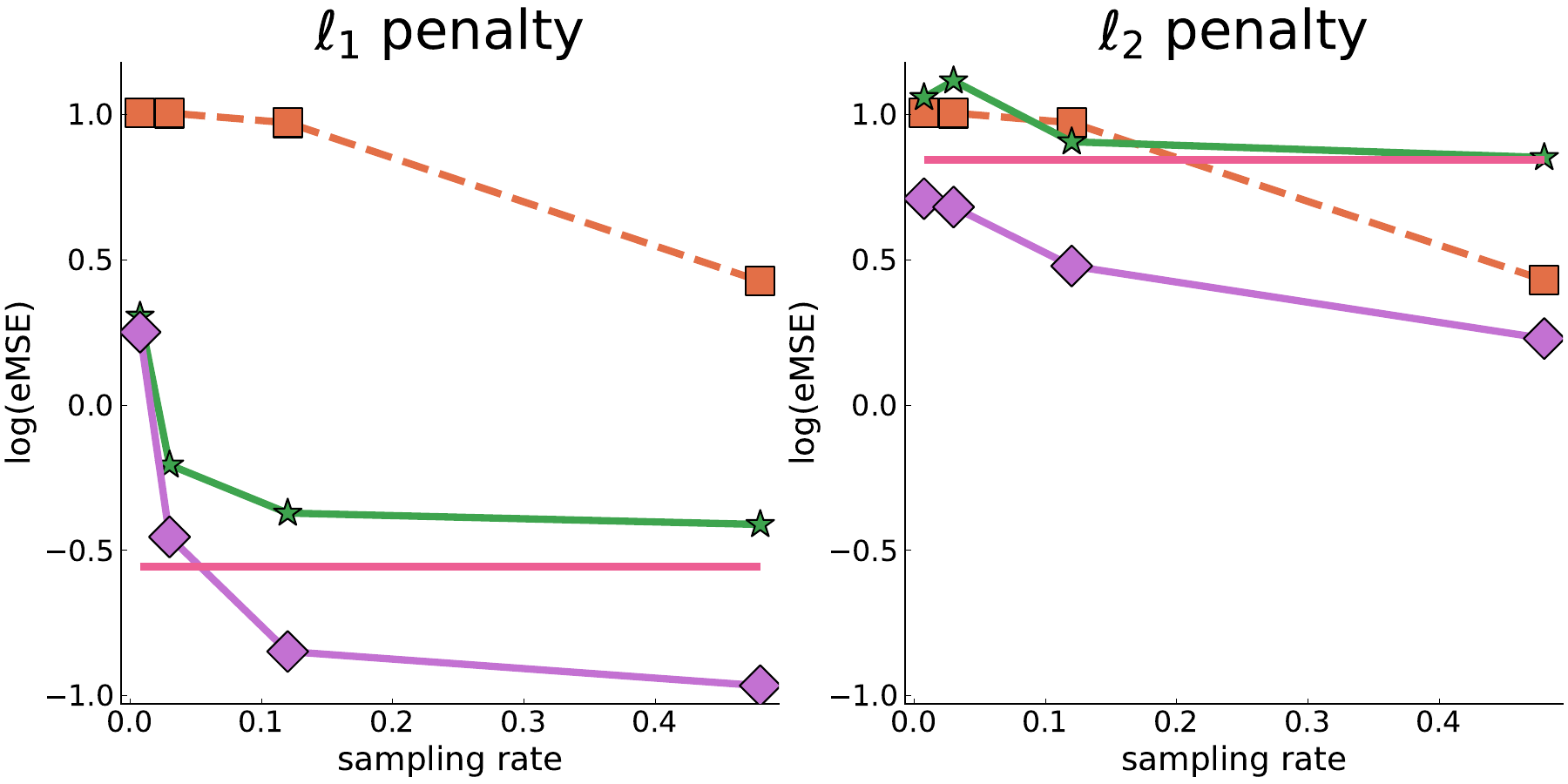}
    \caption*{eMSE}
  \end{subfigure}
  \begin{subfigure}{0.12\textwidth}
    \includegraphics[width=\textwidth]{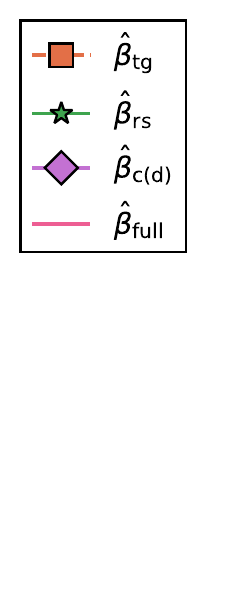}
  \end{subfigure}
  \begin{subfigure}{0.48\textwidth}
    \includegraphics[width=\textwidth]{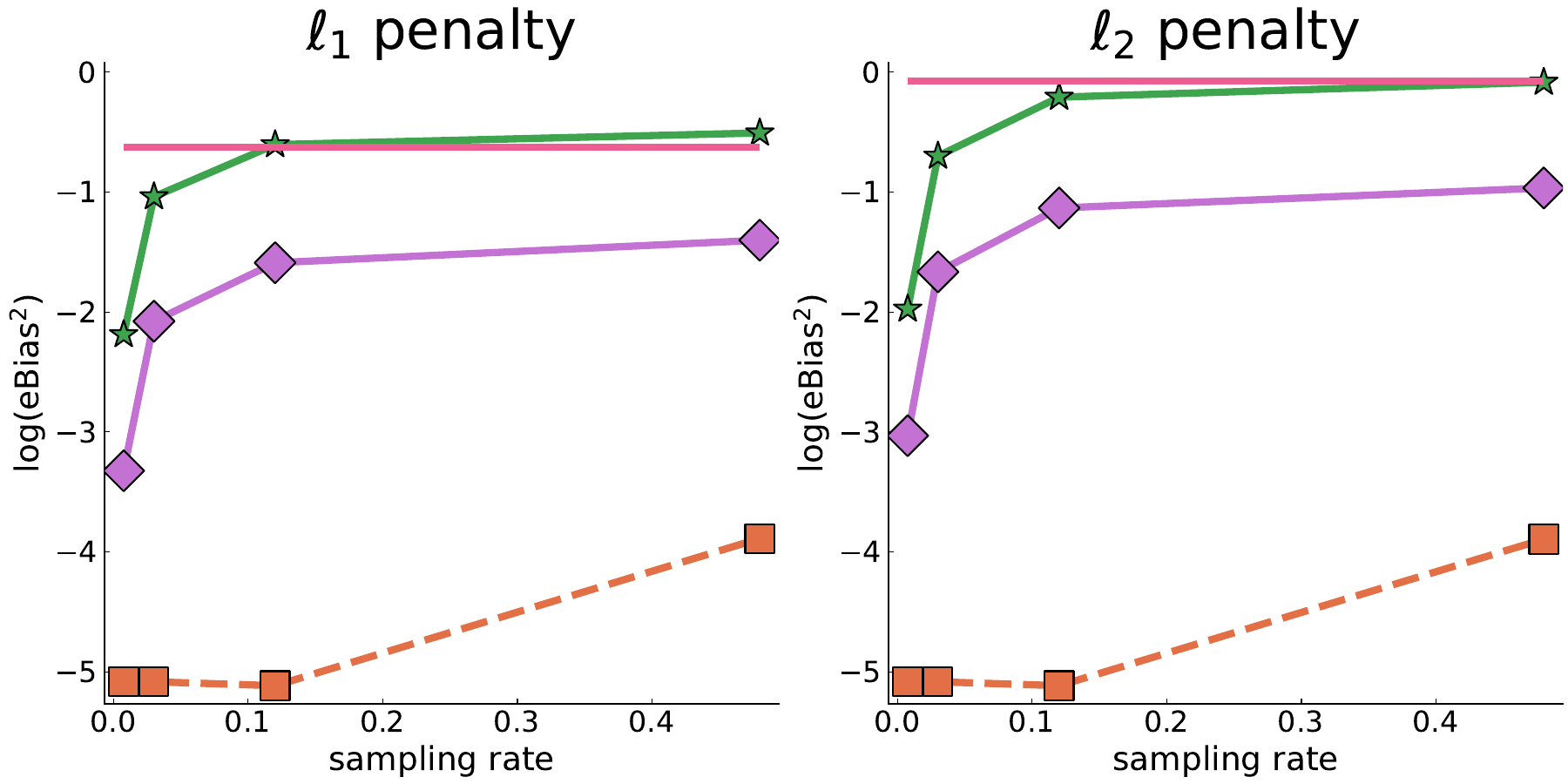}
    \caption*{eBias$^2$}
  \end{subfigure}
  \hspace{0.01\textwidth}
  \begin{subfigure}{0.48\textwidth}
    \includegraphics[width=\textwidth]{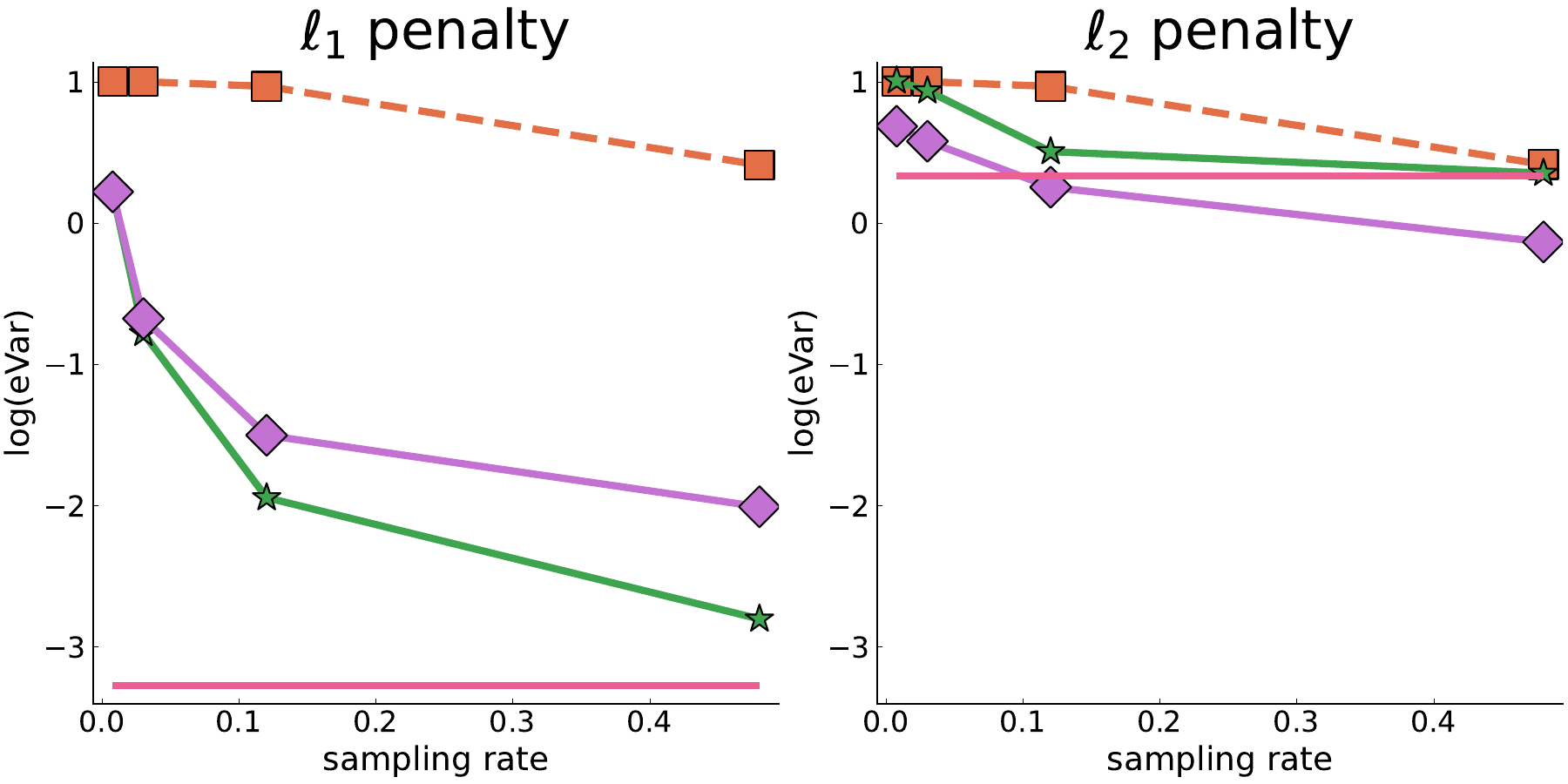}
    \caption*{eVar}
  \end{subfigure}
  \caption{The eMSE, eBias$^2$, and eVar of transfer learning estimators with
    $\ell_1$ and $\ell_2$ penalties for Case ``HL'' with heavy-tailed
    covariates $\x$.}
  \label{fig:corL1L2}
\end{figure}

\subsection{Real Data}\label{sec:realdata}

We apply our proposed methods to predict airplane hard landings using quick
access recorder (QAR) data. Our dataset includes three types of airplanes:
Airbus A380 (A380), Airbus A320 (A320), and Boeing 737 (B737). Each data point
contains 102 covariates extracted from the QAR records. We consider the maximum
vertical acceleration during landing as a numerical measure of hard landings.
The dataset consists of 1,218 observations for A380; 80,814 observations for
A320; and 68,986 observations for B737.  Our goal is to predict the maximum
vertical accelerations during landing for A380 aircrafts using a linear
regression model with other extracted covariates. Accordingly, we treat the A380
data as the target data and merge the A320 and B737 data as the external data.
Since A320 and B737 share similarities with A380 but are still distinct, there
are potential biases in using a model trained on A320 and B737 data to predict
hard landings for A380. We employ the proposed transfer learning methods to
address this.

We repeat our experiments $K=500$ times, In each iteration, we randomly split
the target data into training data and test data in each iteration. Two
splitting percentages are considered: 50\%-50\% and 80\%-20\%, which correspond
to the training target data size $\Ns=609$ and $\Ns=975$, respectively.

In each iteration, we use marginal t-tests to find important covariates that
affect the maximum vertical acceleration based on the training data by
controlling the false discovery rate (FDR) to be 0.1. Marginal t-tests select
7.332 variables on average for the 50\%-50\% split and 8.496 variables on
average for the 80\%-20\% split. The two splitting strategy select the same
seven variables with the highest frequencies. We present them in the following
Table~\ref{tb:slvar1}.

\begin{table}[H]
\centering
\caption{Variables that are selected with most frequency}
\label{tb:slvar1}
\begin{tabular}{ccc}\hline
  \textbf{Code} & \textbf{Definition} & \textbf{Type} \\\hline
  x1102 & Horizontal Speed at Touchdown & Speed \\
  x1103 & Vertical Speed at Touchdown & Speed \\
  x3093 & Max Vertical Acceleration during Flight & Acceleration \\
  x3089 & Time during Touchdown to Start of All Engine Reverser Deployed & Time\&Dist \\
  x1115 & Wind Speed at Touchdown & External \\
  x3104 & Max Vertical Acceleration during Pre Take Off & Acceleration \\
  x3385 & Max Absolute Lateral Acceleration during Landing & Acceleration \\
  x3265 & Total Ground Distance during Taxi Out & Time\&Dist\\\hline
\end{tabular}
\end{table}

The true parameters are unknown for real data. We evaluate the performances of
the transfer learning estimators by the empirical trimmed prediction error
(eMSPE) which is defined as
\begin{equation*}
\mathrm{eMSPE}(\hbeta)=\text{Trim}\left(\left\{
\frac{1}{\#\mathcal{S}_{\mathrm{test}}}
\sum_{i\in\mathcal{S}_{\mathrm{test}}}(y_i-\x_i\tp\hbeta_k)^2:k=1,\ldots,K\right\}\right),
\end{equation*}
where $\mathcal{S}_{\mathrm{test}}$ represents the index set of the test data.

For combined estimators, we try different proportions of the external data
selected by random subsampling and target-guided. 
Specifically, letting $c$ be
the proportion of observations selected by the target-guided selection, we consider $c=0.1,0.3,0.5,0.7,0.9$. We present the
results of $c=0.1,0.3$ in the main paper and leave other results in the
supplementary material to save space. 
Table~\ref{tb:real} gives the percentage 
improved over the estimator using only the target data, defined as
\begin{equation*}
\frac{\mathrm{eMSPE}(\hbeta_{\cS})-\mathrm{eMSPE}(\hbeta)}{\mathrm{eMSPE}(\hbeta_{\cS})}
\times 100\%
\end{equation*}
for an estimator $\hbeta$.

\begin{table}[H]
\caption{Percentages improved on eMSPE of transfer learning estimators over
$\hbeta_{\cS}$.}
\label{tb:real}
\centering
\begin{tabular}{l *{7}{w{c}{12mm}} }\hline
       & \multicolumn{7}{c}{50\%-50\% split of target data}\\
       & \multicolumn{3}{c}{} & \multicolumn{2}{c}{$c=0.1$} & \multicolumn{2}{c}{$c=0.3$}\\\hline
  $r$ & $\hbeta_{\mathrm{uni}}$ & $\hbeta_{\dtm}$
  & $\hbeta_{\rw}$ & $\hbeta_{\comdat}$ & $\hbeta_{\comest}$
  & $\hbeta_{\comdat}$ & $\hbeta_{\comest}$ \\\hline
  $\Ns$  & $+1.56$ & $-0.00$ & $+1.78$ & $+1.65$ & $+1.50$ & $+1.33$ & $+1.30$ \\
  $2\Ns$ & $+1.39$ & $+0.00$ & $+1.53$ & $+1.52$ & $+1.64$ & $+1.38$ & $+1.54$ \\
  $4\Ns$ & $+0.62$ & $+0.00$ & $+0.69$ & $+1.01$ & $+1.26$ & $+1.14$ & $+1.43$ \\
  $8\Ns$ & $-0.44$ & $+0.01$ & $-0.07$ & $+0.66$ & $+0.48$ & $+0.87$ & $+0.85$ \\\hline
\end{tabular}
\begin{tabular}{l *{7}{w{c}{12mm}} }\hline
       & \multicolumn{7}{c}{80\%-20\% split of target data}\\
       & \multicolumn{3}{c}{} & \multicolumn{2}{c}{$c=0.1$} & \multicolumn{2}{c}{$c=0.3$}\\\hline
  $r$ & $\hbeta_{\mathrm{uni}}$ & $\hbeta_{\dtm}$
  & $\hbeta_{\rw}$ & $\hbeta_{\comdat}$ & $\hbeta_{\comest}$
  & $\hbeta_{\comdat}$ & $\hbeta_{\comest}$ \\\hline
  $\Ns$  & $+1.04$ & $-0.00$ & $+1.11$ & $+1.04$ & $+0.97$ & $+0.89$ & $+0.87$ \\
  $2\Ns$ & $+0.66$ & $+0.00$ & $+0.69$ & $+0.74$ & $+0.97$ & $+0.78$ & $+0.98$ \\
  $4\Ns$ & $-0.23$ & $+0.00$ & $-0.20$ & $+0.16$ & $+0.47$ & $+0.52$ & $+0.70$ \\
  $8\Ns$ & $-1.51$ & $+0.01$ & $-1.14$ & $-0.23$ & $-0.49$ & $+0.48$ & $+0.01$ \\\hline
\end{tabular}
\end{table}

From Table~\ref{tb:real}, we observe that transfer learning estimators generally
result in positive transfer. When the target training data is small in the
50\%-50\% split, transfer learning estimators provide more significant
improvements over the estimators that rely solely on target data, compared with
the 80\%-20\% split. This suggests that transfer learning estimators are
especially advantageous when target data itself is limited and insufficient for
reliable prediction. We also see that although $\hbetaw$ might perform well when
the sampling rate is low, the combined estimators are more robust and less
likely to result in negative transfer. For example, the combined estimators
$\hbeta_{\comdat}$ and $\hbeta_{\comest}$ perform well even when we choose $c$
as small as $c=0.1$ for the 50\%-50\% split, while estimators
$\hbeta_{\mathrm{uni}}$ and $\hbetaw$ may result in negative transfer when the
sampling rate is high.

\section{Conclusion}\label{sec:conclusion}

This paper addressed the challenges of transfer learning and subsampling in the
presence of data contamination. We focused on scenarios where external data
contain outliers resulting from arbitrary mean shifts, and investigated robust
subsampling strategies to mitigate their influence. Two distinct subsampling
approaches were introduced: one designed to reduce model bias caused by
outliers, and the other to minimize variance.

Our theoretical analysis provided valuable insights into the interplay of
several key factors, including the sizes of the subsamples, the sampling rates,
the magnitude and frequency of outliers, and the tail behavior of the model
error distribution. Importantly, our analysis demonstrated that, under specific
conditions, estimators based on carefully selected subsamples can outperform
those using the full, contaminated data. This highlights the importance of
strategic subsampling in robust data fusion. To further enhance performance, we
proposed methods that combine these bias-reduction and variance-reduction
strategies. Numerical results demonstrate that our approach can effectively
leverage external data to improve the efficiency of target data analysis.

Future research will explore extending these methodologies to nonlinear models,
such as logistic regression models. A key challenge lies in defining appropriate
target-guidance criteria for data selection in the nonlinear context, moving
beyond the residual-based approach used for linear models. Furthermore, we plan
to investigate the application of our methods to high-dimensional datasets,
where the challenges of data contamination and model complexity are amplified.

\appendix
\renewcommand{\theequation}{S.\arabic{equation}}
\renewcommand{\thefigure}{S.\arabic{figure}}
\renewcommand{\thelemma}{S.\arabic{lemma}}
\renewcommand{\thetable}{S.\arabic{table}}

\section{Supplementary Material}
This supplementary material contains the proofs of all theorems and additional
numerical results.

\subsection{Mathematical Details of Proposition~\ref{pro:thr} and Proposition~\ref{pro:thr2}}
In this section, we prove Proposition~\ref{pro:thr} and
Proposition~\ref{pro:thr2} simultaneously due to their close connection.

\begin{proof}[\textbf{Proof of Proposition~\ref{pro:thr} and Proposition~\ref{pro:thr2}}]
The proof is similar to the proof of Proposition 4.1
in~\cite{she2011outlier}. To prove the results, we first show that the minimizer
defined in~\eqref{eq:ftrans} satisfies the K-K-T condition of the objective
function~\eqref{eq:fhuber}. Then, we will prove that the fixed point of
iterative equation in~\eqref{eq:itgamma} and~\eqref{eq:itbeta} also satisfies the
K-K-T condition. Note that due to the definition of thresholding function
$\Theta(t;\lambda)$, we know that the minimizer $(\hbeta\tp,\bgamma_{\cBs}\tp)\tp$
defined in~\eqref{eq:ftrans} should satisfies the following equations:
\begin{align*}
\V\hbeta=\Xs\tp\ys+\Xbs\tp\W_{\cBs}(\ybs-\hgamma_{\cBs}),
\text{ and, }\hgamma_{\cBs}=\Theta(\ybs-\Xbs\hbeta;\lambda).
\end{align*}
Therefore, by noticing that $\Theta(\ybs-\Xbs\hbeta;\lambda)
=(\ybs-\Xbs\hbeta)-\psi\left(\ybs-\Xbs\hbeta;\lambda\right)$,
we have that
$\Xbs\tp\W_{\cBs}\Xbs\hbeta+\Xbs\tp\W_{\cBs}\psi\left(\ybs-\Xbs\hbeta;\lambda\right)
=\Xbs\tp\W_{\cBs}(\yb-\hgamma_{\cBs})=\V\hbeta-\Xs\tp\ys$.
Since $\V=\V_{\cS}+\V_{\cB}^w$ where
$\V_{\cS}=\X_{\cS}\tp\X_{\cS}$ and $\V_{\cB}^w=\X_{\cBs}\tp\W_{\cBs}\X_{\cBs}$, by
rearranging the terms, we
obtain that the minimizer $(\hbeta\tp,\bgamma_{\cBs}\tp)\tp$ defined
in~\eqref{eq:ftrans} satisfies 
\begin{equation}\label{eq:kkt}
\Xs\tp(\ys-\Xs\hbeta)+\Xbs\tp\W_{\cBs}\psi\left(\ybs-\Xbs\hbeta;\lambda\right)=0,
\end{equation}
which is the K-K-T condition of~\eqref{eq:fhuber} by the definition of
$\mathcal{H}(t;\lambda)$. Next, we prove that the fix point of iterative equation
in~\eqref{eq:itgamma} and~\eqref{eq:itbeta} also satisfies~\eqref{eq:kkt}. 
We know that $\hgamma_{\cBs}$ is the fixed point of the following equation 
\begin{align}\label{eq:fpgamma}
\hgamma_{\cBs}
=\Theta\left\{(\I-\Xbs\V^{-1}\Xbs\tp\W_{\cBs})\ybs
+\Xbs\V^{-1}\Xbs\tp\W_{\cBs}\hgamma_{\cBs}-\Xbs\V^{-1}\Xs\tp\ys;\lambda\right\},
\end{align}
and $\hbeta=\V^{-1}\Xs\tp\ys+\V^{-1}\Xbs\tp\W_{\cBs}(\ybs-\hgamma_{\cBs})$. Therefore,
direct calculation gives 
\begin{align*}
&\Xs\tp(\ys-\Xs\hbeta)+\Xbs\tp\W_{\cBs}\psi\left(\ybs-\Xbs\hbeta;\lambda\right)\\
&=\Xs\tp\ys-\V_{\cS}\V^{-1}\Xs\tp\ys-\V_{\cS}\V^{-1}\Xbs\tp\W_{\cBs}(\ybs-\hgamma_{\cBs})\\
&\quad+\Xbs\tp\W_{\cBs}\ybs-\V_{\cBs}^w\V^{-1}\Xs\tp\ys-\V_{\cBs}^w\V^{-1}\Xbs\tp\W_{\cBs}(\ybs-\hgamma_{\cBs})\\
&\quad-\Xbs\tp\W_{\cBs}\Theta\left\{\ybs-\Xbs\V^{-1}\Xs\tp\ys
-\Xbs\V^{-1}\Xbs\tp\W_{\cBs}(\ybs-\hgamma_{\cBs})
;\lambda\right\}\\
&=\Xs\tp\ys-\Xs\tp\ys-\Xbs\tp\W_{\cBs}(\ybs-\hgamma_{\cBs})+\Xbs\tp\W_{\cBs}\ybs-\Xbs\tp\W_{\cBs}\hgamma_{\cBs}=\0,
\end{align*}
which complete the proof.
\end{proof}
To ease the notation, we denote
$\cL(\bbeta)=\Nb^{-1}\sumcb\mathcal{H}(y_e-\x_e\tp\bbeta;\lambda)$ and
\begin{equation*}
\cLd(\bbeta)=\frac{1}{\rd}\sumcb\delta_e\mathcal{H}(y_e-\x_e\tp\bbeta;\lambda),
\cLp(\bbeta)=\frac{1}{r}\sumcb\frac{\rho\delta_e}{\pi_e}
\mathcal{H}(y_e-\x_e\tp\bbeta;\lambda),
\end{equation*}
in the rest of the appendix.

\subsection{The toy model to illustrate the drawback of OSMAC}\label{sec:toy}

We use a simulation on a toy model to illustrate the potential inefficient
  performances of existing optimal subsampling probabilities. We generate target
data with size 20 and external data with size 2000. The covariates $\x$ of the
target and external data are 9-dimensional random vectors generated from normal
distribution with variance 4. The covariates are independent. The parameter
$\tbeta$ is a 9-dimensional vector with elements equal to 1. For both the target and
the external data, the error terms are generated from normal distribution with
variances equal to 1 and we randomly select 25\% data points of the external
data to be contaminated by adding a bias terms drawn from $\Gamma(3,3)$. We
compute the transfer learning estimator based on~\eqref{eq:rtrans}. We compare
three $\hbeta$, one only uses the target data, one uses data uniformly selected
from the external data and the another uses data randomly selected from
OSMAC. The results are summarized into the following
Table~\ref{tb:badexample}. We use empirical MSE as a measure of performances and
the numbers in the brackets are the percentage of selected data points that are
contaminated. 
\begin{table}[H]
\centering
\caption{Inefficient performance of OSMAC}
\label{tb:badexample}
\begin{tabular}{cccccc}
  \hline
  \textbf{Subsample size} & \textbf{target} & \textbf{full ($100\Ns$)} & \textbf{uni} & \textbf{OSMAC} \\\hline
  $\Ns$  & 0.369 & 2.245 & 0.288 (0.256) & 0.770 (0.792) \\
  $2\Ns$ & 0.369 & 2.245 & 0.239 (0.248) & 0.633 (0.790) \\
  $4\Ns$ & 0.369 & 2.245 & 0.291 (0.253) & 0.641 (0.790) \\
  $8\Ns$ & 0.369 & 2.245 & 0.371 (0.252) & 0.704 (0.789) \\\hline
\end{tabular}
\end{table}
As shown in Table~\ref{tb:badexample}. Existing optimal sampling method
  performs even worse than uniform sampling and estimator without
  transfer learning. The main reason is that OSMAC include too many data points
  with biases. For uniform sampling, only 25\% of the selected data are
  contaminated which is similar to the true contaminated rate. However, there
  are about 79\% data points selected by OSMAC are contaminated, which causes the
  OSMAC performs even worse than the estimator with no transfer learning.

\subsection{Mathematical Details of Section~\ref{sec:rtrans}}

In this section, we present the detailed proofs of theoretical analysis in
Section~\ref{sec:rtrans}. We begin with some lemmas for the later proofs in
Section~\ref{sec:rlem}. We present the proof of Theorem~\ref{thm:rtrans} in
Section~\ref{sec:prf-rtrans1} and We the proof of Theorem~\ref{thm:rtrans2} in
Section~\ref{sec:prf-rtrans2}. In Section, we present the proofs of
Corollary~\ref{cor:norm} and Proposition~\ref{pro:optpi}.

\subsubsection{Lemmas for the proofs}\label{sec:rlem}

\begin{lemma}\label{lem:lemofS}
Under Assumption~\ref{asm:xsubg}, for any $n>0$, $\rho>0$ and $0<\pi_i\le 1$,
$1\le i\le n$, assuming $\sum_{i=1}^n\pi_i=\rho n$ $w_i=(\rho\delta_i)/\pi_i$,
$\min_{1\le i\le n}\pi_i\geq\mpi\rho$, and 
$\Var(\varepsilon_i)=\sigma^2$, where $\delta_i=I(u_i\le\pi_i)$ and
$u_i\sim U[0,1]$, we have that with probability at least
$1-e^{-\varsigma}-C/\sqrt{\rho n}$, $\sum_{i=1}^n\tx_i\tx_i\tp/\rho n\geq 1/8$
and $(\rho n)^{-\frac{1}{2}}\left\|\sum_{i=1}^n w_i\varepsilon_i\tx_i\right\|_2
\lesssim\sigma\sqrt{d+\varsigma}$, where $C$ is a constant that does not depend
on $n$. 
\end{lemma}

\begin{proof}

Note that for any $\bu\in\mathbb{S}^{d-1}$, where $\mathbb{S}^{d-1}$ represents
thee unit sphere in $\mathbb{R}^d$, we have that due to Chebyshev's
inequality, for any $c>0$, 
\begin{align*}
&\Pr\left(\left|\frac{1}{\rho n}\sum_{i=1}^n w_i
\bm{u}\tp\tx_i\tx_i\tp\bm{u}-1\right|>c\right)
\le\frac{\sum_{i=1}^n\Var(w_i\bm{u}\tp\tx_i\tx_i\tp\bm{u})}{(\rho n)^2c^2}\\
&=\frac{\sum_{i=1}^n\Exp\left\{\Var(w_i\bm{u}\tp\tx_i\tx_i\tp\bm{u}|\x_i)\right\}
+\sum_{i=1}^n\Var\left\{\Exp(w_i\bm{u}\tp\tx_i\tx_i\tp\bm{u}|\x_i)\right\}}{
(\rho n)^2c^2}\\
&\le\frac{\Exp\left\{(\bm{u}\tp\tx\tx\tp\bm{u})^2\right\}}{\rho\mpi nc^2}
+\frac{\Var(\bm{u}\tp\tx\tx\tp\bm{u})}{\rho nc^2}
\le\frac{\Exp(\|\tx\|_2^4)}{c^2(\mpi\mins 1)\rho n}.
\end{align*}
Thus, taking $c=7/8$, we have that with probability at least $1-C_1/\rho n$,
$\sum_{i=1}^n w_i\tx_i\tx_i\tp/\rho n\geq 1/8$, where $C_1$ does not depend on
$n$.  Denoting 
\begin{equation*}
\W_i=\left\{\Exp\left(\sum_{i=1}^n\frac{\varepsilon_i^2}{\pi_i}\right)\right\}^{-\frac{1}{2}}
\frac{w_i\varepsilon_i\tx_i}{\rho},
\end{equation*}
we have that $\Exp(\W_i)=\0$ and
$\Var(\sum_{i=1}^n\W_i)=\I_d$. %
Applying Theorem 1.1
in~\cite{raivc2019multivariate}, we have that 
\begin{align*}
&\Pr\left(\left\|\sum_{i=1}^n\W_i\right\|_2>\sqrt{d+8\varsigma}\right)
\le\Pr(\|\Nor(\0,\I_d)\|_2>\sqrt{d+8\varsigma})
+\frac{C_2d^{\frac{1}{4}}(\sqrt{\rho})^3}{(\sqrt{n})^3}
\sum_{i=1}^n\frac{\Exp(\|\delta_i\varepsilon_i\tx\|_2^3)}{\pi_i^3}  \\
&\le\Pr\left(\chi_d^2>d+\varsigma\right)
+\frac{C_2\Exp\{|\varepsilon|^3(\sqrt{\rho})^3\}
\Exp\left\{\|\tx\|_2^3\right\}}{\rho^2\mpi^2\sqrt{n}}
\le e^{-\varsigma}+\frac{C_3}{\sqrt{\rho n}}.
\end{align*}
Therefore, we know that with probability at least $1-e^{-\varsigma}-C_3/\sqrt{\rho n}$,
$\sum_{i=1}^n\tx_i\tx_i\tp/\rho n\geq 1/8$ and 
\begin{align*}
&\left\|\sum_{i=1}^n w_i\varepsilon_i\tx_i\right\|_2
=\left\|\left\{\Exp\left(\sum_{i=1}^n\frac{\varepsilon_i^2}{\pi_i}\right)\right\}^{\frac{1}{2}}
\rho\sum_{i=1}^n\W_i\right\|_2
\le\frac{\sqrt{\rho}}{\sqrt{\mpi}}
\left\|\left\{\Exp\left(\sum_{i=1}^n\varepsilon_i^2\right)\right\}^{\frac{1}{2}}
\sum_{i=1}^n\W_i\right\|_2\\
&=\frac{\sqrt{\rho}}{\sqrt{\mpi}}
\left\|\sqrt{n}\sigma\sum_{i=1}^n\W_i\right\|_2
\lesssim\sqrt{\rho n}\sigma\sqrt{d+\varsigma}.
\end{align*}
\end{proof}

\begin{lemma}\label{lem:psifull}
Under Assumption~\ref{asm:xsubg}, assuming that $\min_{e\in\cB}\pi_e\geq\rho\mpi$,
for any fixed %
$\gamma_1,\gamma_2,\ldots,\gamma_{\Nb}$ and $0<\alpha_1,\alpha_2\le
1$, there exists a constant $C$ does not depend on $\Nb$, such that
\begin{align*}
\frac{1}{\rho\Nb}\left\|\sumcb\frac{\rho\delta_e}{\pi_e}\left\{\psi_1(\gamma_e+\varepsilon_e;\lambda)
-\psi_1(\varepsilon_e;\lambda)\right\}\tx_e\right\|_2
\le\frac{C\vw_{\gammae,\alpha_1}^{\wtd}}{\lambda^{\alpha_1}}
+\frac{C\lambda(d+\varsigma)}{\mpi\rho\Nb}
+C\lambda^{\alpha_2}\vuw_{\gammae,\alpha_2},
\end{align*}
with probability at least $1-e^{-\varsigma}$.
\end{lemma}
\begin{proof}
By defining $\bm{\xi}^{*}
=\Nb^{-1}\sumcb\left\{\xi_e\tx_e-\Exp(\xi_e\tx_e)\right\}$, where
$\xi_e=\pi_e^{-1}\delta_e\{\psi_1(\gamma_e+\varepsilon_e;\lambda)-\psi_1(\varepsilon_e;\lambda)\}$,
our
outline of this proof is to first bound $\bm{\xi}^{*}$ using Bernstein's
inequality and then bound $\Nb^{-1}\sumcb\Exp(\xi_e\tx_e)$. We start with the
bound of $\bm{\xi}^{*}$. For function $\psi_1(x;\lambda)$, it is easy to verify
that 
\begin{equation}\label{eq:boundpsi}
\sup_{x\in\mathbb{R}}|\psi_1(x+y;\lambda)-\psi_1(x;\lambda)|
=|y|I(|y|\le 2\lambda)+2\lambda I(|y|>2\lambda).
\end{equation}
According to the proof of the proof of Lemma 6 in~\cite{sun2020adaptive}, we
know that there exits a $1/2$-net 
$\mathcal{N}_{1/2}$ of the unit sphere $\mathbb{S}^{d-1}$ with
$\#\mathcal{N}_{1/2}\le 5^d$ such that $\|\bm{\xi}^{*}\|_2\le
2\max_{\bm{u}\in\mathcal{N}_{1/2}}|\bm{u}\tp\bm{\xi}^{*}|$. Under
Assumption~\ref{asm:xsubg}, since $\tx$ is sub-gaussian, for every
$\bu\in\mathbb{S}^{d-1}$, we have that $\Exp|\bu\tp\tx|^k\le
C_0^k k\Gamma(k/2)$ for all $k\geq 1$, where $\Gamma(\cdot)$ represents the Gamma
function. We also have that for $0<\alpha_1\le 1$,
$k\geq 2$, and any fixed $\gamma_1,\gamma_2,\ldots,\gamma_{\Nb}$, 
\begin{align*}
&\Exp\left|\frac{\delta_e}{\pi_e}
\left\{\psi_1(\gamma_e+\varepsilon_e;\lambda)-\psi_1(\varepsilon_e;\lambda)\right\}\right|^k
\le\frac{1}{\pi_e^{k-1}}\left\{|\gamma_e|^kI(|\gamma_e|\le 2\lambda)
+2^k\lambda^kI(|\gamma_e|>2\lambda)\right\}\\
&=\frac{1}{\pi_e^{k-1}}
\left\{|\gamma_e|^{k-1-\alpha_1}|\gamma_e|^{1+\alpha_1}I(|\gamma_e|\le 2\lambda)
+(2\lambda)^{k-1-\alpha_1}(2\lambda)^{1+\alpha_1}I(|\gamma_e|>2\lambda)\right\}\\
&=\frac{1}{\pi_e^{k-1}}(2\lambda)^{k-1-\alpha_1}|\gamma_e|^{1+\alpha_1}.
\end{align*}
The inequality holds for both cases that $\gamma_e=0$ and $\gamma_e\neq
0$. Therefore, we have that since
$\sumcb
\pi_e^{-1}|\gamma_e|^{1+\alpha_1}=\sum_{e\in\cO}\pi_e^{-1}
|\gamma_e|^{1+\alpha_1}=\Nb\vw_{\gammae,\alpha_1}^{\wtd}$,
\begin{align*}
&\sumcb\Exp\left\{|\xi_e\bm{u}\tp\tx_e|^k\right\}
\le C_0^kk\Gamma\left(\frac{k}{2}\right)
(2\lambda)^{k-1-\alpha_1}\sumcb\frac{|\gamma_e|^{1+\alpha_1}}{\pi_e^{k-1}}\\
&\le\frac{1}{2}\left(\frac{C_0}{\rho\mpi}\right)^{k-2}k\Gamma\left(\frac{k}{2}\right)(2\lambda)^{k-2}
2C_0^2(2\lambda)^{1-\alpha_1}\Nb\vw_{\gammae,\alpha_1}^{\wtd}
\le\frac{k!}{2}\left(\frac{C_0\lambda}{\rho\mpi}\right)^{k-2}
2C_0^2(2\lambda)^{1-\alpha_1}\Nb\vw_{\gammae,\alpha_1}^{\wtd},
\end{align*}
for any $k\geq 2$. A special case is when $k=2$, we have that
$\sumcb\Exp\left\{|\xi_e\bm{u}\tp\tx_e|^2\right\}
\le2C_0^2(2\lambda)^{1-\alpha_1}\Nb\vw_{\gammae,\alpha_1}^{\wtd}$. Now,
due to Bernstein's inequality, we have that 
\begin{align*}
\Pr\left\{|\bm{u}\tp\bm{\xi}^{*}|
\geq 2C_0\sqrt{\frac{(2\lambda)^{1-\alpha_1}
\vw_{\gammae,\alpha_1}^{\wtd}\cvsigma}{\Nb}}
+\frac{C_0\lambda\cvsigma}{\mpi\rho\Nb}\right\}
\le 2e^{-\cvsigma},
\end{align*}
for any $\cvsigma>0$. Taking union over $\bm{u}\in\mathcal{N}_{1/2}$, we obtain that
with probability at least $1-2\times 5^de^{-\cvsigma}$,
\begin{equation*}
\|\bm{\xi}^{*}\|_2
\le 8C_0\sqrt{\frac{\lambda^{1-\alpha_1}
\vw_{\gammae,\alpha_1}^{\wtd}\cvsigma}{\Nb}}
+\frac{2C_0\lambda\cvsigma}{\mpi\rho\Nb}
\le 4C_0 \lambda^{-\alpha_1}\rho\vw_{\gammae,\alpha_1}^{\wtd}
+\frac{6C_0\lambda\cvsigma}{\mpi\rho\Nb}.
\end{equation*}
The last inequality is because $2\sqrt{(\lambda\cvsigma\Nb^{-1})
  (\lambda^{-\alpha_1}\vw_{\gammae,\alpha_1}^{\wtd})}\le\lambda\cvsigma
\rho^{-1}\Nb^{-1}+\lambda^{-\alpha_1}\rho\vw_{\gammae,\alpha_1}^{\wtd}$
and $r=\sumcb\pi_e\geq\rho\Nb\mpi=r\mpi$. Next, we bound
$\Nb^{-1}\sumcb\Exp(\xi_e\tx_e)$. %
Since 
\begin{align*}
&\Exp|\xi_e|\le |\gamma_e|I(|\gamma_e|\le 2\lambda)
+2\lambda I(|\gamma_e|> 2\lambda)
\le(2\lambda)^{\alpha_2}|\gamma_e|^{1-\alpha_2}
+\frac{|\gamma_e|^{1+\alpha_1}}{(2\lambda)^{\alpha_1}\pi_e},    
\end{align*}
and $\Exp|\bu\tp\tx|\le 2C_0$, we have that
\begin{align*}
\sup_{\bm{u}\in\mathbb{S}^{d-1}}\frac{1}{\Nb}\sumcb\Exp
\left(|\xi_e\bm{u}\tp\tx_e|\right)
\le2C_0\lambda^{-\alpha_1}\vw_{\gammae,\alpha_1}^{\wtd}
+4C_0\lambda^{\alpha_2}\vuw_{\gammae,\alpha_2}.
\end{align*}
Now, by taking $\cvsigma=2(d+\varsigma)$ and $C=6C_0$, we
have that  
\begin{align*}
\frac{1}{\Nb}\left\|\sumcb\frac{\delta_e}{\pi_e}\left\{
\psi_1(\gamma_e+\varepsilon_e;\lambda)-\psi_1(\varepsilon_e;\lambda)\right\}
\tx_e\right\|_2
\le\frac{C\vw_{\gammae,\alpha_1}^{\wtd}}{\lambda^{\alpha_1}}
+\frac{C\lambda\cvsigma}{\mpi\rho\Nb}
+C\lambda^{\alpha_2}\vuw_{\gammae,\alpha_2}.
\end{align*}
with probability at least $1-2\times 5^de^{-2d-2\varsigma}\geq
1-2e^{-2\varsigma}\geq 1-e^{-\varsigma}$.
\end{proof}
\begin{lemma}\label{lem:convexfull}
Under Assumption~\ref{asm:xsubg}, if $\lambda,R$ satisfy $\lambda\gtrsim
(\vw_{\varepsilon,\alpha_1}^{\wtd})^{\frac{1}{1+\alpha_1}}\maxs 
(\vw_{\gammae,\alpha_1}^{\wtd})^{\frac{1}{1+\alpha_1}}\maxs(8C_1^2R)$
and $\Nb\gtrsim\left(\lambda/R\right)^2(d+\varsigma)$,
we will have that with probability at least $1-e^{-\varsigma}$,
$\left\{\nabla\cLp(\bbeta)-
\nabla\cLp(\tbeta)\right\}\tp(\bbeta-\tbeta)
\geq 8^{-1}\|\bbeta-\tbeta\|_{\bSigma,2}^2$,
uniformly over
$\bbeta\in\Theta_0(R)=\left\{\bbeta\in\mathbb{R}^d:\|\bbeta-\tbeta\|_{\bSigma,2}
\le R\right\}$.
\end{lemma}
\begin{proof}
We first state the outline of this proof. We consider a lower bound of
\begin{equation*}
\mathcal{T}(\bbeta):=\left\{\nabla\cLp(\bbeta)-
\nabla\cLp(\tbeta)\right\}\tp(\bbeta-\tbeta)
\end{equation*}
We decompose the lower bound into
an expectation part and a deviation part and then bound the two parts
separately. First, one of the lower bound of $\mathcal{T}(\bbeta)$ is
\begin{align*}
\mathcal{T}(\bbeta)
\geq\frac{1}{\Nb}\sumcb\frac{\delta_e}{\pi_e}
\left\{\psi_1(y_e-\x_e\tp\tbeta;\lambda)
-\psi_1(y_e-\x_e\tp\bbeta;\lambda)\right\}\x_e\tp(\bbeta-\tbeta)I(\mathcal{E}_e),
\end{align*}
where $\mathcal{E}_e=\left\{|\gamma_e+\varepsilon_e|\le\frac{\lambda}{2}\right\}
\cap\left\{|\x_i\tp(\bbeta-\tbeta)|\le\lambda\|\bbeta-\tbeta\|_{\bSigma,2}/2R\right\}$.
It is easy to verify that under event $\mathcal{E}_e$, we have
$|y_e-\x_e\tp\tbeta|\le\lambda$ and $|y_e-\x_e\tp\bbeta|\le\lambda$. Therefore,
$\mathcal{T}(\bbeta)\geq\Nb^{-1}\sumcb
\pi_e^{-1}\delta_e\left\{\x_e\tp(\bbeta-\tbeta)\right\}^2I(\mathcal{E}_e)$.  
Denoting $\phi_A(x)=x^2I(|x|\le A/2)+(x-A)^2I(A/2\le x\le A)+(x+A)^2I(-A\le
x\le-A/2)$
It is easy to verify that $x^2I\left(|x|\le 0.5A\right)\le\phi_A(x)\le
x^2I(|x|\le A)$ and $0\le\phi_A(x)\le 0.25A^2$.
Therefore, we have a lower bound of $\mathcal{T}(\bbeta)$:
\begin{equation*}
\mathcal{T}(\bbeta)
\geq g(\bbeta):=\frac{1}{\Nb}\sumcb g_e(\bbeta)
=\frac{1}{\Nb}\sumcb\frac{\delta_e}{\pi_e}
\phi_{\frac{\lambda\|\bbeta-\tbeta\|_{\bSigma,2}}{2R}}\left\{\x_e\tp(\bbeta-\tbeta)\right\}
I\left(|\varepsilon_e+\gamma_e|\le\frac{\lambda}{2}\right).
\end{equation*}
Next, we decompose the lower bound of $\mathcal{T}(\bbeta)$ into two parts, an
expectation part and a deviation part. Letting
$\Delta(r)=\sup_{\bbeta\in\Theta_0(r)}|g(\bbeta)-\Exp\left\{g(\bbeta)\right\}|
/\|\bbeta-\tbeta\|_{\bSigma,2}^2$ as the deviation part, we have that
\begin{equation*}
\frac{\mathcal{T}(\bbeta)}{\|\bbeta-\tbeta\|_{\bSigma,2}^2}
\geq\frac{\Exp\left\{g(\bbeta)\right\}}{\|\bbeta-\tbeta\|_{\bSigma,2}^2}
-\Delta(r).
\end{equation*}
We thus bound the two parts separately. We first bound
$\Exp\left\{g(\bbeta)\right\}$. Note that, denoting $\bmeta=\bbeta-\tbeta$, we
have that
\begin{align*}
&\Exp\left\{g(\bbeta)\right\}=\frac{1}{\Nb}\sumcb\Exp\left[\frac{\delta_e}{\pi_e}
\phi_{\frac{\lambda\|\bbeta-\tbeta\|_{\bSigma,2}}{2R}}
\left\{\x_e\tp(\bbeta-\tbeta)\right\}I\left(|\varepsilon_e+\gamma_e|
\le\frac{\lambda}{2}\right)\right]\\
&=\frac{1}{\Nb}\sum_{e\in\cB-\cO}\Exp\left[\phi_{\frac{\lambda\|\bmeta\|_{\bSigma,2}}{2R}}
(\x_e\tp\bmeta)I\left(|\varepsilon_e|
\le\frac{\lambda}{2}\right)\right]
+\frac{1}{\Nb}\sum_{e\in\cO}\Exp\left[\phi_{\frac{\lambda\|\bmeta\|_{\bSigma,2}}{2R}}
(\x_e\tp\bmeta)I\left(|\varepsilon_e+\gamma_e|
\le\frac{\lambda}{2}\right)\right]\\
&=\frac{1}{\Nb}\sumcb\Exp\left[\phi_{\frac{\lambda\|\bmeta\|_{\bSigma,2}}{2R}}
(\x_e\tp\bmeta)\right]
-\frac{1}{\Nb}\sum_{e\in\cB-\cO}\Exp\left[\phi_{\frac{\lambda\|\bmeta\|_{\bSigma,2}}{2R}}
(\x_e\tp\bmeta)I\left(|\varepsilon_e|
>\frac{\lambda}{2}\right)\right]\\
&\quad-\frac{1}{\Nb}\sum_{e\in\cO}\Exp\left[\phi_{\frac{\lambda\|\bmeta\|_{\bSigma,2}}{2R}}
(\x_e\tp\bmeta)I\left(|\varepsilon_e+\gamma_e|
>\frac{\lambda}{2}\right)\right]\\
&\geq\frac{1}{\Nb}\sumcb\Exp(\x_e\tp\bmeta)^2
 -\frac{1}{\Nb}\sumcb\Exp\left\{(\x_e\tp\bmeta)^2
I\left(|\x_e\tp\bmeta|>\frac{\lambda\|\bmeta\|_{\bSigma,2}}{4R}\right)\right\}\\
&\quad-\frac{1}{\Nb}\sum_{e\in\cB-\cO}\Exp\left\{
\left(\x_e\tp\bmeta\right)^2I\left(|\varepsilon_e|
>\frac{\lambda}{2}\right)\right\}
-\frac{1}{\Nb}\sum_{e\in\cO}\Exp\left\{
\left(\x_e\tp\bmeta\right)^2I\left(|\varepsilon_e+\gamma_e|
>\frac{\lambda}{2}\right)\right\}\\
&\geq\frac{1}{\Nb}\sumcb\Exp(\x_e\tp\bmeta)^2
 -\frac{1}{\Nb}\sumcb\Exp\left\{(\x_e\tp\bmeta)^2
I\left(|\x_e\tp\bmeta|>\frac{\lambda\|\bmeta\|_{\bSigma,2}}{4R}\right)\right\}\\
&\quad-\frac{1}{\Nb}\sum_{e\in\cB-\cO}\Exp\left\{
\left(\x_e\tp\bmeta\right)^2I\left(|\varepsilon_e|
>\frac{\lambda}{4}\right)\right\}
-\frac{1}{\Nb}\sum_{e\in\cO}\Exp\left\{
\left(\x_e\tp\bmeta\right)^2I\left(|\varepsilon_e|
>\frac{\lambda}{4}\right)\right\}\\
&\quad-\frac{1}{\Nb}\sum_{e\in\cO}\Exp\left\{
\left(\x_e\tp\bmeta\right)^2I\left(|\gamma_e|
>\frac{\lambda}{4}\right)\right\}\\
&\geq\bmeta\tp\bSigma\bmeta
-\left(\frac{4R}{\lambda}\right)^2\frac{\Exp(\x\tp\bmeta)^4}{\|\bmeta\|_{\bSigma,2}^2}
-\left(\frac{4}{\lambda}\right)^{1+\alpha_1}\frac{1}{\Nb}\sumcb\Exp(|\varepsilon_e|^{1+\alpha_1})%
\bmeta\tp\bSigma\bmeta-\left(\frac{4}{\lambda}\right)^{1+\alpha_1}
\vw_{\gammae,\alpha_1}^{\wtd}\bmeta\tp\bSigma\bmeta\\
&\geq\|\bmeta\|_{\bSigma,2}^2
\left\{1-\left(\frac{4}{\lambda}\right)^{1+\alpha_1}
\frac{1}{\Nb}\sumcb\Exp(|\varepsilon_e|^{1+\alpha_1})-\left(\frac{4}{\lambda}\right)^{1+\alpha_1}\vw_{\gammae,\alpha_1}^{\wtd}
-\left(\frac{4C_1^2R}{\lambda}\right)^2\right\}
\geq\frac{1}{4}\|\bmeta\|_{\bSigma,2}^2.
\end{align*}
as long as $\lambda\gtrsim
(\vw_{\varepsilon,\alpha_1}^{\wtd})^{\frac{1}{1+\alpha_1}}\maxs 
(\vw_{\gammae,\alpha_1}^{\wtd})^{\frac{1}{1+\alpha_1}}\maxs(8C_1^2R)$. Next,
we consider a bound of the deviation part $\Delta(R)$. %
Since $0\le I(|\varepsilon_e+\gamma_e|\le\lambda/2)\le 1$ and $\phi_A(x)\le
A^2/4$, we know that $0\le g_e(\bbeta)\le (4R)^{-2}(\rho\mpi)^{-1}
\lambda^2\|\bbeta-\tbeta\|_{\bSigma,2}^2$
Then, by Theorem 7.3 in~\cite{bousquet2003concentration}, we have that for any
$\varsigma>0$,
\begin{align*}
\Delta(R)\le\Exp\{\Delta(R)\}
+\Exp\{\Delta(R)\}^{1/2}\frac{\lambda}{2R}\sqrt{\frac{\varsigma}{\rho\mpi \Nb}}
+\sigma_{\Nb}^2\sqrt{\frac{2\varsigma}{\Nb}}
+\left(\frac{\lambda}{4R}\right)^2\frac{\varsigma}{3\rho\mpi \Nb},
\end{align*}
with probability at least $1-e^{-\varsigma}$ with 
\begin{align*}
\sigma_{\Nb}^2=\frac{1}{\Nb}\sumcb\sup_{\bbeta\in\Theta_0(R)}
\frac{\Exp\left\{g_e^2(\bbeta)\right\}}{\|\bmeta\|_{\bSigma,2}^4}
\le\frac{\Exp\left\{(\bmeta\tp\bSigma^{\frac{1}{2}}\tx_e)^4\right\}}{\rho^2\mpi^2
\|\bmeta\|_{\bSigma,2}^4}
\le \rho^{-2}\mpi^{-2}C_1^4.
\end{align*}
To bound $\Exp\left\{\Delta(R)\right\}$,
using symmetrization inequality and the connection between Gaussian complexity
and Rademacher complexity, we have $\Exp\left\{\Delta(R)\right\}\le\sqrt{2\pi}
\rho^{-1}\mpi^{-1}\Exp\left\{\sup_{\bbeta\in\Theta_0(R)}|\mathbb{G}_{\bbeta}|\right\}$, 
where 
\begin{equation*}
\mathbb{G}_{\bbeta}=\frac{1}{\Nb}\sumcb
\frac{G_e}{\|\bmeta\|_{\bSigma,2}^2}\phi_{\frac{\lambda\|\bmeta\|_{\bSigma,2}^2}{2R}}
(\x_e\tp\bmeta)I\left(|\varepsilon_e+\gamma_e|\le\frac{\lambda}{2}\right),
\end{equation*}
and $G_e$ are i.i.d. standard normal random variables that are independent of
$\{\varepsilon_e,\x_e\}_{b=1}^{\Nb}$. Thus, using exactly the same argument in the
proof of Lemma 4 in~\cite{sun2020adaptive}, we have that
$\Exp\{\Delta(R)\}\le\sqrt{2\pi}\rho^{-1}\mpi^{-1}\left\{2\lambda R^{-1}d^{\frac{1}{2}}\Nb^{-\frac{1}{2}}
+\lambda(4R)^{-1}\Nb^{-\frac{1}{2}}\right\}$.
Therefore, as long as we take $\Nb\gtrsim(\lambda/R)^2(d+\varsigma)$, we have that
$\Delta(R)\le 1/8$ and thus, with probability at least $1-e^{-\varsigma}$,
\begin{equation*}
\frac{\mathcal{T}(\bbeta)}{\|\bbeta-\tbeta\|_{\bSigma,2}^2}
\geq\frac{\Exp\left\{g(\bbeta)\right\}}{\|\bbeta-\tbeta\|_{\bSigma,2}^2}
-\Delta(r)\geq\frac{1}{4}-\frac{1}{8}
=\frac{1}{8}.
\end{equation*}

\end{proof}

\begin{lemma}\label{lem:psirand2}
Under Assumption~\ref{asm:xsubg}, assuming that $\min_{e\in\cB}\pi_e\geq\mpi\rho$,
for any $0<\alpha_1\le 1$,
there exists a constant $C$ does not depend on $\Nb$, such that 
\begin{align*}
\frac{1}{\Nb}\left\|\bSigma^{-\frac{1}{2}}
\sumcb\frac{\delta_e}{\pi_e}\psi_1(\varepsilon_e;\lambda)\x_e\right\|_2
\le C\lambda^{-\alpha_1}\vw_{\varepsilon,\alpha_1}^{\wtd}
+\frac{C\lambda(d+\varsigma)}{\mpi\rho \Nb}
\end{align*}
with probability at least $1-e^{-\varsigma}$.
\end{lemma}
\begin{proof}
By defining $\bm{\xi}^{*}
=\Nb^{-1}\sumcb\left\{\xi_i\tx_i-\Exp(\xi_i\tx_i)\right\}$, where
$\xi_e=\pi_e^{-1}\delta_e\psi_1(\varepsilon_e;\lambda)$. 
Similar to the proof of Lemma~\ref{lem:psifull}, there exits a $1/2$-net 
$\mathcal{N}_{1/2}$ of the unit sphere $\mathbb{S}^{d-1}$ with
$\#\mathcal{N}_{1/2}\le 5^d$ such that $\|\bm{\xi}^{*}\|_2\le
2\max_{\bm{u}\in\mathcal{N}_{1/2}}|\bm{u}\tp\bm{\xi}^{*}|$. Therefore, for
$0<\alpha_1\le 1$, $k\geq 2$,%
\begin{align*}
&\Exp\left|\frac{\delta_e}{\pi_e}\psi_1(\varepsilon_e;\lambda)\right|^k
=\frac{1}{\pi_e^{k-1}}
\Exp\left|\psi_1(\varepsilon_e;\lambda)\right|^k\\    
&\le\frac{1}{\pi_e^{k-1}}\Exp\left\{|\varepsilon_e|^kI(|\varepsilon_e|\le \lambda)
+\lambda^kI(|\varepsilon_e|>\lambda)\right\}
\le\frac{1}{\pi_e^{k-1}}\lambda^{k-1-\alpha_1}|\varepsilon_e|^{1+\alpha_1}.
\end{align*}
The inequality holds for both cases that $\gamma_e=0$ and $\gamma_e\neq
0$. Therefore, we have that  
\begin{align*}
\sumcb\Exp\left\{|\xi_e\bm{u}\tp\tx_e|^k\right\}
&\le C_0^kk\Gamma\left(\frac{k}{2}\right)\sumcb
\frac{\Exp(|\varepsilon_e|^{1+\alpha_1}|\x_i)}{\pi_e^{k-1}}\lambda^{k-1-\alpha_1}\\
&\le\frac{k!}{2}\left(\frac{C_0\lambda}{\rho\mpi}\right)^{k-2}
2C_0^2\lambda^{1-\alpha_1}\mpi^{-1}\Nb\vw_{\varepsilon,\alpha_1}^{\wtd},
\end{align*}
for any $k\geq 2$. A special case is when $k=2$, we have that
$\sumcb\Exp\left\{|\xi_e\bm{u}\tp\tx_e|^2\right\}\le
2C_0^2\lambda^{1-\alpha_1}\mpi^{-1}\Nb\vw_{\varepsilon,\alpha_1}^{\wtd}$. Now due to
Bernstein's inequality, we have  
\begin{align*}
\Pr\left\{|\bm{u}\tp\bm{\xi}^{*}|
\geq 2C_0\sqrt{\frac{\lambda^{1-\alpha_1}
  \vw_{\varepsilon,\alpha_1}^{\wtd}\cvsigma}{\mpi \Nb}}
+\frac{C_0\lambda\cvsigma}{\mpi\rho\Nb}\right\}
\le 2e^{-\cvsigma},
\end{align*}
for any $\cvsigma>0$. Taking union over $\bm{u}\in\mathcal{N}_{1/2}$, we obtain that
with probability at least $1-2\times 5^de^{-\varsigma}$,
\begin{equation*}
\|\bm{\xi}^{*}\|_2
\le 4C_0\sqrt{\frac{\lambda^{1-\alpha_1}
\vw_{\varepsilon,\alpha_1}^{\wtd}\cvsigma}{\mpi \Nb}}
+\frac{2C_0\lambda\cvsigma}{\mpi\rho\Nb}.
\le 2C_0\lambda^{-\alpha_1}\rho
\vw_{\varepsilon,\alpha_1}^{\wtd}+\frac{4C_0\lambda\cvsigma}{\mpi\rho\Nb}.
\end{equation*}
The last inequality is because $2\sqrt{(\lambda\cvsigma
  \Nb^{-1}\mpi^{-1})(\lambda^{-\alpha_1}\rho
  \vw_{\varepsilon,\alpha_1}^{\wtd})}\le\lambda\cvsigma\rho^{-1}
\Nb^{-1}\mpi^{-1}+\lambda^{-\alpha_1}\rho\vw_{\varepsilon,\alpha_1}^{\wtd}$. For
$\Nb^{-1}\sumcb\Exp(\xi_e\tx_e)$, due to (44) in~\cite{sun2020adaptive}, we have
that 
\begin{align*}
\sup_{\bu\in\mathbb{S}^{d-1}}\frac{1}{\Nb}\left|\sumcb\Exp
\left(\xi_e\bu\tp\tx_e\right)\right|
&=\sup_{\bu\in\mathbb{S}^{d-1}}\frac{1}{\Nb}\left|\sumcb\Exp
\left\{\psi_1(\varepsilon_e;\lambda)\bu\tp\tx_e\right\}\right|\\
&\le 2C_0\lambda^{-\alpha_1}\frac{1}{\Nb}\sumcb\Exp(|\varepsilon_e|^{1+\alpha_1}|\x_i)\\
&\le 2C_0\lambda^{-\alpha_1}\frac{1}{\Nb}\sumcb\Exp\left(\frac{|\varepsilon_e|^{1+\alpha_1}}{\pi_e}\Big|\x_i\right)
\le 2C_0\lambda^{-\alpha_1}\vw_{\varepsilon,\alpha_1}^{\wtd}.
\end{align*}
Therefore, by taking $\cvsigma=2(d+\varsigma)$ and $C=4C_0$, we have  
\begin{align*}
\frac{1}{\Nb}\left\|\bSigma^{-\frac{1}{2}}\sumcb\frac{\delta_e}{\pi_e}
\psi_1(\varepsilon_e;\lambda)\x_e\right\|_2
\le C\lambda^{-\alpha_1}\vw_{\varepsilon,\alpha_1}^{\wtd}
+\frac{C\lambda(d+\varsigma)}{\mpi\rho \Nb},
\end{align*}
with probability at least $1-2\times 5^de^{-2d-2\varsigma}\geq
1-2e^{-2\varsigma}\geq 1-e^{-\varsigma}$.
\end{proof}

\subsubsection{Proof of Theorem~\ref{thm:rtrans}}\label{sec:prf-rtrans1}

\begin{proof}[\textbf{Proof of Theorem~\ref{thm:rtrans}}]
For any prespecified $R>0$, we build an estimator
$\hbeta_{\eta}=\tbeta+\eta(\hbetaw-\tbeta)$, such that
$\|\hbeta_{\eta}-\tbeta\|\le R$. To be more specific, we take $\eta=1$ if
$\|\hbetaw-\tbeta\|_2\le R$; otherwise, we choose some $\eta\in(0,1)$ so that
$\|\hbeta_{\eta}-\tbeta\|_2=R$. Due to Lemma 2 in \cite{sun2020adaptive}, we know that
\begin{align*}
&\rho\Nb\left\{\nabla\cLp(\hbeta_{\eta})-\nabla\cLp(\tbeta)\right\}\tp(\hbeta_{\eta}-\tbeta)\\
&\le\eta\rho\Nb\nabla\cLp(\hbetaw)\tp(\hbetaw-\tbeta)
-\eta\rho\Nb\nabla\cLp(\tbeta)(\hbetaw-\tbeta)\\
&=-(\hbeta_{\eta}-\tbeta)\tp\Xbs\tp\W_{\cBs}\psi_1\left(\ybs-\Xbs\hbetaw;\lambda\right)
+(\hbeta_{\eta}-\tbeta)\tp\Xbs\tp\W_{\cBs}\psi_1\left(\ybs-\Xbs\tbeta;\lambda\right).
\end{align*}
We have that the K-K-T condition in~\eqref{eq:kkt} implies
\begin{align*}
&\rho\Nb\left\{\nabla\cLp(\hbeta_{\eta})-\nabla\cLp(\tbeta)\right\}\tp
(\hbeta_{\eta}-\tbeta)\\
&\le\eta(\hbetaw-\tbeta)\tp\Xs\tp(\ys-\Xs\hbetaw)
+\eta(\hbetaw-\tbeta)\tp\Xbs\tp\W_{\cBs}\psi_1\left(\ybs-\Xbs\tbeta;\lambda\right)\\
&=\eta(\hbetaw-\tbeta)\tp\Xs\tp(\Xs\tbeta+\epS-\Xs\hbetaw)
+\eta(\hbetaw-\tbeta)\tp\Xbs\tp\W_{\cBs}\psi_1\left(\ybs-\Xbs\tbeta;\lambda\right)\\
&=\eta(\hbetaw-\tbeta)\tp\Xs\tp\epS-(\hbeta_{\eta}-\tbeta)\tp\Xs\tp\Xs(\hbetaw-\tbeta)
+\eta(\hbetaw-\tbeta)\tp\sumcb\frac{\rho\delta_e}{\pi_e}\psi_1(\varepsilon_e+\gamma_e)\x_e.
\end{align*}
Rearranging the terms, we obtain that 
\begin{align}
\label{eq:rand1}
\begin{split}
&\rho\Nb\left\{\nabla\cLp(\hbeta_{\eta})-\nabla\cLp(\tbeta)\right\}\tp
(\hbeta_{\eta}-\tbeta)+(\hbeta_{\eta}-\tbeta)\tp\Xs\tp\Xs(\hbetaw-\tbeta)\\
&\le\eta(\hbetaw-\tbeta)\tp\Xs\tp\epS
+\eta(\hbetaw-\tbeta)\tp\sumcb\frac{\rho\delta_e}{\pi_e}\psi_1(\varepsilon_e+\gamma_e)\x_e.
\end{split}
\end{align}
Applying Lemma~\ref{lem:convexfull} and taking $R=\lambda/(8C_1^2)$, there
exists a constant $C_1$ that does not depend on sample sizes such that as long
as $\lambda\geq C_1\{(\vw_{\varepsilon,\alpha_1}^{\wtd})^{\frac{1}{1+\alpha_1}}\maxs
(\vw_{\gammae,\alpha_1}^{\wtd})^{\frac{1}{1+\alpha_1}}\}$ and $\Nb\geq C_1(d+\varsigma)$, we
have that with probability at least $1-e^{-\varsigma}$,
\begin{equation}
\label{eq:rand2}
\rho\Nb\left\{\nabla\cLp(\hbeta_{\eta})-\nabla\cLp(\tbeta)\right\}\tp(\hbeta_{\eta}-\tbeta)
\geq\frac{1}{8}\rho\Nb\|\hbeta_{\eta}-\tbeta\|_{\bSigma,2}.
\end{equation}
Taking $n=\Ns$, $\rho=1$ and $\pi_t=\Ns^{-1}$ for $t\in\cS$,we have that with
probability at least $1-e^{-\varsigma}-1/\sqrt{\Ns}$, 
\begin{equation}
\label{eq:rand3}
(\hbeta_{\eta}-\tbeta)\tp\Xs\tp\Xs(\hbetaw-\tbeta)
\geq(\hbeta_{\eta}-\tbeta)\tp\Xs\tp\Xs(\hbeta_{\eta}-\tbeta)
\geq\frac{1}{8}\Ns\|\hbeta_{\eta}-\tbeta\|_{\bSigma,2},
\end{equation}
and $I_{\cS}\le C_2\sigma_{\cS}\sqrt{(d+\varsigma)\Ns}$, where $C_2$ does not depend
on sample sizes. On the other hand, from Lemma~\ref{lem:psirand2}, we know that  
\begin{align*}
\frac{1}{\rho\Nb}\left\|\sumcb
\frac{\rho\delta_e}{\pi_e}\psi_1(\varepsilon_e;\lambda)\tx_e\right\|_2
\le C_3\lambda^{-\alpha_1}\vw_{\varepsilon,\alpha_1}^{\wtd}
+\frac{C_3\lambda(d+\varsigma)}{\mpi\rho\Nb},
\end{align*}
where $C_3$ does not depend on $\Nb$ with probability at least
$1-e^{-\varsigma}$. From Lemma~\ref{lem:psifull}, we know that with probability
at least $1-e^{-\varsigma}$,
\begin{align*}
&\frac{1}{\rho\Nb}\left\|\sumcb\frac{\rho\delta_e}{\pi_e}
\left\{\psi_1(\gamma_e+\varepsilon_e;\lambda)-\psi_1(\varepsilon_e;\lambda)\right\}
\tx_e\right\|_2
\le C_4\lambda^{-\alpha_1}\vw_{\varepsilon,\alpha_1}^{\wtd}
+\frac{C_4\lambda(d+\varsigma)}{\mpi\rho\Nb}+C_4\lambda^{\alpha_2}\vuw_{\gammae,\alpha_2},
\end{align*}
where $C_4$ does not depend on $\Nb$. Now, since %
$\lambda=
\left\{(\rho\mpi \Nb)/(d+\varsigma)\right\}^{\frac{1}{1+\alpha_1}}$,
we have that  
\begin{equation}\label{eq:rand4}
\begin{split}
\frac{1}{\rho\Nb}\left\|\sumcb\frac{\rho\delta_e}{\pi_e}
\psi_1(\varepsilon_e;\lambda)\tx_e\right\|_2\le
C_3(\vw_{\varepsilon,\alpha_1}^{\wtd}\maxs 1)
\left(\frac{d+\varsigma}{\mpi\rho\Nb}\right)^{\frac{\alpha_1}{1+\alpha_1}},
\end{split}
\end{equation}
and 
\begin{align}\label{eq:rand5}
\begin{split}
&\frac{1}{\rho\Nb}\left\|\sumcb\frac{\rho\delta_e}{\pi_e}
\left\{\psi_1(\gamma_e+\varepsilon_e;\lambda)-\psi_1(\varepsilon_e;\lambda)\right\}
\tx_e\right\|_2\\
&\le C_4(\vw_{\gammae,\alpha_1}^{\wtd}\maxs 1)
\left(\frac{d+\varsigma}{\mpi\rho\Nb}\right)^{\frac{\alpha_1}{1+\alpha_1}}
+C_4\lambda^{\alpha_2}\vuw_{\gammae,\alpha_2}.
\end{split}
\end{align}
Therefore, combining \eqref{eq:rand1}, \eqref{eq:rand2},
\eqref{eq:rand3}, \eqref{eq:rand4}, \eqref{eq:rand5}, we have that, taking
$R=\lambda/8C_1^2$,
\begin{equation*}
\|\hbetaw-\tbeta\|_{\bSigma,2}
\lesssim\frac{\Ns}{N_{f}}\sqrt{\frac{d+\varsigma}{\Ns}}
+\frac{r}{N_{f}}\left(\frac{d+\varsigma}{r}\right)^{\frac{\alpha_1}{1+\alpha_1}}
+\frac{r}{N_{f}}\lambda^{\alpha_2}\vuw_{\gammae,\alpha_2}:=R_{0,1}+R_{0,2}+R_{0,3}:=R_0.
\end{equation*}
Note that as long as $R_0<R$, we have that $\hbeta_{\eta}$ lies in the interior
of the ball with radius $R$ and by the construction of $\hbeta_{\eta}$, we have
that $\eta=1$ in this case, which implies
$\hbeta_{\eta}=\hbetaw$. Therefore, we prove that
$R_{0,1},R_{0,2},R_{0,3}<R/3$, respectively. Note that 
\begin{equation*}
\frac{\sqrt{d+\varsigma}}{2\sqrt{r}}
\le\frac{\sqrt{d+\varsigma}}{\sqrt{\Ns}+\frac{r}{\sqrt{\Ns}}}
\le\frac{\sqrt{\Ns}}{N_{f}}\sqrt{d+\varsigma},
\end{equation*}
and therefore as long as $r\gtrsim (d+\varsigma)$, we have that
$R_{1,0}<R/3$. Since $R_{0,2}\sim(r/N_{f})\lambda(d+\varsigma)/r$, it is also
easy to know that $R_{0,2}<R/3$ as long as $(\Ns+r)\gtrsim (d+\varsigma)$. The fact
that $R_{0,3}<R/3$ is due to the fact that
$N_{f}r^{\frac{-\alpha_1-\alpha_2}{1+\alpha_1}}\gtrsim
\vuw_{\gammae,\alpha_2}(d+\varsigma)^{\frac{1-\alpha_2}{1+\alpha_1}}$.
Thus, we
complete the proof of the bound.

\end{proof}

\subsubsection{Proof of Theorem~\ref{thm:rtrans2}}\label{sec:prf-rtrans2}

\begin{proof}[\textbf{Proof of Theorem~\ref{thm:rtrans2}}]
From Proposition~\ref{pro:thr}, we know that when $P(\cdot;\lambda)$ is
$\ell_2$ penalty, the corresponding $\mathcal{H}(t;\lambda)$ is
\begin{equation*}
\mathcal{H}(t;\lambda)=\frac{\lambda t^2}{1+\lambda}.
\end{equation*}
Thus, we consider a more general form of weighted transfer learning estimator
which minimizes 
\begin{equation*}
\sumcs(y_t-\x_t\tp\bbeta)^2
+\sumcb w_e\mathcal{H}(y_e-\x_e\tp\bbeta;\lambda)
=\sumcs(y_t-\x_t\tp\bbeta)^2
+\sumcb w_e \frac{\lambda(y_e-\x_e\tp\bbeta)^2}{1+\lambda}.
\end{equation*}
Directly solve the optimization problem, we have that 
\begin{equation*}
\hbetaw=\left(\sumcs\x_t\x_t\tp
+\frac{\lambda}{1+\lambda}\sumcb w_e\x_e\x_e\tp\right)^{-1}
\left(\sumcs\x_ty_t+\frac{\lambda}{1+\lambda}\sumcb w_e\x_ey_e\right).
\end{equation*}
Therefore, we have that 
\begin{align*}
&\bSigma^{\frac{1}{2}}\hbetaw-\bSigma^{\frac{1}{2}}\tbeta=\left(\sumcs\tx_t\tx_t\tp
+\frac{\lambda}{1+\lambda}\sumcb w_e\tx_e\tx_e\tp\right)^{-1}
\left(\sumcs\tx_ty_t+\frac{\lambda}{1+\lambda}\sumcb w_e\tx_ey_e\right)
-\bSigma^{\frac{1}{2}}\tbeta\\
&=\left(\sumcs\tx_t\tx_t\tp
+\frac{\lambda}{1+\lambda}\sumcb w_e\tx_e\tx_e\tp\right)^{-1}
\left(\sumcs\varepsilon_t\tx_t
+\frac{\lambda}{1+\lambda}\sumcb w_e\varepsilon_e\tx_e
+\frac{\lambda}{1+\lambda}\sumcb w_e\gamma_e\tx_e\right).
\end{align*}
Letting 
\begin{equation*}
\rho^2\Nb\vw_{\gammae,1}^{\wtd}\Var(\tx)\le\M
=\Exp\left\{\sumcb\frac{\rho^2|\gamma_e|^2}{\pi_e}\tx_e\tx_e\tp\right\}-\rho^2
\sumcb|\gamma_e|^2\Exp(\tx_e)\Exp(\tx_e)\tp
\le\rho^2\Nb\vw_{\gammae,1}^{\wtd}\Exp(\tx\tx\tp),
\end{equation*}
and $\W_e=\M^{-\frac{1}{2}}\{w_e\gamma_e\tx_e-\rho\gamma_e\Exp(\tx_e)\}$.
Due to Berry-Essen theorem (see Theorem 1.1 in~\cite{raivc2019multivariate}), we
have that for every $\varsigma\in\mathbb{R}$,
\begin{align*}
&\Pr\left(\left\|\sumcb\W_e\right\|_2>\sqrt{\varsigma}\right)-\Pr(\|\bm{Z}\|_2>\sqrt{\varsigma})
\le\frac{Cd^{\frac{1}{4}}}{(\sqrt{\Nb})^3\rho^3(\vw_{\gammae,1}^{\wtd})^{\frac{3}{2}}}
\sumcb\Exp\left\{\|w_e\gamma_e\tx_e-\rho\gamma_e\Exp(\tx_e)\|_2^3\right\}\\
&\lesssim\frac{\sumcb|\gamma_e|^3\Exp(\|w_e\tx_e\|_2^3)}{(\sqrt{\Nb})^3\rho^3(\vw_{\gammae,1}^{\wtd})^{\frac{3}{2}}}
+\frac{\sumcb|\gamma_e|^3\Exp(\|\tx\|_2^3)}{(\sqrt{\Nb})^3(\vw_{\gammae,1}^{\wtd})^{\frac{3}{2}}}\\
&\le\frac{\Exp(\|\tx\|_2^3)}{(\sqrt{\Nb})^3\mpi\rho(\vw_{\gammae,1}^{\wtd})^{\frac{3}{2}}}
\sumcb\frac{|\gamma_e|^3}{\pi_e}
+\frac{\sumcb|\gamma_e|^3\Exp(\|\tx\|_2^3)}{(\sqrt{\Nb})^3(\vw_{\gammae,1}^{\wtd})^{\frac{3}{2}}}\\
&\le\frac{\Exp(\|\tx\|_2^3)}{\mpi\rho(\vw_{\gammae,1}^{\wtd})^{\frac{3}{2}}\sqrt{\rho\Nb}}\vw_{\gammae,2}^{\wtd}
+\frac{\Exp(\|\tx\|_2^3)}{(\vw_{\gammae,1}^{\wtd})^{\frac{3}{2}}\sqrt{\Nb}}\frac{\sum_{e\in\cO}|\gamma_e|^3}{\Nb}
\lesssim \frac{\vw_{\gammae,2}^{\wtd}}{\sqrt{\rho\Nb}},
\end{align*}
where $\bm{Z}\sim\Nor(\0,\I_d)$. Therefore, we have that with probability at
least $1-e^{-d-\varsigma}-(C_2\vw_{\gammae,2}^{\wtd})/\sqrt{\rho\Nb}$, we have that 
\begin{align*}
  \left\|\sumcb w_e\gamma_e\tx_e\right\|_2
  &\lesssim\rho\Exp(\|\tx\|_2)\sum_{e\in\cO}|\gamma_e|
+d^{\frac{1}{4}}\sqrt{\rho\Nb}\sqrt{\vw_{\gammae,1}^{\wtd}}\sqrt{\varsigma}\\
&=\rho\Nb\Exp(\|\tx\|_2)\vuw_{\gammae,0}+\sqrt{\rho\Nb}d^{\frac{1}{4}}\sqrt{\vw_{\gammae,1}^{\wtd}
}\sqrt{\varsigma}
\end{align*}
Applying Lemma~\ref{lem:lemofS}, by taking $n=\Ns$, $\rho=1$ and
$\pi_i=\Ns^{-1}$, we have that $\sumcs\tx_t\tx_t\tp\geq 8^{-1}\Ns$ and
$\left\|\sumcs\varepsilon_t\tx_t\right\|_2\le\sqrt{\Ns}\sigma_{\cS}\sqrt{d+\varsigma}$ with
probability at least $1-e^{-\varsigma}-C_3/\sqrt{\Ns}$. From
Lemma~\ref{lem:lemofS}, we can also know that $\sumcb w_e\tx_e\tx_e\tp\geq
8^{-1}\rho\Nb$ and $\left\|\sumcb
  w_e\varepsilon_e\tx_e\right\|_2\le\sqrt{\rho\Nb}\sqrt{\vw_{\varepsilon,1}^{\wtd}}
\sqrt{d+\varsigma}$ with probability at least
$1-e^{-\varsigma}-C_4/\sqrt{\rho\Nb}$. Therefore, we have that
\begin{equation*}
\V_{\lambda}=\frac{1}{N_{f\lambda}}\left(\sumcs\tx_t\tx_t\tp
+\frac{\lambda}{1+\lambda}\sumcb w_e\tx_e\tx_e\tp\right)
\geq \frac{1}{8},
\end{equation*}
where $N_{f\lambda}=\Ns+\{\lambda/(1+\lambda)\}\rho \Nb$, and 
\begin{equation}\label{eq:rtrans-1}
\|\hbetaw-\tbeta\|_{\bSigma,2}\lesssim
\frac{\Ns}{\Ns+c_{\lambda}r}I_{\cS}+\frac{c_{\lambda}r}{\Ns+c_{\lambda}r}I_{\cB}
+\frac{c_{\lambda}r}{\Ns+c_{\lambda}r}I_{\gammae},
\end{equation}
where $I_{\cS}=\sigma_{\cS}\sqrt{d}\sqrt{\varsigma/\Ns}$,
$I_{\cB}=\sqrt{\vw_{\varepsilon,1}^{\wtd}}\sqrt{d}\sqrt{\varsigma/r}$
and
$I_{\gammae}=\vuw_{\gammae,0}+\sqrt{\vw_{\gammae,1}^{\wtd}}\sqrt{d}\sqrt{\varsigma/r}$,
with probability at least
$1-Ce^{-\varsigma}-C/\sqrt{\Ns}-(C\vw_{\gammae,1}^{\wtd})/\sqrt{r}$. Thus,
under the condition that $\sqrt{r}e^{-\varsigma}\gtrsim
\vw_{\gammae,1}^{\wtd}$, the inequality in~\eqref{eq:rtrans-1} holds
with probability at least $1-Ce^{-\varsigma}-C/\sqrt{\Ns}$.

\end{proof}

\subsubsection{Proofs of Corollary~\ref{cor:norm} and
  Proposition~\ref{pro:optpi}}\label{sec:prf-optpi}

\begin{proof}[\textbf{Proof of Corollary~\ref{cor:norm} (I)}]
To prove Corollary~\ref{cor:norm} (I), we first define a random
process  
\begin{align*}
\bm{B}(\bbeta)&=\bSigma^{-\frac{1}{2}}\frac{1}{N_{f}}\Big\{-\Xs\tp(\ys-\Xs\bbeta)
+\rho\Nb\nabla\cLp(\bbeta)\\
&\quad+\Xs\tp(\ys-\Xs\tbeta)-\rho\Nb\nabla\cLp(\tbeta)\Big\}
-\bSigma^{\frac{1}{2}}(\bbeta-\tbeta).
\end{align*}
where $N_{f}=\Ns+\rho \Nb$. We start with a decomposition,
\begin{align*}
&\left\|\bSigma^{\frac{1}{2}}(\hbetaw-\tbeta)
-\bSigma^{-\frac{1}{2}}\frac{1}{N_{f}}\left\{\Xs\tp(\ys-\Xs\tbeta)
+\sumcb\frac{\rho\delta_e}{\pi_e}\psi_1(\varepsilon_e;\lambda)\x_e\right\}\right\|_2\\
&=\Big\|\bSigma^{\frac{1}{2}}(\hbetaw-\tbeta)
-\bSigma^{-\frac{1}{2}}\frac{1}{N_{f}}\Big\{\Xs\tp(\ys-\Xs\tbeta)
+\sumcb\frac{\rho\delta_e}{\pi_e}\psi_1(\varepsilon_e+\gamma_e;\lambda)\x_e\\
&\quad-\sumcb\frac{\rho\delta_e}{\pi_e}\psi_1(\varepsilon_e+\gamma_e;\lambda)\x_e
+\sumcb\frac{\rho\delta_e}{\pi_e}\psi_1(\varepsilon_e;\lambda)\x_e\Big\}\Big\|_2\\
&\le\|\bm{B}(\hbetaw)\|_2
+\frac{\rho\Nb}{N_{f}}\frac{1}{\Nb}\left\|\bSigma^{-\frac{1}{2}}
\sumcb\frac{\delta_e}{\pi_e}\{\psi_1(\varepsilon_e+\gamma_e;\lambda)-
\psi_1(\varepsilon_e;\lambda)\}\x_e\right\|_2,    
\end{align*}
where the last inequality is because
$-\Xs\tp(\ys-\Xs\hbetaw)+\rho\Nb\nabla\cLp(\hbetaw)=\0$, which coincides the K-K-T
condition in~\eqref{eq:kkt}. Note that we have already obtained that
\begin{equation}\label{eq:rand6}
\frac{\rho\Nb}{N_{f}}\frac{1}{\Nb}\left\|\sumcb\frac{\delta_e}{\pi_e}
\left\{\psi_1(\varepsilon_e+\gamma_e;\lambda)-\psi_1(\varepsilon_e;\lambda)\right\}
\tx_e\right\|\le\frac{\rho\Nb}{N_{f}}I^{\pi}_{\gammae}.
\end{equation}
with probability at least $1-Ce^{-\varsigma}-1/\sqrt{\Ns}$. Thus, we only need to bound
$\sup_{\bbeta\in\Theta_0(R_0)}\|\bm{B}(\bbeta)\|_2$. Our approach to bound the
supremum of the empirical process 
$\{\bm{B}(\bbeta):\bbeta\in\Theta_0(R_0)\}$ is to bound
$\bm{B}(\bbeta)-\Exp\{\bm{B}(\bbeta)\}$ and $\Exp\{\bm{B}(\bbeta)\}$ separately. We first
compute $\Exp\{\bm{B}(\bbeta)\}$. Due to the mean value theorem,
\begin{align*}
\Exp\{\bm{B}(\bbeta)\}&=\bSigma^{-\frac{1}{2}}\frac{1}{N_{f}}
\left\{\Ns\bSigma(\bbeta-\tbeta)+\rho \Nb\Exp\nabla\cLp(\bbeta)
-\rho \Nb\Exp\nabla\cLp(\tbeta)\right\}-\bSigma^{\frac{1}{2}}(\bbeta-\tbeta)\\
&=\frac{\rho \Nb}{N_{f}}\left[\bSigma^{-\frac{1}{2}}\Exp\left\{\nabla^2\cLp(\tilde{\bbeta})\right\}
\bSigma^{-\frac{1}{2}}-\I_d\right]\bSigma^{\frac{1}{2}}(\bbeta-\tbeta)\\
&\le R\times\frac{\rho \Nb}{N_{f}}\sup_{\bbeta\in\Theta_0(R)}\left\|
\bSigma^{-\frac{1}{2}}\Exp\left\{\nabla^2\cLp(\tilde{\bbeta})\right\}
\bSigma^{-\frac{1}{2}}-\I_d\right\|_2
\end{align*} 
Therefore, denoting $\bdelta=\bSigma^{\frac{1}{2}}(\bbeta-\tbeta)$, we have that or
$\bbeta\in\Theta_0(R)$ and $\bu\in\mathbb{S}^{d-1}$, 
\begin{align*}
&\left|\bu\tp\left[\bSigma^{-\frac{1}{2}}\Exp\left\{\nabla^2\cLp(\bbeta)\right\}
\bSigma^{-\frac{1}{2}}-\I_d\right]\bu\right|\\
&=\left|\bu\tp\left[\frac{1}{\Nb}\sumcb\bSigma^{-\frac{1}{2}}\Exp\left\{
\frac{\delta_e}{\pi_e}I(|y_e-\x_e\tp\bbeta|\le\lambda)\x_e\x_e\tp\right\}
\bSigma^{-\frac{1}{2}}-\I_d\right]\right|\\
&=\frac{1}{\Nb}\sumcb\Exp\left\{I(|y_e-\x_e\tp\bbeta|>\lambda)(\bu\tp\tx_e)^2\right\}\\
&=\frac{1}{\Nb}\sumcb\Exp\left\{I(|\varepsilon_e+\gamma_e+\x_e\tp(\tbeta-\bbeta)|
>\lambda)(\bu\tp\tx_e)^2\right\}\\
&=\frac{1}{\Nb}\sum_{e\in\cB-\cO}\Exp\left\{I(|\varepsilon_e-\tx_e\tp\bdelta|
>\lambda)(\bu\tp\tx_e)^2\right\}
+\frac{1}{\Nb}\sum_{e\in\cO}\Exp\left\{I(|\varepsilon_e+\gamma_e-\tx_e\tp\bdelta|
>\lambda)(\bu\tp\tx_e)^2\right\}\\
&\le\frac{1}{\Nb}\sum_{e\in\cB-\cO}\Exp\left\{I\left(|\varepsilon_e-\tx_e\tp\bdelta|
>\lambda\right)(\bu\tp\tx_e)^2\right\}\\
&\quad+\frac{1}{\Nb}\sum_{e\in\cO}\Exp\left\{I\left(|\varepsilon_e-\tx_e\tp\bdelta|
>\frac{\lambda}{2}\right)(\bu\tp\tx_e)^2\right\}
+\frac{1}{\Nb}\sum_{e\in\cO}\Exp\left\{I\left(|\gamma_e|
>\frac{\lambda}{2}\right)(\bu\tp\tx_e)^2\right\}\\
&\le\frac{16}{\Nb\lambda^2}\sumcb\Exp(|\varepsilon_e|^2)
+\frac{16}{\Nb\lambda^2}\sumcb\Exp\left\{(\bdelta\tp\tx_e)^2(\bu\tp\tx_e)^2\right\}
+\frac{2^{1+\alpha_1}}{\Nb\lambda^{1+\alpha_1}}\sum_{e\in\cO}|\gamma_e|^{1+\alpha_1}\\
&\le\frac{16}{\Nb\lambda^2}\sumcb\Exp(|\varepsilon_e|^2)
+\frac{16}{\Nb\lambda^2}\sumcb\Exp\left\{(\bdelta\tp\tx_e)^2(\bu\tp\tx_e)^2\right\}
+\frac{2^{1+\alpha_1}}{\Nb\lambda^{1+\alpha_1}}\sum_{e\in\cO}|\gamma_e|^{1+\alpha_1}\\
&\le 16\lambda^{-2}\vw_{\varepsilon,1}^{\wtd}
+2^{1+\alpha_1}\lambda^{-1-\alpha_1}\vw_{\gammae,\alpha_1}^{\wtd}
+16C_1^4\lambda^{-2}R^2
\le 16\lambda^{-2}\vw_{\varepsilon,1}^{\wtd}
+4I_{\gammae}^{\pi}
+16C_1^4\lambda^{-2}R^2.
\end{align*}
From the proof of Theorem~\ref{thm:rtrans} and
$\sqrt{\Ns}+\sqrt{\rho\Nb}\le\sqrt{2}\sqrt{\Ns+\rho\Nb}=\sqrt{2}\sqrt{N_{f}}$,
we know that, there exits a constant $C$, such that 
\begin{align*}
&R=C\sqrt{d+\varsigma}\left(\frac{\sqrt{\Ns}}{N_{f}}+\frac{\sqrt{\rho\Nb}}{N_{f}}\right)
+\frac{\rho\Nb}{N_{f}}I_{\gammae}
\le\frac{C\sqrt{2(d+\varsigma)}}{\sqrt{N_{f}}}+\frac{\rho\Nb}{N_{f}}I_{\gammae},
\end{align*}
which implies that $R^3\le
32\sqrt{2}C^3(d+\varsigma)^{\frac{3}{2}}N_{f}^{-\frac{3}{2}}
+8(\rho\Nb)^3N_{f}^{-3}(I_{\gammae}^{\pi})^3$. Therefore, we have that 
\begin{align}\label{eq:rand7}
\begin{split}
&\sup_{\bbeta\in\Theta_0(R)}\|\Exp\left\{\bm{B}(\bbeta)\right\}\|_2
\le\frac{16\rho\Nb}{N_{f}}\lambda^{-2}\vw_{\varepsilon,1}^{\wtd}R+\frac{4\rho\Nb}{N_{f}}I_{\gammae}R
+\frac{16C_1^4\rho\Nb}{N_{f}}\lambda^{-2}R^3\\
&\le16\sqrt{2}C\vw_{\varepsilon,1}^{\wtd}\left(\frac{d+\varsigma}{N_{f}}\right)^{\frac{3}{2}}
+\frac{16\sqrt{2}C\vw_{\varepsilon,1}^{\wtd}(d+\varsigma)}{N_{f}}\frac{\rho\Nb}{N_{f}}I_{\gammae}
+4\sqrt{2}C\sqrt{\frac{d+\varsigma}{N_{f}}}\frac{\rho\Nb}{N_{f}}I_{\gammae}
+4\left(\frac{\rho\Nb}{N_{f}}I_{\gammae}\right)^2\\
&\quad+512\sqrt{2}C_1^4C^3
\left(\frac{d+\varsigma}{N_{f}}\right)^{\frac{5}{2}}
+\frac{128C_1^4(d+\varsigma)}{N_{f}}\left(\frac{\rho\Nb}{N_{f}}I_{\gammae}\right)^3\\
&\le C_5\frac{d+\varsigma}{N_{f}}+C_5\left(\frac{\rho\Nb}{N_{f}}I_{\gammae}\right)^2
+\frac{C_5(d+\varsigma)}{N_{f}}\left(\frac{\rho\Nb}{N_{f}}I_{\gammae}\right)^3,
\end{split}
\end{align}
where $C_5$ does not depend on sample size.
Now, defining a new process, $\bar{\bm{B}}(\bdelta)$, such that
\begin{align*}
&\bar{B}(\bdelta)=\bm{B}(\bbeta)-\Exp\left\{\bm{B}(\bbeta)\right\}\\
&=\bSigma^{-\frac{1}{2}}\frac{1}{N_{f}}\left\{\Xs\tp\Xs(\bbeta-\tbeta)
+\rho\Nb\nabla\cLp(\bbeta)-\rho \Nb\nabla\cLp(\tbeta)\right\}\\
&\quad-\bSigma^{-\frac{1}{2}}\frac{1}{N_{f}}\left\{\Ns\bSigma(\bbeta-\tbeta)
+\rho \Nb\Exp\nabla\cLp(\bbeta)-\rho \Nb\Exp\nabla\cLp(\tbeta)\right\}\\
&=\frac{\Ns}{N_{f}}\bSigma^{-\frac{1}{2}}\left\{\frac{1}{\Ns}\sumcs\x_t\x_t\tp-\bSigma\right\}
\bSigma^{-\frac{1}{2}}\bdelta\\
&\quad+\frac{\rho \Nb}{N_{f}}\bSigma^{-\frac{1}{2}}\nabla\cLp\left(\tbeta+\bSigma^{-\frac{1}{2}}\bdelta\right)
-\frac{\rho \Nb}{N_{f}}\bSigma^{-\frac{1}{2}}\nabla\cLp\left(\tbeta\right)\\
&\quad-\frac{\rho \Nb}{N_{f}}\bSigma^{-\frac{1}{2}}\Exp\left\{\nabla\cLp\left(\tbeta
+\bSigma^{-\frac{1}{2}}\bdelta\right)\right\}
+\frac{\rho \Nb}{N_{f}}\bSigma^{-\frac{1}{2}}\Exp\left\{\nabla\cLp\left(\tbeta
\right)\right\}.
\end{align*}
Thus, we have $\bar{\bm{B}}(\0)=\0$,
$\Exp\left\{\bar{\bm{B}}(\bdelta)\right\}=\0$, and for every
$\bu,\bv\in\mathbb{S}^{d-1}$ and $\kappa\in\mathbb{R}$,
\begin{align*}
&\Exp e^{\kappa\sqrt{N}_{\rho}\bu\tp\nabla_{\bdelta}\bar{\bm{B}}(\bdelta)\bv}\\
&=\prod_{t=1}^{\Ns}\Exp e^{\frac{\kappa}{\sqrt{N_{f}}}\left\{(\bu\tp\tx_t)(\bv\tp\tx_t)
-\Exp(\bu\tp\tx_t)(\bv\tp\tx_t)\right\}}
\prod_{e=1}^{\Nb}\Exp e^{\frac{\kappa}{\sqrt{N_{f}}}\left\{\frac{\rho\delta_e}{\pi_e}
(\bu\tp\tx_e)(\bv\tp\tx_e)-\rho\Exp(\bu\tp\tx)(\bv\tp\tx)\right\}}\\
&\le\prod_{t=1}^{\Ns}\left[1+\frac{\kappa^2}{N_{f}}\Exp\left\{|\bu\tp\tx_t|^2
|\bv\tp\tx_t|^2+\Exp|\bu\tp\tx\bv\tp\tx|^2\right\}e^{\frac{|\kappa|}{\sqrt{N_{f}}}
(|\bu\tp\tx_t\bv\tp\tx_t|+\Exp|\bu\tp\tx\bv\tp\tx|)}\right]\\
&\quad\times\prod_{e=1}^{\Nb}\left[1+\frac{\kappa^2}{N_{f}}\Exp\left\{\frac{\rho^2\delta_e}{\pi_e^2}|\bu\tp\tx_e|^2
|\bv\tp\tx_e|^2+\rho\Exp|\bu\tp\tx\bv\tp\tx|^2\right\}e^{\frac{|\kappa|}{\sqrt{N_{f}}}
\left(\frac{\rho\delta_e}{\pi_e}|\bu\tp\tx_e\bv\tp\tx_e|+\rho\Exp|\bu\tp\tx\bv\tp\tx|\right)}\right]\\
&\le\prod_{t=1}^{\Ns}\left[1+\frac{\kappa^2}{N_{f}}\Exp\left\{|\bu\tp\tx_t|^2
|\bv\tp\tx_t|^2+\Exp|\bu\tp\tx\bv\tp\tx|^2\right\}e^{\frac{|\kappa|}{\sqrt{N_{f}}}
(|\bu\tp\tx_t\bv\tp\tx_t|+\Exp|\bu\tp\tx\bv\tp\tx|)}\right]\\
&\quad\times\prod_{e=1}^{\Nb}\Big[1+\frac{\kappa^2}{N_{f}}\pi_e\Exp\left\{\frac{\rho^2}{\pi_e^2}|\bu\tp\tx_e|^2
|\bv\tp\tx_e|^2+\rho\Exp|\bu\tp\tx\bv\tp\tx|^2\right\}e^{\frac{|\kappa|}{\sqrt{N_{f}}}
\left(\frac{\rho}{\pi_e}|\bu\tp\tx_e\bv\tp\tx_e|+\rho\Exp|\bu\tp\tx\bv\tp\tx|\right)}\\
&\quad\quad\quad\quad+\frac{\kappa^2}{N_{f}}(1-\pi_e)\rho\Exp|\bu\tp\tx\bv\tp\tx|^2
e^{\frac{|\kappa|}{\sqrt{N_{f}}}\rho\Exp|\bu\tp\tx\bv\tp\tx|}\Big]\\
&\le\prod_{t=1}^{\Ns}\left[1+e^{\frac{|\kappa|}{\sqrt{N_{f}}}}\frac{\kappa^2}{N_{f}}
\Exp e^{\frac{|\kappa|}{\sqrt{N_{f}}}(|\bu\tp\tx_t\bv\tp\tx_t|)}+e^{\frac{|\kappa|}{\sqrt{N_{f}}}}
\frac{\kappa^2}{N_{f}}\Exp\left\{|\bu\tp\tx_t|^2
|\bv\tp\tx_t|^2e^{\frac{|\kappa|}{\sqrt{N_{f}}}(|\bu\tp\tx_t\bv\tp\tx_t|)}\right\}\right]\\
&\quad\times\prod_{e=1}^{\Nb}\Big[1+e^{\frac{\rho|\kappa|}{\sqrt{N_{f}}}}\frac{\kappa^2}{N_{f}}
\rho\pi_e\Exp e^{\frac{|\kappa|}{\mpi\sqrt{N_{f}}}(|\bu\tp\tx_e\bv\tp\tx_e|)}+e^{\frac{\rho|\kappa|}{\sqrt{N_{f}}}}
\frac{\kappa^2}{N_{f}}\frac{\rho}{\mpi}\Exp\left\{|\bu\tp\tx_e|^2
|\bv\tp\tx_e|^2e^{\frac{|\kappa|}{\sqrt{N_{f}}}(|\bu\tp\tx_e\bv\tp\tx_e|)}\right\}\\
&\quad\quad\quad\quad+e^{\frac{\rho|\kappa|}{\sqrt{N_{f}}}}\frac{\kappa^2}{N_{f}}
(1-\pi_e)\rho\Big]\\
&\le\prod_{t=1}^{\Ns}\left[1+e^{\frac{|\kappa|}{\sqrt{N_{f}}}}\frac{\kappa^2}{N_{f}}
\max_{\bm{w}\in\mathbb{S}^{d-1}}\Exp e^{\frac{|\kappa|}{\sqrt{N_{f}}}(|\bm{w}\tp\tx|^2)}+e^{\frac{|\kappa|}{\sqrt{N_{f}}}}
\frac{\kappa^2}{N_{f}}\max_{\bm{w}\in\mathbb{S}^{d-1}}\Exp\left\{|\bm{w}\tp\tx|^4
e^{\frac{|\kappa|}{\sqrt{N_{f}}}(|\bm{w}\tp\tx|^2)}\right\}\right]\\
&\quad\times\prod_{b=1}^{\Nb}\Big[1+e^{\frac{\rho|\kappa|}{\sqrt{N_{f}}}}\frac{\rho\kappa^2}{N_{f}}
\max_{\bm{w}\in\mathbb{S}^{d-1}\cup\0}\Exp e^{\frac{|\kappa|}{\mpi\sqrt{N_{f}}}(|\bm{w}\tp\tx|^2)}
+e^{\frac{\rho|\kappa|}{\sqrt{N_{f}}}}
\frac{\rho\kappa^2}{\mpi N_{f}}\max_{\bm{w}\in\mathbb{S}^{d-1}}\Exp\left\{|\bu\tp\tx|^4
e^{\frac{|\kappa|}{\sqrt{N_{f}}}(|\bu\tp\tx|^2)}\right\}\Big]\\
&\le\exp\left[e^{\frac{|\kappa|}{\sqrt{N_{f}}}}\frac{\Ns\kappa^2}{N_{f}}
\max_{\bm{w}\in\mathbb{S}^{d-1}}\Exp e^{\frac{|\kappa|}{\sqrt{N_{f}}}(|\bm{w}\tp\tx|^2)}+e^{\frac{|\kappa|}{\sqrt{N_{f}}}}
\frac{\Ns\kappa^2}{N_{f}}\max_{\bm{w}\in\mathbb{S}^{d-1}}\Exp\left\{|\bm{w}\tp\tx|^4
e^{\frac{|\kappa|}{\sqrt{N_{f}}}(|\bm{w}\tp\tx|^2)}\right\}\right]\\
&\quad\times\exp\Big[e^{\frac{\rho|\kappa|}{\sqrt{N_{f}}}}\frac{\rho \Nb\kappa^2}{N_{f}}
\max_{\bm{w}\in\mathbb{S}^{d-1}\cup\0}\Exp e^{\frac{|\kappa|}{\mpi\sqrt{N_{f}}}(|\bm{w}\tp\tx|^2)}
+e^{\frac{\rho|\kappa|}{\sqrt{N_{f}}}}
\frac{\rho \Nb\kappa^2}{\mpi N_{f}}\max_{\bm{w}\in\mathbb{S}^{d-1}}\Exp\left\{|\bu\tp\tx|^4
e^{\frac{|\kappa|}{\sqrt{N_{f}}}(|\bu\tp\tx|^2)}\right\}\Big]\\
&\le\exp\Big[e^{\frac{|\kappa|}{\mpi\sqrt{N_{f}}}}\kappa^2\frac{1}{\mpi}
\max_{\bm{w}\in\mathbb{S}^{d-1}\cup\0}\Exp e^{\frac{|\kappa|}{\mpi\sqrt{N_{f}}}(|\bm{w}\tp\tx|^2)}
+e^{\frac{|\kappa|}{\mpi\sqrt{N_{f}}}}
\kappa^2\frac{1}{\mpi}\max_{\bm{w}\in\mathbb{S}^{d-1}}\Exp\left\{|\bu\tp\tx|^4
e^{\frac{|\kappa|}{\mpi\sqrt{N_{f}}}(|\bu\tp\tx|^2)}\right\}\Big]\\
\end{align*}
Therefore, there exists constants $C_6,C_7>0$ depending only on $C_0$ such that
for any $|\kappa|\le\sqrt{N_{f}/C_6}$, $\sup_{\bu,\bv\in\mathbb{S}^{d-1}}
\Exp e^{\kappa\sqrt{N}\bu\tp\nabla_{\bdelta}\bar{\bm{B}}(\bdelta)\bv}\le e^{C_7^2\kappa^2/(2\mpi)}
$.
Applying Theorem A.3 in~\cite{spokoiny2013bernstein}, 
\begin{equation*}
\Pr\left\{\sup_{\bbeta\in\Theta_0(R)}\|\bm{B}(\bbeta)-\Exp\bm{B}(\bbeta)\|_2\geq
6C_7(8d+2\varsigma)^{\frac{1}{2}}\mpi^{-\frac{1}{2}}N_{f}^{-\frac{1}{2}}R\right\}\le e^{-\varsigma},
\end{equation*}
as long as $N_{f}\geq C_6(8d+2t)$, which implies that 
\begin{equation}\label{eq:rand8}
\sup_{\bbeta\in\Theta_0(R)}\|\bm{B}(\bbeta)-\Exp\bm{B}(\bbeta)\|_2
\le \frac{C_8(d+\varsigma)}{N_{f}}
+\sqrt{\frac{C_8(d+\varsigma)}{N_{f}}}\frac{\rho \Nb}{N_{f}}I_{\gammae}
\le\frac{3}{2}\frac{C_8(d+\varsigma)}{N_{f}}
+\left(\frac{\rho\Nb}{N_{f}}I_{\gammae}\right)^2
\end{equation}
where $C_8=24C_7C$ with probability at least $1-e^{-\varsigma}$ at least
$1-e^{-\varsigma}$. Thus, combining~\eqref{eq:rand6},~\eqref{eq:rand7},
and~\eqref{eq:rand8}. we have that 
\begin{equation*}
\sup_{\bbeta\in\Theta_0(r)}\|\bm{B}(\bbeta)\|_2\le
C_9\frac{\rho\Nb}{N_{f}}\check{I}_{\gammae}+\frac{C_9(d+\varsigma)}{N_{f}},
\end{equation*}
where $\check{I}_{\gammae}=I_{\gammae}+\rho\Nb
N_{f}^{-1}(I_{\gammae})^2+(d+\varsigma)(\rho\Nb)^2
N_{f}^{-3}(I_{\gammae})^3$. This complete the proof.

\end{proof}

\begin{proof}[\textbf{Proof of Corollary~\ref{cor:norm} (II)}]
To prove Corollary~\ref{cor:norm} (II), first notice that
$\Exp\{\vc(\tx\tx\tp)-\vc(\bSigma)\}=\0$. Since
$\vc(\tx\tx\tp)=\tx\otimes\tx$, we have that
$\Var\{\vc(\tx\tx\tp)\}=\Exp\{\vc(\tx\tx\tp)\vc(\tx\tx\tp)\tp\}
-\vc(\bSigma)\vc(\bSigma)\tp=\Exp\{(\tx\tx\tp)\otimes(\tx\tx\tp)\}
-\vc(\bSigma)\vc(\bSigma)\tp$ and $\|\Var\{\vc(\tx\tx\tp)\}\|_F
\le\Exp\{\|(\tx\tx\tp)\otimes(\tx\tx\tp)\|_F\}=\Exp\{\|\tx\|^4\}$. Therefore,
taking
$\W_t=\Ns^{-\frac{1}{2}}\Var\{\vc(\tx\tx\tp)\}^{-\frac{1}{2}}\{\vc(\tx_t\tx_t\tp)-\vc(\bSigma)\}$,
we have that $\Exp(\W_t)=\0$ and $\Var(\sumcs W_t)=\I_{d\times d}$. Applying
Theorem 1.1 in \cite{raivc2019multivariate}, we have that 
\begin{align*}
\Pr\left(\left\|\sumcs\W_t\right\|_2>\sqrt{\varsigma+d}\right)
-\Pr(\|\Z\|_2>\sqrt{\varsigma})\le\frac{C\sqrt{d}\Exp(\|\tx\|_2^6)}{\sqrt{\Ns}}.
\end{align*}
Thus, we know that with probability at least $1-e^{-\varsigma}-C/\sqrt{\Ns}$,
\begin{align*}
\left\|\sumcs\tx_t\tx_t\tp-\bSigma\right\|_F\le\sqrt{\Ns}\sqrt{\Exp\|\tx\|^4}\sqrt{\varsigma+d}
\lesssim\sqrt{\Ns}\sqrt{d+\varsigma}.  
\end{align*}
Using the same argument, we can also obtain that with probability at least
$1-e^{-\varsigma}-C/\sqrt{r}$,
\begin{align*}
\left\|\sumcb\frac{\rho\delta_e}{\pi_e}\tx_e\tx_e\tp-\bSigma\right\|_F
\le\sqrt{\rho\Nb}\sqrt{\Exp\|\tx\|^4}\sqrt{\varsigma+d}
\lesssim\sqrt{r}\sqrt{d+\varsigma}.  
\end{align*}
Therefore, with probability at least
$1-Ce^{-\varsigma}-C/\sqrt{\Ns}$, 
\begin{equation*}
\left\|\V_{\lambda}-\bSigma\right\|\lesssim\sqrt{\frac{d+\varsigma}{N_{f\lambda}}}.
\end{equation*}
Combining with the fact that 
\begin{equation*}
\frac{1}{N_{f\lambda}}\left\{\sumcs\varepsilon_t\tx_t
+\frac{\lambda}{1+\lambda}\sumcb w_e\varepsilon_e\tx_e\right\}
\lesssim\frac{\sqrt{\Ns(d+\varsigma)}}{N_{f\lambda}}
+\frac{\lambda}{1+\lambda}\frac{\sqrt{r(d+\varsigma)}}{N_{f}}
\le \sqrt{\frac{d+\varsigma}{N_{f\lambda}}},
\end{equation*}
we complete the proof.

\end{proof}

\begin{proof}[\textbf{Proof of Proposition~\ref{pro:optpi}}]
We first prove that $\psi(z;\lambda)$ is an odd function of $z$ as long as 
$P(\cdot)$ is an even penalty function. This is true due to the definition of
$\psi(z;\lambda)$. Since
$\Theta(z;\lambda)=\argmin_{\gamma\in\mathbb{R}}\left\{(z-\gamma)^2+\lambda
  P(\gamma)\right\}^2$, we have that 
\begin{equation*}
  \Theta(-z;\lambda)
  =\argmin_{\gamma\in\mathbb{R}}\left\{(-z-\gamma)^2+\lambda P(\gamma)\right\}^2
  =\argmin_{\gamma\in\mathbb{R}}\left[\left\{z-(-\gamma)\right\}^2+\lambda
    P(-\gamma)\right]^2
  =-\Theta(z;\lambda).
\end{equation*}
Therefore, we have that
$\psi(-z;\lambda)=-z-\Theta(-z;\lambda)=-z+\Theta(z;\lambda)=-\psi(z;\lambda)$. Thus,
$\psi(z;\lambda)$ is odd as long as $P(\cdot)$ is even. Therefore, it is easy to
know that the $\psi_{\nu}(\varepsilon;\lambda)$ is odd for both $\nu=1,2$, where
$\psi_{\nu}(z;\lambda)$ represents the $\psi(z;\lambda)$ function corresponding
to $\ell_1$ or $\ell_2$ penalties discussed in the main paper.
Since $\varepsilon_e$ is symmetric and $\psi_{\nu}(\varepsilon;\lambda)$ is an odd
function of $\varepsilon$, we have that
$\Exp\left\{\psi_{\nu}(\varepsilon;\lambda)\right\}=0$, and thus, 
\begin{align*}
&\Exp\left\{\left\|\sumcb\frac{\rho\delta_e}{\pi_e}\psi_{\nu}(\varepsilon_e;\lambda)\bSigma^{-\frac{1}{2}}\x_e
\right\|^2\Big|\Xb\right\}\\
&=\sumcb\Exp\left\{\frac{\rho^2\delta_e}{\pi_e^2}
\psi_{\nu}(\varepsilon_e;\lambda)^2\x_e\tp\bSigma^{-1}\x_e
\Big|\Xb\right\}\\
&\qquad+2\sum_{1\le i<j\le \Nb}^{\Nb}\Exp\left\{\frac{\rho^2\delta_i\delta_j}{\pi_i\pi_j}
\psi_{\nu}(\varepsilon_i;\lambda)\psi_{\nu}(\varepsilon_j;\lambda)\x_i\tp\bSigma^{-1}\x_j
\Big|\Xb\right\}\\
&=\rho^2\Exp\left\{\psi_{\nu}(\varepsilon;\lambda)\right\}^2
\sumcb\frac{\x_e\tp\bSigma^{-1}\x_e}{\pi_e}
\end{align*}
Let $t_e=\|\bSigma^{-\frac{1}{2}}\x_e\|_2^2$. 
Minimizing $\Exp\left\{\left\|\sumcb\frac{\rho\delta_e}{\pi_e}\psi_{\nu}(\varepsilon_e;\lambda)\bSigma^{-\frac{1}{2}}\x_e
\right\|^2\Big|\Xb\right\}$ is equivalent to minimizing 
\begin{equation*}
\sumcb\frac{t_e^2}{\pi_e},
\end{equation*}
subject to $\sumcb\pi_e=r$ and $0\le\pi_e\le 1$ for any
$e\in\cB$. Following the same proof of Theorem 4 in \cite{wang2022sampling} with
$t_{(i)}=\|\bSigma^{-\frac{1}{2}}\x_e\|_{(i),2}^2$, we obtain that 
\begin{equation*}
\pi_e^{\opt}=\frac{r\|\bSigma^{-\frac{1}{2}}\x_e\|_2\mins H}{\sum_{i=1}^{\Nb}
\|\bSigma^{-\frac{1}{2}}\x_i\|_2\mins H},
\end{equation*}
where 
\begin{equation*}
H=\frac{r\sum_{i=1}^{\Nb-g}\|\bSigma^{-\frac{1}{2}}\x\|_{(i)}}{r-g},
\end{equation*}
$\|\bSigma^{-\frac{1}{2}}\x\|_{(1)}\le\ldots\le\|\bSigma^{-\frac{1}{2}}\x\|_{(\Nb)}$
are order statistics of
$\|\bSigma^{-\frac{1}{2}}\x_1\|,\ldots,\|\bSigma^{-\frac{1}{2}}\x_{\Nb}\|$, and $g$
is an integer such that 
\begin{equation*}
\frac{r\|\bSigma^{-\frac{1}{2}}\x\|_{(\Nb-g)}}{
\sum_{i=1}^{\Nb-g}\|\bSigma^{-\frac{1}{2}}\x\|_{(i)}}<\frac{r}{r-g}, 
\end{equation*}
and
\begin{equation*}
\frac{r\|\bSigma^{-\frac{1}{2}}\x\|_{(\Nb-g+1)}}{
\sum_{b=1}^{\Nb-g+1}\|\bSigma^{-\frac{1}{2}}\x\|_{(b)}}\geq\frac{r}{r-g+1}, 
\end{equation*}
in which we define $\|\bSigma^{-\frac{1}{2}}\x\|_{(\Nb+1)}=\infty$.
\end{proof}

\subsection{Mathematical details of Section~\ref{sec:dtrans}}

In this section, our goal is to prove Proposition~\ref{pro:nobias} and
Theorem~\ref{thm:dtrans}. Since in this section, we mainly focus on the external
data, therefore, without special illustration, we always use $\varepsilon$ to
represent the random error of the external data and use $\x$ to represent the
covariate variable of the external data. For a random variable $Y$, we use notation $Y_{(k)}$
to represent its order statistics, and for a random variable $X$ that is related
to $Y$, we use notation $X_{[k]}$ to represent the corresponding concomitant of
order statistics $Y_{(k)}$. For example, we selected data points with order
statistics $|y-\x\tp\hbeta_{\cS}|_{(1)},|y-\x\tp\hbeta_{\cS}|_{(2)},\ldots
|y-\x\tp\hbeta_{\cS}|_{(r)}$, and therefore, the selected data points can be
represented as concomitants
$(\varepsilon_{[1]},\x_{[1]}\tp,\gamma_{[1]})\tp,(\varepsilon_{[2]},\x_{[2]}\tp,\gamma_{[2]})\tp,
\ldots,(\varepsilon_{[r]},\x_{[r]}\tp,\gamma_{[r]})\tp$. To ease the notation, we denote
$\deltaS=\sqrt{\Ns}(\sumcs\tx_t\tx_t\tp)^{-1}\sumcs\varepsilon_t\tx_t$ in this
section and denote $\Ze=\varepsilon-\Ns^{-\frac{1}{2}}\tx\tp\deltaS$ and
$\Zg=\varepsilon-\Ns^{-\frac{1}{2}}\tx\tp\deltaS+\gamma$. Note that 
\begin{align*}
&y-\x\tp\hbeta_{\cS}=\varepsilon+\gamma+\x\tp\tbeta
-\x\tp\left(\sumcs\x_t\x_t\tp\right)^{-1}\sumcs\varepsilon_t\x_t
-\x\tp\tbeta\\
&=\varepsilon+\gamma+\x\tp\tbeta
-\x\tp\bSigma^{\frac{1}{2}}\left(\sumcs\tx_t\tx_t\tp\right)^{-1}\sumcs\varepsilon_t\tx_t
-\x\tp\tbeta
=\Zg.
\end{align*}
We first prove Proposition~\ref{pro:nobias} in Section~\ref{sec:proofnobias}.

\subsubsection{Proof of Proposition~\ref{pro:nobias}}\label{sec:proofnobias}

We prove Proposition~\ref{pro:nobias} and start with a lemma.
\begin{lemma}\label{lem:order}
Under Assumptions~\ref{asm:eps},~\ref{asm:gamma}, and~\ref{asm:epsLip}, we have
that $\forall1\le r\le\Nb$, 
\begin{align*}
\Pr\left(|\Zg|_{(r)}>\frac{r+\sqrt{r\varsigma}}{\Nb}\Big|\Ds\right)
\lesssim 
e^{-\varsigma}+r^{\frac{1}{2}}\left(\frac{r}{\Nb}\right)^{\alpha_{\varepsilon}\mins
  1}, \text{ almost surely},
\end{align*}
and for any function $h(z)$ that satisfies $\sup_{z\in\mathbb{R}}|h(z)|\le 1$,
we have that for any Borel set $\mathcal{A}$, 
\begin{equation}
\label{eq:iorder}
\left|\int_{\mathcal{A}} h(z)\dd F_{\Nb}^{(r)}(z)\right|\lesssim
\frac{1}{r!}\int_{\mathcal{A}}|h(z)|z^r e^{-z}\dd z
+r^{\frac{1}{2}}\left(\frac{r}{\Nb}\right)^{\alpha_{\varepsilon}\mins 1},
\text{ almost surely},
\end{equation}
where $F_{\Nb}^{(r)}$ is the c.d.f. of $\Nb|\Zg|_{(r)}$ conditional on $\Ds$.
\end{lemma}
\begin{proof}
We have that the density of $|\Zg|$ conditional on $\Ds$ is given as
\begin{align*}
f_{|\Zg|}(z)
&=\sum_{s\in\{-1,1\}}\int_{-\infty}^{\infty}\int_0^{\infty}
\feps\left(sz-\gamma+\frac{u}{\sqrt{\Ns}}\right)
f_{\tx\tp\deltaS}(u)\fgam(\gamma)\dd\gamma\dd u.
\end{align*} 
Without loss of generality, we assume that $f_{|\Zg|}(0)=1$ in this proof, since
we can always consider $\Zg=\kappa(\varepsilon+\gamma-\Ns^{-\frac{1}{2}}\tx\tp\deltaS)$,
where $\kappa=f_{|\varepsilon+\gamma-\Ns^{-\frac{1}{2}}\tx\tp\deltaS|}(0)$, so
that $f_{|\Zg|}(0)=1$. We first prove that the c.d.f of $\Nb|\Zg|$ conditional
on $\Ds$, denoted as $F_{\Nb}^{(r)}$, converges to a
degenerate distribution. This result can be prove with results of extreme
statistics of i.d.d. data points since $\Zg_1,\Zg_2,\ldots,\Zg_{\Nb}$ are i.i.d.
data points conditional on $\Ds$. Letting $\xi=1/|\Zg|$, we have that
$\xi_{(i)}=1/|\Zg|_{(\Nb-i+1)}$,
$F_{\xi}(z)=1-F_{|\Zg|}(z^{-1})$ and $f_{\xi}(z)=z^{-2}f_{|\Zg|}(z^{-1})$. We
now prove that $F_{\xi}$ satisfies VM1 in~\cite{falk1993mises}. Since
$\omega(F_{\xi})=\sup\{z\in\mathbb{R}:F_{\xi}(z)<1\}=\infty$ and
$f_{|\Zg|}(0)=1$, we only need to show that as $z\to\infty$,  
\begin{align*}
f_{\xi}(z)=\frac{1}{z^2}f_{|\Zg|}\left(\frac{1}{z}\right)
=\frac{1}{z^2}\left\{1+f_{|\Zg|}\left(\frac{1}{z}\right)
-f_{|\Zg|}(0)\right\}
=\frac{1}{z^2}\left\{1+Cz^{-\alpha_{\varepsilon}}\right\},
\end{align*}
where $C$ is a fixed constant that does not depend on $\Ds$, which will verify
VM1 in~\cite{falk1993mises}. Under Assumption~\ref{asm:epsLip},
applying dominating convergence theorem, we have that  
\begin{align*}
&\lim_{z\to\infty}z^{\alpha_{\varepsilon}}\left\{f_{|\Zg|}\left(\frac{1}{z}\right)
-f_{|\Zg|}(0)\right\}  
=\lim_{w\to 0}\frac{f_{|\Zg|}(w)-f_{|\Zg|}(0)}{w^{\alpha_{\varepsilon}}}\\
&=\lim_{w\to0}\int_{-\infty}^{\infty}\int_0^{\infty}
\frac{\feps\left(w-\gamma+\frac{u}{\sqrt{\Ns}}\right)
-\feps\left(-\gamma+\frac{u}{\sqrt{\Ns}}\right)}{w^{\alpha_{\varepsilon}}}f_{\tx\tp\deltaS}(u)
\fgam(\gamma)\dd\gamma\dd u\\
&\quad+\lim_{w\to0}\int_{-\infty}^{\infty}\int_0^{\infty}
\frac{\feps\left(-w-\gamma+\frac{u}{\sqrt{\Ns}}\right)
-\feps\left(-\gamma+\frac{u}{\sqrt{\Ns}}\right)}{w^{\alpha_{\varepsilon}}}
f_{\tx\tp\deltaS}(u)\fgam(\gamma)\dd\gamma\dd u\\
&=2\Exp\left\{H\left(-\gamma+\frac{\tx\tp\deltaS}{\sqrt{\Ns}}\right)\Big|\Ds\right\}\le C_1,
\end{align*}
where the last inequality is due to Assumption~\ref{asm:epsLip}, which implies
that $H(\cdot)$ is bounded by a constant almost surely.  
Therefore, we have that
$f_{\xi}(z)=z^{-2}\left\{1+O(z^{-\alpha_{\varepsilon}})\right\}$. Due to Theorem
2.4 and Theorem 3.1 in~\cite{falk1993mises}, we have that for any
Borel set $\mathcal{A}$ and $r\in\mathbb{N}$, 
\begin{align*}
\left|\Pr\left\{\frac{\xi_{(n-r+1)}}{\Nb}\in\mathcal{A}\Big|\Ds\right\}
-\int_{\mathcal{A}}g_{r+1}(z)\dd z\right|
\le C_2\left\{\left(\frac{r}{\Nb}\right)^{\alpha_{\varepsilon}}r^{\frac{1}{2}}+\frac{r}{\Nb}\right\}, \text{ almost surely},
\end{align*} 
where $g_{r+1}(z)=\left\{-\log G(z)\right\}^r g(z)/r!$,
$G(z)=e^{-\frac{1}{z}}$, and $C_2$ is a constant that only determined by $C_1$. Here, $g_{r+1}(z)$ is the asymptotic distribution of
extreme statistics that satisfies condition VM1 in~\cite{falk1993mises}. The
explicit formula of $g_{r+1}(z)$ and $G(z)$ is given
in~\cite{nagaraja1994distribution}. The fact that $C_2$ is determined by $C_1$
is because of the proof in ~\cite{falk1993mises}. Therefore, we have that for any Borel set
$\mathcal{A}$,   
\begin{align*}
\left|\Pr\left\{\Nb|\Zg|_{(r)}\in\mathcal{A}\Big|\Ds\right\}
-\frac{1}{r!}\int_{\mathcal{A}} z^r e^{-z}\dd z\right|
\lesssim r^{\frac{1}{2}}\left(\frac{r}{\Nb}\right)^{\alpha_{\varepsilon}\mins 1},
\text{ almost surely}.
\end{align*}
Now, letting $\mathcal{A}=((r+\sqrt{r\cvsigma})/\Nb,\infty)$ and denoting
$\Gamma_r$ as a gamma distributed random variable with shape parameter $r$ and
scale parameter 1, we have that $\forall1\le r\le\Nb$, 
\begin{align*}%
\begin{split}
&\Pr\left(|\Zg|_{(r)}>\frac{r+\sqrt{r\varsigma}}{\Nb}\Big|\Ds\right)
\lesssim\frac{1}{r!}\int_{r+\sqrt{r\varsigma}}^{\infty}z^r e^{-z}\dd z
+r^{\frac{1}{2}}\left(\frac{r}{\Nb}\right)^{\alpha_{\varepsilon}\mins 1}\\
&\lesssim\Pr\left(\frac{\Gamma_r}{r}-1>\sqrt{\frac{\varsigma}{r}}\right)
+r^{\frac{1}{2}}\left(\frac{r}{\Nb}\right)^{\alpha_{\varepsilon}\mins 1}\\
&\lesssim e^{-r\frac{\varsigma}{r}}+r^{\frac{1}{2}}\left(\frac{r}{\Nb}\right)^{\alpha_{\varepsilon}\mins 1}
=e^{-\varsigma}+r^{\frac{1}{2}}\left(\frac{r}{\Nb}\right)^{\alpha_{\varepsilon}\mins
  1},
\end{split}
\end{align*}
almost surely. Further, we have that
\begin{align*}
&\left|\int_{\mathcal{A}}h(z)\dd F_{\Nb}^{(r)}(z)\right|
\le\frac{1}{r!}\int_{\mathcal{A}}|h(z)|z^r e^{-z}\dd z
+\left|\Pr\left\{\Nb|\Zg|_{(r)}\in\mathcal{A}\right\}
-\int_{\mathcal{A}}z^r e^{-z}\dd z\right|\\
&\lesssim\frac{1}{r!}\int_{\mathcal{A}} |h(z)|z^r e^{-z}\dd z
+r^{\frac{1}{2}}\left(\frac{r}{\Nb}\right)^{\alpha_{\varepsilon}\mins 1},
\end{align*}
almost surely, which completes the proof of Lemma~\ref{lem:order}.
\end{proof}

Now, we present the proof of Proposition~\ref{pro:nobias}.
\begin{proof}[\textbf{Proof of Proposition~\ref{pro:nobias}}]
Letting 
\begin{align*}
\geps(t)=\int_{-\infty}^{\infty}\feps\left(t+\frac{u}{\sqrt{\Ns}}\right)
f_{\tx\tp\deltaS}(u)\dd u\text{, and }
\getg(t)=\int_{-\infty}^{\infty}\fetg\left(t+\frac{u}{\sqrt{\Ns}}\right)
f_{\tx\tp\deltaS}(u)\dd u,
\end{align*}
since $\feps(\cdot)$ is a Lipschitz function with a Lipschitz constant $\Leps$,
we know that $\forall t\in\mathbb{R}$, condtional on $\Ds$, 
\begin{align*}
|\geps(t)-\feps(0)|
\le\int_{-\infty}^{\infty}\left|\feps\left(t+\frac{u}{\sqrt{\Ns}}\right)
-\feps(0)\right|f_{\tx\tp\deltaS}(u)\dd u
\le\Leps|t|+\frac{\Leps\Exp(\|\tx\|_2)\|\deltaS\|_2}{\sqrt{\Ns}}.
\end{align*}  
Since under Assumption~\ref{asm:epsLip}, we have that $\fetg$ is also a
Lipschitz function with Lipschitz constant $\Leps$ because
\begin{align*}
\left|\fetg(z+w)-\fetg(z)\right|
\le\int_0^{\infty}
\left|\feps(z+w-\tilde{\gamma})-\feps(z-\tilde{\gamma})\right|
\ftga(\tilde{\gamma})\dd\tilde{\gamma}
\le\Leps|w|.
\end{align*}
Therefore, we similarly have that $\forall t\in\mathbb{R}$,
\begin{align*}
|\getg(t)-\fetg(0)|
\le\Leps|t|+\frac{\Leps\Exp(\|\tx\|_2)\|\deltaS\|_2}{\sqrt{\Ns}}.
\end{align*}
Now, we consider a ``good'' set $\mathcal{A}$, such that on $\mathcal{A}$, we
have that $\|\deltaS\|_2^2\le C\sigma_{\cS}^2\cvsigma$. Then, if $\sqrt{\Ns}\geq
4C\{f_{\varepsilon}(0)\}^{-1}\Leps\Exp(\|\tx\|_2)\sigma_{\cS}\sqrt{\cvsigma}$, we
have that $\sqrt{\Ns}\geq 4\{f_{\varepsilon}(0)\}^{-1}\Leps\Exp(\|\tx\|_2)\|\deltaS\|_2$.
On ``good'' set $\mathcal{A}$, we have that applying Bayesian's rule, $\forall
0<z<\frac{\feps(0)}{8\Leps}$,
\begin{align*}
&\Pr\left(\gamma>0\big||\Zg|=z,\Ds\right)
=\frac{f_{\left|\varepsilon+\tilde{\gamma}-\frac{\tx\tp\deltaS}{\sqrt{\Ns}}\right|}(z)\Pr(\gamma\neq
  0)}{f_{\left|\varepsilon-\frac{\tx\tp\deltaS}{\sqrt{\Ns}}\right|}(z)\Pr(\gamma=
  0)+f_{\left|\varepsilon+\tilde{\gamma}-\frac{\tx\tp\deltaS}{\sqrt{\Ns}}\right|}(z)\Pr(\gamma\neq
  0)}\\
&=\frac{p_{\gamma}\sum_{s\in\{-1,1\}}\getg(sz)}{(1-p_{\gamma})\sum_{s\in\{-1,1\}}\geps(sz)
+p_{\gamma}\sum_{s\in\{-1,1\}}\getg(sz)}\\
&\le\frac{p_{\gamma}}{1-p_{\gamma}}\frac{\fetg(0)+\Leps|z|+\frac{\Leps\Exp(\|\tx\|_2)\|\deltaS\|_2}{\sqrt{\Ns}}}{
\feps(0)-\Leps|z|-\frac{\Leps\Exp(\|\tx\|_2)\|\deltaS\|_2}{\sqrt{\Ns}}}\\
&\le Cp_{\gamma}\left\{\fetg(0)+z+\frac{\|\deltaS\|_2}{\sqrt{\Ns}}\right\},
\end{align*}
where the last inequality is because $p_{\gamma}$ is upper bounded away from 1
and $C$ is a constant that does not depend on $\Ds$.
Therefore, using the same arguments in~\cite{nagaraja1994distribution} to obtain
(2.1), we can obtain that $\forall 0<z_0<\frac{\feps(0)}{8\Leps}$,
\begin{align*}
&\Pr\left(\max_{1\le k\le r}\gamma_{[k]}=0\Big||\Zg|_{(r+1)}=z_0,\Ds\right)
=\left\{\int_0^{z_0}\Pr\left(\gamma=0\Big||\Zg|=z\right)\frac{f_{|\Zg|}(z)}{F_{|\Zg|}(z_0)}\dd z\right\}^r\\
&\geq\left[1-Cp_{\gamma}\left\{\feg(0)+z_0+\frac{\|\deltaS\|_2}{\sqrt{\Ns}}\right\}\right]^r
\geq 1-Crp_{\gamma}\left\{\feg(0)+z_0+\frac{\|\deltaS\|_2}{\sqrt{\Ns}}\right\}.
\end{align*}
Thus, due to~\eqref{eq:iorder}, we have that 
\begin{align*}
&\Pr\left(\max_{1\le k\le r}\gamma_{[k]}>0\Big|\Ds\right)
\lesssim\frac{1}{(r+1)!}\int_0^{\frac{\Nb\feps(0)}{8\Leps}}\Pr\left(\max_{1\le k\le
r}\gamma_{[k]}>0\Big||\Zg|_{(r+1)}=\frac{z}{\Nb},\Ds\right)z^{r+1} e^{-z}\dd z\\
&\quad+\frac{1}{(r+1)!}\int_{\frac{\Nb\feps(0)}{8\Leps}}^{\infty}z^{r+1} e^{-z}\dd z
+(r+1)^{\frac{1}{2}}\left(\frac{r+1}{\Nb}\right)^{\alpha_{\varepsilon}\mins 1}\\
&\lesssim rp_{\gamma}\fetg(0)
+\frac{rp_{\gamma}\|\deltaS\|_2}{\sqrt{\Ns}}
+\frac{rp_{\gamma}}{\Nb}\frac{1}{(r+1)!}\int_0^{\infty}z^{r+2}e^{-z}\dd z
  +\Pr\left(\Gamma_{r+1}>\frac{\Nb\feps(0)}{2}\right)\\
  &\qquad\qquad+r^{\frac{1}{2}}\left(\frac{r}{\Nb}\right)^{\alpha_{\varepsilon}\mins 1}\\
&\lesssim rp_{\gamma}\fetg(0)
+\frac{rp_{\gamma}\|\deltaS\|_2}{\sqrt{\Ns}}
+\frac{r^2p_{\gamma}}{\Nb}
+\frac{r}{\Nb}+r^{\frac{1}{2}}\left(\frac{r}{\Nb}\right)^{\alpha_{\varepsilon}\mins 1}
:=\check{p}
\end{align*}
where $\Gamma_{r+1}$ represent a Gamma distributed random variable with shape
parameter $r+1$ and scale parameter 1. Therefore, we know that on ``good'' set
$\mathcal{A}$, we have that 
\begin{equation*}
\Pr\left(\gamma_{[i]}=0,k=1,\ldots,r|\Ds\right)\geq 1-C\check{p}
\end{equation*}
Now, we know that 
\begin{align*}
&\Pr\left(\gamma_{[i]}=0,k=1,\ldots,r\right)
\geq\int_{\mathcal{A}}\Pr\left(\gamma_{[i]}=0,k=1,\ldots,r\big|\Ds\right)\dd\Pr
\geq (1-C\check{p})\{1-\Pr(\mathcal{A}^c))\}\\
&\geq 1-C\check{p}-\Pr(\mathcal{A}^c).
\end{align*}
Now, since ``good'' set $\mathcal{A}$ satisfies with probability at least
$1-\frac{C}{\sqrt{\Ns}}$ according to Lemma~\ref{lem:lemofS}, we complete the
proof. 

\end{proof}

\subsubsection{Proof of Theorem~\ref{thm:dtrans}}

In the following, our proofs are based on normal assumptions such that
$\varepsilon\sim\Nor(0,\sigma_{\varepsilon}^2)$ and $\x\sim\Nor(\0,\bSigma)$.
To ease the
notations, we denote $\z_f=(z_1,z_2,z_3)\tp$, $\z_c=(z_1,z_2)\tp$, 
$\sigma_{\Ze}^2=\sigma_{\varepsilon}^2+\frac{\deltaS\tp\deltaS}{\Ns}$,
\begin{align*}
\bmu_c=
\begin{pmatrix}
\frac{\sigma_{\varepsilon}^2}{\sigma_{\Ze}^2}\\
-\frac{\bu\tp\deltaS}{\sqrt{\Ns}\sigma_{\Ze}^2}\\
\end{pmatrix},
\bSigma_f=
\begin{pmatrix}
\sigma_{\varepsilon}^2 & 0 & \sigma_{\varepsilon}^2\\
0 & 1 & -\frac{\bu\tp\deltaS}{\sqrt{\Ns}}\\
\sigma_{\varepsilon}^2 & -\frac{\bu\tp\deltaS}{\sqrt{\Ns}} & \sigma_{\Ze}^2\\
\end{pmatrix},
\end{align*}
and 
\begin{align*}
\bSigma_c=
\begin{pmatrix}
\frac{1}{\Ns}\frac{\deltaS\tp\deltaS}{\sigma_{\Ze}^2}
 & \frac{\sigma_{\varepsilon}^2\bu\tp\deltaS}{\sqrt{\Ns}\sigma_{\Ze}^2}\\
\frac{\sigma_{\varepsilon}^2\bu\tp\deltaS}{\sqrt{\Ns}\sigma_{\Ze}^2} &
1
-\frac{1}{\Ns}\frac{\deltaS\tp\deltaS}{\sigma_{\Ze}^2}\\
\end{pmatrix},
\end{align*}
where $\bu\in\mathbb{S}^{d-1}$ represent any vector with norm equals to 1 in
$\mathbb{R}^d$ in this section. To prove Theorem~\ref{thm:dtrans}, we start with
some lemmas.

\begin{lemma}\label{lem:nobias2}
Under Assumption~\ref{asm:gamma} and assume that
$\varepsilon\sim\Nor(0,\sigma_{\varepsilon}^2)$ and $\x\sim\Nor(\0,\bSigma)$, we
have that $\forall\cvsigma>0,0<\consp<8$, as long as $\Ns\gtrsim e^{2\cvsigma}$,
$r\le\{1-(1+\consp)p_{\gamma}\}\Nb$, and $\gamma_b\geq
4\sqrt{6}\sigma_{\varepsilon}\sqrt{\log\Nb}$, we have that 
\begin{equation*}
\Pr\left(\gamma_{[1]}=\ldots\gamma_{[r]}=0\right)
\geq 1-Crp_{\gamma}
e^{\frac{-\gamma_b^2}{4\sigma_{\varepsilon}^2}}-Ce^{-C\{(\consp^2 p_{\gamma}\Nb)\mins\log\Nb\}},%
\end{equation*}
where $C$ is a constant the does not depend on any sample sizes.
\end{lemma}

\begin{proof}
Under normal assumptions, it is easy to know that given any $u\geq 0$, random
variable $\varepsilon+u-\frac{\tx\tp\bdelta}{\sqrt{\Ns}}$ is
normally distributed and the density function of
random variable $\varepsilon+u-\frac{\tx\tp\deltaS}{\sqrt{\Nb}}$ given $u\geq
0$ is
\begin{align*}
f_{\varepsilon+u-\frac{\tx\tp\deltaS}{\sqrt{\Ns}}}(z)
=\frac{1}{\sqrt{2\pi}\sigma_{\Ze}}e^{-\frac{(z-u)^2}{2\sigma_{\Ze}^2}}.
\end{align*}
Thus, applying Bayesian's rule, we have that $\forall z>0$, 
\begin{align*}
&\Pr\left(\gamma=0\big||\Zg|=z,\Ds\right)
=\frac{\Pr(\gamma=0)f_{\left|\varepsilon-\frac{\tx\tp\deltaS}{\sqrt{\Ns}}\right|}(z)
}{\Pr(\gamma=0)f_{\left|\varepsilon-\frac{\tx\tp\deltaS}{\sqrt{\Ns}}\right|}(z)
+\Pr(\gamma\neq 0)\int_{\gamma_b}^{\infty}f_{\left|\varepsilon+u-\frac{\tx\tp\deltaS}{\sqrt{\Ns}}\right|}(z)f_{\tilde{\gamma}}(u)\dd
                 u}\\
&=\frac{(1-p_{\gamma})\frac{2}{\sqrt{2\pi}\sigma_{\Ze}}e^{-\frac{z^2}{2\sigma_{\Ze}^2}}}{
(1-p_{\gamma})\frac{2}{\sqrt{2\pi}\sigma_{\Ze}}e^{-\frac{z^2}{\sigma_{\Ze}^2}}
+p_{\gamma}\sum_{s\in\{-1,1\}}\frac{1}{\sqrt{2\pi}\sigma_{\Ze}}\int_{\gamma_b}^{\infty}e^{-\frac{(sz-u)^2}{2\sigma_{\Ze}^2}}
f_{\tilde{\gamma}}(u)\dd u}\\
&=1-\frac{\frac{1}{2}p_{\gamma}\sum_{s\in\{-1,1\}}\int_{\gamma_b}^{\infty}e^{\frac{2szu-u^2}{2\sigma_{\Ze}^2}}
f_{\tilde{\gamma}}(u)\dd u}{
1-p_{\gamma}+\frac{1}{2}p_{\gamma}\sum_{s\in\{-1,1\}}\int_{\gamma_b}^{\infty}e^{\frac{2szu-u^2}{2\sigma_{\Ze}^2}}
f_{\tilde{\gamma}}(u)\dd u}\\
&\geq 1-\frac{p_{\gamma}}{1-p_{\gamma}}\frac{1}{2}\sum_{s\in\{-1,1\}}\int_{\gamma_b}^{\infty}
\exp\left\{\frac{2szu-u^2}{2\sigma_{\Ze}^2}\right\}f_{\tilde{\gamma}}(u)\dd u
\geq 1-\frac{p_{\gamma}}{1-p_{\gamma}}\exp\left\{\frac{2z\gamma_b-\gamma_b^2}{2\sigma_{\Ze}^2}\right\}.
\end{align*}
The last inequality is because
$\exp\left(\frac{-2zu-u^2}{2\sigma_{\Ze}^2}\right)\le\exp\left(\frac{2zu-u^2}{2\sigma_{\Ze}^2}\right)$. Now,
letting
$K_{\gamma}:=\#\{\gamma_i\neq0,i=1,2,\ldots,\Nb\}$,
we consider a ``good'' set $\mathcal{A}$ that satisfies the following
conditions:
\begin{align}\label{eq:goodset}
K_{\gamma}\le\Nb-r,\text{ }
\frac{\|\deltaS\|_2^2}{\Ns}\le\frac{1}{2}\sigma_{\varepsilon}^2\text{, and}
\max_{1\le i\le\Nb}|\Ze|_i\le 2\sigma_{\Ze}\sqrt{\log\Nb}.
\end{align}
Under ``good'' set $\mathcal{A}$, using similar argument of \textbf{Claim}
(3.3.1) in~\cite{oliveira2025finite}, we can obtain that $\forall
i\in\{1,2,\ldots,r\}$, 
\begin{align*}
|\Zg|_{(i)}\le|\Ze|_{(i+K_{\gamma})}\le 2\sqrt{\sigma_{\varepsilon}^2+\frac{\|\deltaS\|_2^2}{\Ns}}
\sqrt{\log\Nb}\le\sigma_{\varepsilon}\sqrt{6\log\Nb}.
\end{align*}
Thus, we know that when $\gamma_b\geq
4\sqrt{6}\sigma_{\varepsilon}\sqrt{\log\Nb}$, applying Lemma 3.3.1
in~\cite{bhattacharya198418}, 
\begin{align*}
&\Pr\left(\gamma_{[1]}=\ldots\gamma_{[r]}=0\right)
=\Exp\left\{\Pr\left(\gamma_{[1]}=\ldots\gamma_{[r]}=0|\Ds,|\Zg|_{(1)},\ldots|\Zg|_{(\Nb)}\right)\right\}\\
&\geq\int_{\mathcal{A}}\Pr\left(\gamma_{[1]}=\ldots\gamma_{[r]}=0|\Ds,|\Zg|_{(1)},\ldots|\Zg|_{(\Nb)}\right)\dd\Pr\\
&%
\geq\int_{\mathcal{A}}\prod_{i=1}^r\left[1-\frac{p_{\gamma}}{1-p_{\gamma}}
\exp\left\{\frac{2|\Zg|_{(i)}\gamma_b-\gamma_b^2}{2\sigma_{\Ze}^2}\right\}\right]\dd\Pr\\
&\geq\left\{1-\frac{p_{\gamma}}{1-p_{\gamma}}
\exp\left(\frac{2\sqrt{6}\sigma_{\varepsilon}\sqrt{\log\Nb}\gamma_b-\gamma_b^2}{
2\sigma_{\varepsilon}^2}\right)\right\}^r\Pr(\mathcal{A})\\
&\geq\left\{1-\frac{p_{\gamma}}{1-p_{\gamma}}
\exp\left(\frac{\frac{\gamma_b}{2}-\gamma_b^2}{
2\sigma_{\varepsilon}^2}\right)\right\}^r\Pr(\mathcal{A})
=\left(1-\frac{p_{\gamma}}{1-p_{\gamma}}
e^{\frac{-\gamma_b^2}{4\sigma_{\varepsilon}^2}}\right)^r\Pr(\mathcal{A})\\
&=\left(1-\frac{p_{\gamma}}{1-p_{\gamma}}
e^{\frac{-\gamma_b^2}{4\sigma_{\varepsilon}^2}}\right)^r\left\{1-\Pr(\mathcal{A}^c)\right\}
\geq\left(1-\frac{p_{\gamma}}{1-p_{\gamma}}
e^{\frac{-\gamma_b^2}{4\sigma_{\varepsilon}^2}}\right)^r-\Pr(\mathcal{A}^c).
\end{align*}
Now, we only need to calculate the probability of $\mathcal{A}^c$.
Note that $I(\gamma_i\neq 0),i=1,2,\ldots,\Nb$ is a Bernoulli random variable
with probability $p_{\gamma}$ and therefore, $K_{\gamma}=\sum_{i=1}^{\Nb}I(\gamma_i\neq 0)$ is a binomial random
variable. Using Theorem 3 in~\cite{okamoto1959some}, we have that $\forall c>0$, 
\begin{align*}
&\Pr\left\{K_{\gamma}>\Nb(\sqrt{p_{\gamma}}+\sqrt{c})^2\right\}
=\Pr\left\{\sum_{i=1}^{\Nb}I(\gamma_i\neq0)>\Nb(\sqrt{p_{\gamma}}+\sqrt{c})^2\right\}\\
&=\Pr\left\{\sqrt{\frac{Binomial(\Nb,p_{\gamma})}{\Nb}}
>\sqrt{p_{\gamma}}+\sqrt{c}\right\}
<e^{-2c\Nb}.
\end{align*}
Taking $c=\frac{\consp^2 p_{\gamma}}{16}$, we have that with probability at least
$1-e^{-\frac{\consp^2 p_{\gamma}\Nb}{8}}$,
\begin{equation}\label{eq:Kgamma1}
K_{\gamma}\le p_{\gamma}\Nb\left(1+\frac{\consp^2}{16}+\frac{\consp}{2}\right)  
<(1+\consp)p_{\gamma}\Nb,
\end{equation}
where the last inequality is because $\consp<8$. Therefore, as long as
$r\le\{1-(1+\consp)p_{\gamma}\}\Nb$, we have that
$K_{\gamma}\le\Nb-r$ with probability at least $1-e^{-\frac{\consp^2p_{\gamma}\Nb}{8}}$. We also have
that with probability at least $1-\frac{C}{\sqrt{\Ns}}$,
$\|\deltaS\|_2^2\le C\sigma_{\cS}^2\cvsigma$. Therefore, when $\Ns\geq
2C\sigma_{\cS}^2\cvsigma/\sigma_{\varepsilon}^2$, we know that
$\frac{\|\deltaS\|_2^2}{\Ns}\le\frac{1}{2}\sigma_{\varepsilon^2}$ with
probability at least $1-C/\sqrt{\Ns}$. Finally, since $\Ze_i$,
$i=1,2,3,\ldots,\Nb$, are normal distributed with variances $\sigma_{\Ze}^2$,
which implies $\forall\cvsigma_1>0$, $\max_{1\le
  i\le\Nb}|\Ze|_i\le\sigma_{\Ze}\sqrt{2(\log\Nb+\cvsigma_1)}$ with probability at
least $1-e^{-\cvsigma_1}$. Taking $\cvsigma_1=\log\Nb$, we have that $\max_{1\le
  i\le\Nb}|\Ze|_i\le 2\sigma_{\Ze}\sqrt{\log\Nb}$ with
probability at least $1-\Nb^{-1}$, Thus, we have that 
\begin{equation*}
\Pr(\mathcal{A}^c)\le\frac{C}{\sqrt{\Ns}}+e^{-\frac{\consp^2 p_{\gamma}\Nb}{8}}+\frac{1}{\Nb},
\end{equation*}
which completes the proof.

\end{proof}

We highlight the following corollary obtained from the proof of
Lemma~\ref{lem:nobias2}. Assume that we sort $|\Ze|_e$,
$e=1,\ldots,\Nb$, such that
$|\Ze|_{(1)}<|\Ze|_{(2)}<\cdots<|\Ze|_{(\Nb)}$. Define $i[k]$ to be the index
$j$ such that $|\Zg|_{(k)}=|\Ze|_{(j)}$ when $\gamma=0$. Our definition of index
$i[k]$ can be interpreted as follows. Since
$|\Zg|=\varepsilon+\gamma-\frac{\tx\tp\deltaS}{\sqrt{\Ns}}$ can be considered as
a contaminated version of clean sample
$|\Ze|=\varepsilon-\frac{\tx\tp\deltaS}{\sqrt{\Ns}}$. The order statistics of
contaminated samples should correspond to a clean sample in the original clean
samples. Therefore, the index $i[k]$ tells us which clean sample (in the sense
of order statistics) is the $k$-th ordered contaminated sample $|\Zg|_{(k)}$
corresponding to.
\begin{corollary}\label{cor:orderstat}
Under the same Assumptions in Lemma~\ref{lem:nobias2}, we further have that with
probability at least
$1-Crp_{\gamma}e^{\frac{-\gamma_b^2}{4\sigma_{\varepsilon}^2}}-Ce^{-C\{(\consp^2 p_{\gamma}\Nb)\mins\log\Nb\}}$, 
we have that $\forall k=1,2,\ldots,r$ 
\begin{equation*}
|\Ze|_{(i[k])}=|\Zg|_{(k)}\le 2\sigma_{\Ze}\sqrt{\log\Nb}.
\end{equation*}
\end{corollary}
\begin{proof}
The result can be derived directly from the fact that under event
$\{\gamma_{[1]}=\gamma_{[2]}=\ldots\gamma_{[r]}=0\}\cap\mathcal{A}$, where
$\mathcal{A}$ is the ``good'' set defined in the proof of
Lemma~\ref{lem:nobias2}, we have that
\begin{equation}\label{eq:Kgamma2}
|\Ze|_{(i[k])}=|\Zg|_{(k)}\le|\Ze|_{(k+K_{\gamma})},
\end{equation}
and $\max_{1\le
  i\le\Nb}|\Ze|_i\le 2\sigma_{\Ze}\sqrt{\log\Nb}$.
\end{proof}

\begin{lemma}\label{lem:concox}
Under Assumption~\ref{asm:gamma} and assume that
$\varepsilon\sim\Nor(0,\sigma_{\varepsilon}^2)$ and $\x\sim\Nor(\0,\bSigma)$,
$\forall\cvsigma>0,0<\consp<8$, as long as $\Ns\gtrsim e^{2\cvsigma}$,
$r\le\{1-(1+\consp)p_{\gamma}\}\Nb$, $\gamma_b\geq
4\sqrt{6}\sigma_{\varepsilon}\sqrt{\log\Nb}$, and there exists a
constant $0<\alpha_3<1$, such that $r\geq 2^{\frac{2}{\alpha_3}}$ and
$r^{\frac{\alpha_3}{2}}\Ns\gtrsim\cvsigma\log\Nb$, we have that with
probability at least
$1-Crp_{\gamma}e^{-\frac{\gamma_b^2}{4\sigma_{\varepsilon}^2}}-Ce^{-C\{(\consp^2
  p_{\gamma}\Nb)\mins\log\Nb\}}-2e^{-Cr^{1-\alpha_3}}$,  
\begin{equation*}
\mineg\left(\frac{1}{r}\sum_{i=1}^r\tx_{[i]}\tx_{[i]}\tp\right)\geq\frac{1}{8}.
\end{equation*}
\end{lemma}

\begin{proof}

Letting %
$\mu_{\x,c}=-\frac{\bu\tp\deltaS}{\sqrt{\Ns}\sigma_{\Ze}^2}$ and
$\sigma_{\x,c}^2=1-\frac{1}{\Ns}\frac{\|\deltaS\|_2^2}{\sigma_{\Ze}^2}$, we know
that $\forall\bu\in\mathbb{S}^d$, 
\begin{equation*}
f_{\x\tp\bu}^{|\Ds,|\Ze|=z_3}(z)
=\frac{1}{2}\sum_{s\in\{-1,1\}}\frac{1}{\sqrt{2\pi}\sigma_{\x,c}}
e^{-\frac{(z-s\mu_{\x,c}z_3)^2}{2\sigma_{\x,c}^2}}.
\end{equation*}
It is easy to know that
$\Exp\left(\tx\tp\bu\big||\Ze|=z_3,\Ds\right)=0$,
and therefore, $\forall\lambda\in\mathbb{R}$,
\begin{align*}
&\Exp\left(e^{\lambda\tx\tp\bu}\big||\Ze|=z_3,\Ds\right)
=\frac{e^{\lambda\mu_{\x,c}z_3}+e^{-\lambda\mu_{\x,c}z_3}}{2}
e^{\frac{\lambda^2\sigma_{\x,c}^2}{2}}\le\exp\left\{\frac{\lambda^2(\mu_{\x,c}^2z_3^2+\sigma_{\x,c}^2)}{2}\right\}\\
&=\exp\left\{\frac{\lambda^2\left(1-\frac{\|\deltaS\|_2^2}{\Ns\sigma_{\Ze}^2}
+\frac{(\bu\tp\deltaS)^2z_3^2}{\Ns\sigma_{\Ze}^4}\right)}{2}\right\}
:=e^{\frac{\lambda^2K(z_3)}{2}}.
\end{align*}
Therefore, conditional on $\Ds$ and $|\Ze|=z_3$, $\tx\tp\bu$ is a
sub-gaussian random variable and %
\begin{equation*}
\Exp\left\{(\tx\tp\bu)^2\big||\Ze|=z_3,\Ds\right\}=\mu_{\x,c}^2z_3^2+\sigma_{\x,c}^2
=1-\frac{\|\deltaS\|_2^2}{\Ns\sigma_{\Ze}^2}
+\frac{(\bu\tp\deltaS)^2z_3^2}{\Ns\sigma_{\Ze}^4}=K(z_3).
\end{equation*}
Now, we know that given $\Ds$ and $|\Ze|=z_3$, $\tx\tp\bu$ is sub-gaussian with
variance proxy $K(z_3)$, Applying Lemma 3.1 in \cite{bhattacharya198418} and
Hanson-Wright inequality, we have that $\forall t>0$, 
\begin{align*}
&\Pr\left(\left|\frac{1}{r}\sum_{k=1}^r(\bu\tp\tx_{[k]})^2
-\frac{1}{r}\sum_{k=1}^r K(|\Ze|_{(i[k])})\right|>t\Big||\Zg|_{(k)},k=1,\ldots,r,\Ds\right)\\
&=\Pr\left(\left|\frac{1}{r}\sum_{k=1}^r(\bu\tp\tx_{[i[k]]})^2
-\frac{1}{r}\sum_{k=1}^r K(|\Ze|_{(i[k])})\right|>t\Big||\Ze|_{(i[k])},k=1,\ldots,r,\Ds\right)\\
&\le2\exp\left\{-C\min\left(\frac{rt^2}{A^2},\frac{rt}{A}\right)\right\},
\end{align*}
where $C$ is a constant that does not dependent on any sample sizes and 
\begin{equation*}
A=\max_{1\le k\le r}K(\Ze_{(i[k])})=1-\frac{\|\deltaS\|_2^2}{\Ns\sigma_{\Ze}^2}
+\frac{(\bu\tp\deltaS)^2}{\Ns\sigma_{\Ze}^4}\max_{1\le k\le r}|\Ze|_{(i[k])}^2.
\end{equation*}
Letting $t=\frac{A}{r^{\frac{\alpha_3}{2}}}$ and taking expectation on both
sides, we obtain that 
\begin{align*}
&\Pr\left\{\left|\frac{1}{r}\sum_{k=1}^r(\bu\tp\tx_{[k]})^2
-\frac{1}{r}\sum_{k=1}^rK(|\Ze|_{(i[k])})\right|>\frac{1}{r^{\frac{\alpha_3}{2}}}
\left(1-\frac{\|\deltaS\|_2^2}{\Ns\sigma_{\Ze}^2}
+\frac{(\bu\tp\deltaS)^2}{\Ns\sigma_{\Ze}^4}\max_{1\le k\le r}|\Ze|_{(i[k])}^2\right)\right\}\\
&=\Exp\left\{\Pr\left(\left|\frac{1}{r}\sum_{k=1}^r(\bu\tp\tx_{[i[k]]})^2
-\frac{1}{r}\sum_{k=1}^r K(|\Ze|_{(i[k])})\right|>t\Big||\Ze|_{(i[k])},k=1,\ldots,r,\Ds\right)\right\}\\
&\le\Exp\left[2\exp\left\{-C\min\left(r^{1-\alpha_3},r^{1-\frac{\alpha_3}{2}}\right)\right\}\right]
=2e^{-Cr^{1-\alpha_3}}.
\end{align*}
Therefore, we have that with probability at least $1-2e^{-Cr^{1-\alpha_3}}$,
\begin{align*}
&\frac{1}{r}\sum_{k=1}^r(\bu\tp\tx_{[k]})^2\\
&\geq \left(1-\frac{1}{r^{\frac{\alpha_3}{2}}}\right)\left(1-\frac{\|\deltaS\|_2^2}{\Ns\sigma_{\Ze}^2}\right)
+\frac{(\bu\tp\deltaS)^2}{\Ns\sigma_{\Ze}^4}\frac{1}{r}\sum_{k=1}^r|\Ze|_{(i[k])}^2
-\frac{(\bu\tp\deltaS)^2}{\Ns\sigma_{\Ze}^4}\frac{1}{r^{\frac{\alpha_3}{2}}}\max_{1\le k\le r}|\Ze|_{(i[k])}^2\\
&\geq\left(1-\frac{1}{r^{\frac{\alpha_3}{2}}}\right)\left(1-\frac{\|\deltaS\|_2^2}{\Ns\sigma_{\Ze}^2}\right)
-\frac{\|\deltaS\|_2^2}{\Ns\sigma_{\Ze}^4}\frac{1}{r^{\frac{\alpha_3}{2}}}\max_{1\le k\le r}|\Ze|_{(i[k])}^2.
\end{align*}
Since $\|\deltaS\|_2^2\le C\sigma_{\cS}^2\cvsigma$ and $\max_{1\le k\le
  r}|\Ze|_{(i[k])}\le 2\sigma_{\Ze}\sqrt{\log\Nb}$ with probability at
least
$1-Crp_{\gamma}e^{-\frac{\gamma_b^2}{4\sigma_{\varepsilon}^2}}-Ce^{-C\{(\consp^2
  p_{\gamma}\Nb)\mins\log\Nb\}}$
according to Lemma~\ref{lem:nobias2} and Corollary~\ref{cor:orderstat}, we know
that 
\begin{align*}
\frac{\|\deltaS\|_2^2}{\Ns\sigma_{\Ze}^2}
&\le \frac{\sigma_{\cS}^2\cvsigma}{\Ns\sigma_{\varepsilon}^2}\le\frac{1}{2},\\
\frac{\|\deltaS\|_2^2}{\Ns\sigma_{\Ze}^4}\frac{1}{r^{\frac{\alpha_3}{2}}}\max_{1\le k\le r}|\Ze|_{(i[k])}^2
&\le\frac{\sigma_{\cS}^2\cvsigma}{\Ns\sigma_{\varepsilon}^2}
\frac{4\log\Nb}{r^{\frac{\alpha_3}{2}}}\le\frac{1}{8},
\end{align*}
when $\Ns\geq 2C\sigma_{\cS}^2\cvsigma/\sigma_{\varepsilon}^2$ and
$r^{\frac{\alpha_3}{2}}\Ns\geq
32\{\sigma_{\cS}^2\cvsigma\log\Nb\}/\sigma_{\varepsilon}^2$. Taking
$r\geq 2^{\frac{2}{\alpha_3}}$, we obtain that 
\begin{equation*}
\frac{1}{r}\sum_{k=1}^r(\bu\tp\tx_{[k]})^2\geq \frac{1}{2}\left(1-\frac{1}{2}\right)-\frac{1}{8}=\frac{1}{8}.
\end{equation*}
Therefore, we know that
$\mineg\left(\frac{1}{r}\sum_{i=1}^r\tx_{[i]}\tx_{[i]}\tp\right)\geq \frac{1}{8}$.
\end{proof}
Next, we prove a lemma of order statistics of $\chi_1^2$ distribution that we
will use later.
\begin{lemma}\label{lem:chisq}
Letting $X_1,X_2,\ldots,X_n$ are i.i.d. $\chi_1^2$ distributed random variables,
we have that $\forall 1\le r\le n$, the 
order statistics $X_{(1)}<X_{(2)}<\cdots <X_{(r)}$ satisfy that
\begin{equation}\label{eq:chisq1}
\left|\frac{1}{r}\sum_{i=1}^r(1-X_{(i)})\right|\le Q(U_{(r+1)})+\frac{4\cvsigma}{\sqrt{r}}
+\frac{\cvsigma^2\log n}{3r},
\end{equation}
with probability at least $1-2e^{-\frac{\cvsigma^2}{2}}-n^{-1}$, where 
\begin{equation*}
Q(u)=\frac{2\sqrt{F_{\chi_1^2}^{-1}(u)}\exp\{-\frac{1}{2}F_{\chi_1^2}^{-1}(u)\}}{\sqrt{2\pi}u},
\end{equation*}
$U_{(r+1)}=F_{\chi_1^2}(X_{(r+1)})$ follows the same distribution as the
$r+1$-th order statistics of $\{U_i\}_{i=1}^n$, where $U_i\sim U[0,1]$, and
$F_{\chi_1^2}^{-1}(u)$ is the quantile function of $\chi_1^2$ 
distribution. Further, we have that $Q(u)$ is continuous, decreasing,
$Q(0)=1$ and $Q(1)=0$.
\end{lemma}

\begin{proof}
Our proof is very similar to the proof of Theorem 5.1.1 in \cite{oliveira2025finite}
but slightly different since $r^{-1}\sum_{i=1}^rX_{(i)}$ can be considered as a one-sided
trimmed mean. However, the proof of Theorem 5.1.1 in \cite{oliveira2025finite} only
covers two-sided trimmed mean. Note that we can represent $X_i$'s and their
order statistics as $X_i=F_{\chi_1^2}^{-1}(U_i)$ and
$X_{(i)}=F_{\chi_1^2}^{-1}(U_{(i)})$, where $\{U_i\}_{i=1}^n$ are i.i.d. random
variables distributed as uniform distribution $U[0,1]$. We denote $\Delta_{n,r}=X_{(r+1)}$. Due to
Corollary 3.2.3 in \cite{oliveira2025finite}, we know that conditional on
$U_{(r+1)}=b$, $X_{(i)},1\le i\le r$ are i.i.d. random variables with mean
$\mu^{(0,b)}$ and standard deviation
$\sigma^{(0,b)}$, where $\forall 0<b<1$ 
\begin{align*}
\mu^{(0,b)}=\Exp\left\{F_{\chi_1^2}^{-1}(U^{(0,b)})\right\},
\{\sigma^{(0,b)}\}^2=\Var\left\{F_{\chi_1^2}^{-1}(U^{(0,b)})\right\}.
\end{align*}
and $U^{(0,b)}$ is a uniform random variable over $(0,b)$. Here, our definition
of $\mu^{(0,b)}$ and $\sigma^{(0,b)}$ are same as
Definition 3.2.2 in~\cite{oliveira2025finite} ans similarly to the proof of
Theorem 5.1.1 in~\cite{oliveira2025finite}, the random variable in the average,
when centered are bounded by $\Delta_{n,r}$ in absolute value. Therefore,
Bernstein's inequality (Theorem 2.3.1 in~~\cite{oliveira2025finite}) can be used
conditionally on $U_{(r+1)}$ to obtain the following inequality, $\forall\cvsigma>0$,  
\begin{align*}
&\Pr\left\{\left|\frac{1}{r}\sum_{i=1}^rX_{(i)}-\mu^{(0,U_{(r+1)})}\right|
>\frac{\cvsigma\sigma^{(0,U_{(r+1)})}}{\sqrt{r}}
+\frac{\cvsigma^2\Delta_{n,r}}{12r}\right\}\\
&=\Exp\left[\Pr\left\{\left|\frac{1}{r}\sum_{i=1}^rX_{(i)}-\mu^{(0,U_{(r+1)})}\right|
>\frac{\cvsigma\sigma^{(0,U_{(r+1)})}}{\sqrt{r}}
+\frac{\cvsigma^2\Delta_{n,r}}{12r}\Big|U_{(r+1)}\right\}\right]
\le 2e^{-\frac{\cvsigma^2}{2}}.
\end{align*}
Thus, with probability at least $1-2e^{-\frac{\cvsigma^2}{2}}$, we have that 
\begin{align*}
&\left|\frac{1}{r}\sum_{i=1}^r(1-X_{(i)})\right|
\le|\mu^{(0,U_{(r+1)})}-1|+\left|\frac{1}{r}\sum_{i=1}^r(X_{(i)}-\mu^{(0,U_{(r+1)})})\right|\\
&\le|\mu^{(0,U_{(r+1)})}-1|+\frac{\cvsigma\sigma^{(0,U_{(r+1)})}}{\sqrt{r}}
+\frac{\cvsigma^2X_{(r+1)}}{12r}.
\end{align*}
Our next goal is to bound $|\mu^{(0,U_{(r+1)})}-1|$,
$\sigma^{(0,U_{(r+1)})}$, and $X_{(r+1)}$, respectively. We start with
calculating $\mu^{(0,b)}$ and $\sigma^{(0,b)}$ for every $b\in(0,1)$.
Denoting random variable $F_{\chi_1^2}^{-1}(U^{(0,b)})=\xi^2$, we know that
$\mu^{(0,b)}=\Exp(\xi^2)$ and
$\{\sigma^{(0,b)}\}^2=\Var(\xi^2)\le\Exp(\xi^4)$. It is easy to calculate the
density function of $F_{\chi_1^2}^{-1}(U^{(0,b)})$, which is 
\begin{equation*}
f_{F_{\chi_1^2}^{-1}(U^{(0,b)})}(z)=\frac{f_{\chi_1^2}(z)}{b}I\left\{z\le F_{\chi_1^2}^{-1}(b)\right\}.
\end{equation*}
Therefore, 
\begin{align*}
&f_{\xi}(z)=\frac{\phi(z)}{b}
I\left\{-\sqrt{F_{\chi_1^2}^{-1}(b)}\le z\le\sqrt{F_{\chi_1^2}^{-1}(b)}\right\}\\
&=\frac{\phi(z)}{\Phi\left\{\sqrt{F_{\chi_1^2}^{-1}(b)}\right\}-\Phi\left\{-\sqrt{F_{\chi_1^2}^{-1}(b)}\right\}}
I\left\{-\sqrt{F_{\chi_1^2}^{-1}(b)}\le z\le\sqrt{F_{\chi_1^2}^{-1}(b)}\right\},
\end{align*}
where $\phi(\cdot)$ is the density function of standard normal distribution and
$\Phi(\cdot)$ is the c.d.f. of standard normal distribution.
Therefore, we know that $\xi$ is distributed as a symmetric truncated normal
distribution. Due to Theorem 2.2 in \cite{barreto2024optimal}, we have that
$\xi$ is sub-gaussian with a 
variance proxy smaller than 1, which implies $\Exp(\xi^4)\le 16$. Applying (2)
in~\cite{barreto2024optimal}, we have that 
\begin{align*}
&\mu^{(0,b)}=\Exp(\xi^2)=\Var(\xi)
=1-\frac{2\sqrt{F_{\chi_1^2}^{-1}(b)}\phi\left(\sqrt{F_{\chi_1^2}^{-1}(b)}\right)}{
\Phi\left(\sqrt{F_{\chi_1^2}^{-1}(b)}\right)-\Phi\left(-\sqrt{F_{\chi_1^2}^{-1}(b)}\right)}\\  
&=1-\frac{2\sqrt{F_{\chi_1^2}^{-1}(b)}\exp\{-\frac{1}{2}F_{\chi_1^2}^{-1}(b)\}}{\sqrt{2\pi}b}
=1-Q(b).
\end{align*}
For $\sigma^{(0,b)}$, we have already proven that
$\sigma^{(0,b)}\le\sqrt{\Exp(\xi^4)}\le 4$. As for $X_{(r+1)}$, since
$X_i\sim\chi_1^2$, it is easy to know that with probability at least
$1-n^{-1}$, we have that $X_{(r+1)}\le\max_{1\le i\le n}X_i\le 4\log
n$. Therefore, with probability at least
$1-2e^{-\frac{\cvsigma^2}{2}}-n^{-1}$, 
\begin{align*}
\left|\frac{1}{r}\sum_{i=1}^r(1-X_{(i)})\right|\le |Q(U_{(r+1)})|+\frac{4\cvsigma}{\sqrt{r}}
+\frac{\cvsigma^2\log n}{3r},
\end{align*}
and the bound in~\eqref{eq:chisq1} holds since $Q(u)\geq 0,\forall u\geq 0$.
To prove $Q(u)$ is decreasing and $Q(0)=1$, $Q(1)=0$, we start with some
properties of the following function:
\begin{equation*}
\tilde{Q}(u)=\frac{2u\phi(u)}{\Phi(u)-\Phi(-u)}
=\frac{ue^{-\frac{u^2}{2}}}{\int_0^ue^{-\frac{t^2}{2}}\dd t}, u\geq 0,
\end{equation*}
We have that 
\begin{equation*}
\frac{\dd\tilde{Q}(u)}{\dd u}
=\frac{e^{-\frac{u^2}{2}}\{(1-u^2)\int_0^ue^{-\frac{t^2}{2}}\dd t-ue^{-\frac{u^2}{2}}\}}{
\left(\int_0^ue^{-\frac{t^2}{2}}\dd t\right)^2}.
\end{equation*}
Letting $g(u)=(1-u^2)\int_0^ue^{-\frac{t^2}{2}}\dd t-ue^{-\frac{u^2}{2}}$, we
have that $g(0)=0$ and the derivative $g'(u)=-2u\int0^ue^{-\frac{t^2}{2}}\dd
t\le 0,\forall u\geq 0$. Thus, $g(u)\le g(0)=0,\forall u\geq 0$.
Therefore, we
have that $\tilde{Q}(u)$ is decreasing. We also have that
\begin{align*}
\tilde{Q}(0)=\lim_{u\to0}\frac{ue^{-\frac{u^2}{2}}}{\int_0^ue^{-\frac{t^2}{2}}\dd
  t}=\lim_{u\to0}(1-u^2)=1\text{, and }\tilde{Q}(\infty)=\lim_{u\to\infty}2e^{-\frac{u^2}{2}}=0,
\end{align*}
which implies $Q(u)$ is decreasing, $Q(0)=\tilde{Q}(0)=1$, and
$Q(1)=\tilde{Q}(\infty)=0$, since $\sqrt{F_{\chi_1^2}^{-1}(u)}$ is an increasing function,
$\sqrt{F_{\chi_1^2}^{-1}(0)}=0$, and $\sqrt{F_{\chi_1^2}^{-1}(1)}=\infty$.

\end{proof}

\begin{lemma}\label{lem:deteps}
Under Assumption~\ref{asm:gamma} and assume that
$\varepsilon\sim\Nor(0,\sigma_{\varepsilon}^2)$ and $\x\sim\Nor(\0,\bSigma)$,
$\forall\cvsigma>0$, $0<\consp<8$, as long as $\Ns\gtrsim e^{2\cvsigma}$,
$r\le\{1-(1+\consp)p_{\gamma}\}\Nb$, $\gamma_b\gtrsim\sqrt{\log\Nb}$, we have that with
probability at least
$1-Ce^{-\cvsigma}-Ce^{-\{C(p_{\gamma}\Nb)\mins\log\Nb\}}-e^{-C\sqrt{r}}$,
\begin{align*}
\frac{1}{r}\left\|\sum_{k=1}^r\varepsilon_{[k]}\tx_{[k]}\tp\right\|_2
\lesssim\frac{\sigma_{\cS}\sqrt{\cvsigma}}{\sqrt{\Ns}}Q\left\{\left(1-\frac{1}{\log\Nb}\right)\frac{r}{\Nb}\right\}
+\frac{\sqrt{\cvsigma}\{(\cvsigma^2\log\Nb)\maxs (p_{\gamma}\Nb)\}}{\sqrt{r}}.
\end{align*} 
\end{lemma}
\begin{proof}
For any $\bu\in\mathbb{S}^{d-1}$, we have that the joint distribution of
$(\varepsilon,\tx\tp\bu,\varepsilon-\tx\tp\deltaS/\sqrt{\Ns})$ conditional on $\deltaS$ is a
multivariate normal distribution, which is given as 
\begin{align*}
f^{|\Ds}_{(\varepsilon,\tx\tp\bu,\Ze)}(\z_f)
&=\frac{1}{(2\pi)^{\frac{3}{2}}\sqrt{|\bSigma_f|}}e^{-\frac{1}{2}\z_f\tp\bSigma_f^{-1}\z_f}\\
&=\frac{1}{2\pi\sqrt{|\bSigma_c|}}e^{-\frac{1}{2}(\z_c-\bmu_c
  z_3)\tp\bSigma_c^{-1}(\z_c-\bmu_c z_3)}
\frac{1}{\sqrt{2\pi}\sigma_{\Ze}}e^{-\frac{z_3^2}{2\sigma_{\Ze}^2}}.
\end{align*}
Thus,
\begin{align*}
&f^{|\Ds}_{(\varepsilon,\tx\tp\bu,|\Ze|=z)}(\z_f)
=\frac{1}{2}\sum_{s\in\{-1,1\}}\frac{1}{2\pi\sqrt{|\bSigma_c|}}e^{-\frac{1}{2}(\z_c-s\bmu_c
  z)\tp\bSigma_c^{-1}(\z_c-s\bmu_c z)}
\frac{1}{\sqrt{2\pi}\sigma_{\Ze}}e^{-\frac{(sz)^2}{2\sigma_{\Ze}^2}}\\
&=\left\{\frac{1}{2}\sum_{s\in\{-1,1\}}\frac{1}{2\pi\sqrt{|\bSigma_c|}}e^{-\frac{1}{2}(\z_c-s\bmu_c
  z_3)\tp\bSigma_c^{-1}(\z_c-s\bmu_c z_3)}\right\}\frac{1}{\sqrt{2\pi}\sigma_{\Ze}}e^{-\frac{z^2}{2\sigma_{\Ze}^2}}.
\end{align*}
which implies
\begin{align*}
f^{|\Ds,|\Ze|=z_3}_{(\varepsilon,\tx\tp\bu)}(\z_f)
=\frac{1}{2}\sum_{s\in\{-1,1\}}\frac{1}{2\pi\sqrt{|\bSigma_c|}}e^{-\frac{1}{2}(\z_c-s\bmu_c
  z_3)\tp\bSigma_c^{-1}(\z_c-s\bmu_c z_3)}.
\end{align*}
Therefore, we know that conditional on $\Ds$ and $|\Ze|=z_3$, random vector
$(\varepsilon,\tx\tp\bu)\tp$ follows $\xi\bv$, where $\xi$
is a random variable that taking values $-1$ and $1$ with probability
$\frac{1}{2}$ and $\bv$ follows $\Nor(\bmu_cz_3,\bSigma_c)$.
We have
$\varepsilon\tx\tp\bu=(\varepsilon,\tx\tp\bu)\bm{A}(\varepsilon,\tx\tp\bu)\tp
=\xi\bv\tp\bm{A}\xi\bv=\bv\tp\bm{A}\bv$,
where 
\begin{align*}
\bm{A}=
\begin{pmatrix}
0 & \frac{1}{2}\\
\frac{1}{2} & 0\\
\end{pmatrix}.
\end{align*}
Thus, $\forall\bu\in\mathbb{S}^{d-1}$,
\begin{align*}
&\Exp(\varepsilon\tx\tp\bu||\Ze|=z_3,\Ds)
=\Exp(\bv\tp\bm{A}\bv)
=trace(\bm{A}\bSigma_c)+\bmu_c\tp\bSigma_c\bmu_c z_3^2\\
&=\frac{\sigma_{\varepsilon}^2\bu\tp\deltaS}{\sqrt{\Ns}\sigma_{\Ze}^2}
-\frac{\sigma_{\varepsilon}^2\bu\tp\deltaS z_3^2}{\sqrt{\Ns}\sigma_{\Ze}^4}
=\frac{\sigma_{\varepsilon}^2\bu\tp\deltaS(\sigma_{\Ze}^2-z_3^2)}{\sqrt{\Ns}\sigma_{\Ze}^4}
:=m(z_3)
\end{align*}
and similarly, $\forall\bu\in\mathbb{S}^{d-1}$,
\begin{align*}
&\Var(\varepsilon\tx\tp\bu||\Ze|=z_3,\Ds)=\Var(\bv\tp\bm{A}\bv)
=2trace(\bm{A}\bSigma_c\bm{A}\bSigma_c)+4\bmu_c\tp
\bm{A}\bSigma_c\bm{A}\bmu_c z_3^2\\
&=\frac{\sigma_{\varepsilon}^4(\bu\tp\deltaS)^2}{\Ns\sigma_{\Ze}^4}
+\frac{\|\deltaS\|_2^2}{\Ns\sigma_{\Ze}^2}
-\frac{\|\deltaS\|_2^4}{\Ns\sigma_{\Ze}^4}
+\left\{\frac{(\bu\tp\deltaS)^2}{\Ns\sigma_{\Ze}^6}\left(
\frac{\|\deltaS\|_2^2}{\Ns}-2\sigma_{\varepsilon}^4\right)+\frac{\sigma_{\varepsilon}^6}{\sigma_{\Ze}^6}\right\}z_3^2\\
&\le\left(1+\frac{1}{\sigma_{\varepsilon}^2}\right)\frac{\|\deltaS\|_2^2}{\Ns}+z_3^2,
\end{align*}
as long as $\Ns\geq\frac{\|\deltaS\|_2^2}{2\sigma_{\varepsilon}^4}$.
Applying Lemma 3.1 in~\cite{bhattacharya198418}, we have that 
\begin{align*}
&\Pr\left[\left|\sum_{k=1}^r\left\{\varepsilon_{[k]}\tx_{[k]}\tp\bu
-m(|\Ze|_{(i[k])})\right\}\right|>A_1\right]\\
&=\Exp\left(\Pr\left[\left|\sum_{k=1}^r\left\{\varepsilon_{[k]}\tx_{[k]}\tp\bu
-m(|\Ze|_{(i[k])})\right\}\right|>A_1\Big||\Zg|_{(k)},1\le k\le r,\Ds\right]\right)\\
&=\Exp\left(\Pr\left[\left|\sum_{k=1}^r\left\{\varepsilon_{[i[k]]}\tx_{[i[k]]}\tp\bu
-m(|\Ze|_{(i[k])})\right\}\right|>A_1\Big||\Ze|_{(i[k])},1\le k\le r,\Ds\right]\right)\\
&\le\Exp\left\{\frac{\sum_{k=1}^r\Var\left(\varepsilon\tx\tp\bu\big||\Ze|=|\Ze|_{(i[k])},\Ds\right)}{A_1^2}\right\}
=\frac{1}{\cvsigma},
\end{align*}
where 
\begin{align*}
A_1=\sqrt{\cvsigma\sum_{k=1}^r \left\{|\Ze|_{(i[k])}^2
+\left(1+\frac{1}{\sigma_{\varepsilon}^2}\right)\frac{\|\deltaS\|_2^2}{\Ns}\right\}},
\end{align*}
and the last inequality is due to Chebyshev's inequality. We then bound the value 
\begin{equation*}
\left|\sum_{k=1}^r\left(1-\left|\frac{\Ze_{(i[k])}}{\sigma_{\Ze}}\right|^2\right)\right|.
\end{equation*}
For any fixed $K_{\gamma}$, if $\gamma_b\geq|\Ze|_{(\Nb)}-\Ze_{(1)}$, we know
that all the data points with nonzero
bias terms satisfy $|\Zg|_i>|\Ze|_{(\Nb)}$, $\forall i\in\{j:\gamma_j\neq
0\}$. In this case, it is equivalent to say that we put all the ``contaminated''
points $|\Zg|_i$, $i\in\{j:\gamma_j\neq 0\}$ after the value
$|\Ze|_{(\Nb)}$. Therefore, as long as $r<\Nb-K_{\gamma}$, we know that for any
``clean'' sample $|\Ze|_{(i)}$, $i=1,2,\ldots,r$, it is either be remained at the same
position or be put after $|\Ze|_{(\Nb)}$. Note that under this circumstance, we
know that as long as $r<\Nb-K_{\gamma}$, the selected points are data points
with no bias terms. Therefore, $\forall k=1,2,\ldots,r$, we have that
$|\Zg|_{(k)}=|\Ze|_{(i[k])}$. Using similar arguments as in~Claim 3.3.1
in~\cite{oliveira2025finite}, we can conclude that under the scenario we
considered,
$|\Ze|_{(k)}\le|\Ze|_{(i[k])}=|\Zg|_{(k)}\le|\Ze|_{(k+K_{\gamma})}$, $\forall
k=1,2,\ldots,r$. Therefore, 
\begin{align*}
\sum_{k=1}^r\left(1-\frac{|\Ze|_{(i[k])}^2}{\sigma_{\Ze}^2}\right)
=\sum_{k=1}^r\left(1-\frac{|\Zg|_{(k)}^2}{\sigma_{\Ze}^2}\right)
\le \sum_{k=1}^r\left(1-\frac{|\Ze|_{(k)}^2}{\sigma_{\Ze}^2}\right)
\end{align*}
and
\begin{align*}
&\sum_{k=1}^r\left(1-\frac{|\Zg|_{(k)}^2}{\sigma_{\Ze}^2}\right)
\geq\sum_{k=K_{\gamma}+1}^{r+K_{\gamma}+1}\left(1-\frac{|\Zg|_{(k)}^2}{\sigma_{\Ze}^2}\right)
\geq\sum_{k=1}^{K_{\gamma}}\frac{|\Ze|_{(k)}^2}{\sigma_{\Ze}^2}-K_{\gamma}
+\sum_{k=1}^{r+K_{\gamma}+1}\left(1-\frac{|\Ze|_{(k)}^2}{\sigma_{\Ze}^2}\right)\\
&\geq-K_{\gamma}+\sum_{k=1}^{r+K_{\gamma}+1}\left(1-\frac{|\Ze|_{(k)}^2}{\sigma_{\Ze}^2}\right).
\end{align*}
Thus, we know that 
\begin{align*}
\left|\sum_{k=1}^r\left(1-\left|\frac{\Ze_{(i[k])}}{\sigma_{\Ze}}\right|^2\right)\right|
\le\max\left\{\left|\sum_{k=1}^{r+K_{\gamma}+1}\left(1-\frac{|\Ze|_{(k)}^2}{\sigma_{\Ze}^2}\right)\right|,
\left|\sum_{k=1}^r\left(1-\frac{|\Ze|_{(k)}^2}{\sigma_{\Ze}^2}\right)\right|\right\}+K_{\gamma}.
\end{align*}
Now, since conditional on $\Ds$, $\frac{|\Ze|^2}{\sigma_{\Ze}^2}\sim\chi_1^2$, we
know that, applying Lemma~\ref{lem:chisq} 
\begin{align*}
\Pr\left\{\left|\frac{1}{r}\sum_{k=1}^r\left(1-\frac{|\Ze|_{(k)}^2}{\sigma_{\Ze}^2}\right)\right|\geq
  A_2\right\}
&=\Exp\left[\Pr\left\{\left|\frac{1}{r}\sum_{k=1}^r\left(1-\frac{|\Ze|_{(k)}^2}{\sigma_{\Ze}^2}\right)\right|\geq
  A_2\Big|\Ds\right\}\right]\\
  &\le 1-2e^{-\frac{\cvsigma^2}{2}}-\Nb^{-1},
\end{align*}
where 
\begin{align*}
A_2=Q(U_{(r+1)})+\frac{4\cvsigma}{\sqrt{r}}+\frac{\cvsigma^2\log\Nb}{3r}.
\end{align*}
Therefore, 
\begin{align*}
\left|\sum_{k=1}^r\left(1-\frac{|\Ze|_{(k)}^2}{\sigma_{\Ze}^2}\right)\right|
\le rQ(U_{(r+1)})+4\sqrt{r}\cvsigma+\frac{\cvsigma^2\log\Nb}{3}
\end{align*}
Similarly, we have that with probability at least
$1-2e^{-\frac{\cvsigma^2}{2}}-\Nb^{-1}$,
\begin{align*}
&\left|\sum_{k=1}^{r+K_{\gamma}+1}\left(1-\frac{|\Ze|_{(k)}^2}{\sigma_{\Ze}^2}\right)\right|
\le(r+K_{\gamma}+1)Q(U_{(r+K_{\gamma}+1)})+4\sqrt{r+K_{\gamma}}\cvsigma+\frac{\cvsigma^2\log\Nb}{3}\\
&\le rQ(U_{(r+1)})+4\sqrt{r+K_{\gamma}+1}\cvsigma+\frac{\cvsigma^2\log\Nb}{3}+(K_{\gamma}+1)\\
&\le rQ(U_{(r+1)})+4\sqrt{r}\cvsigma+\frac{\cvsigma^2\log\Nb}{3}+(K_{\gamma}+1)
+4\sqrt{K_{\gamma}+1}\cvsigma\\
&\le rQ(U_{(r+1)})+4\sqrt{r}\cvsigma+\frac{\cvsigma^2\log\Nb}{3}
+3(K_{\gamma}+1)+2\cvsigma^2,
\end{align*}
where the second inequality is because $Q(u)$ is decreasing and $Q(u)\le 1$.
Therefore,
\begin{align*}
\left|\sum_{k=1}^r\left(1-\left|\frac{\Ze_{(i[k])}}{\sigma_{\Ze}}\right|^2\right)\right|
\le rQ(U_{(r+1)})+4\sqrt{r}\cvsigma+\frac{\cvsigma^2\log\Nb}{3}
+4K_{\gamma}+3+2\cvsigma^2:=A_3.
\end{align*}
Hence,
\begin{align*}
&\frac{1}{r}\left|\sum_{k=1}^r\varepsilon_{[k]}\tx_{[k]}\tp\bu\right|
\le \frac{1}{r}\left|\sum_{k=1}^r\left\{\varepsilon_{[k]}\tx_{[k]}\tp\bu
-m(|\Ze|_{(i[k])})\right\}\right|+\frac{1}{r}\left|\sum_{k=1}^rm(|\Ze|_{i[k]})\right|\\
&\le\frac{1}{r}\left|\sum_{k=1}^r\left\{\varepsilon_{[k]}\tx_{[k]}\tp\bu
-m(|\Ze|_{(i[k])})\right\}\right|
+\frac{1}{r}\left|\sum_{k=1}^r\frac{\sigma_{\varepsilon}^2\bu\tp\deltaS(\sigma_{\Ze}^2-|\Ze_{(i[k])}|^2)}{\sqrt{\Ns}\sigma_{\Ze}^4}\right|\\
&\le\frac{1}{r}A_1+\frac{1}{r}A_3\\
&\le\frac{\sqrt{\cvsigma}}{\sqrt{r}}\left\{2\sigma_{\Ze}\sqrt{\log\Nb}
+\sqrt{1+\frac{1}{\sigma_{\varepsilon}^2}}\frac{\|\deltaS\|_2}{\sqrt{\Ns}}\right\}\\
&+\frac{\|\deltaS\|_2}{\sqrt{\Ns}}Q(U_{(r+1)})+\frac{4\cvsigma\|\deltaS\|_2}{\sqrt{r\Ns}}
+\frac{\cvsigma^2\log\Nb\|\deltaS\|_2}{3r\sqrt{\Ns}}
+\frac{(4K_{\gamma}+3+2\cvsigma^2)\|\deltaS\|_2}{r\sqrt{\Ns}}\\
&\lesssim \frac{\|\deltaS\|_2}{\sqrt{\Ns}}Q(U_{(r+1)})
+\frac{\sqrt{\cvsigma}\sqrt{\log\Nb}}{\sqrt{r}}
+\frac{\{\cvsigma^2\log\Nb\maxs K_{\gamma}\}\|\deltaS\|_2}{r\sqrt{\Ns}}.
\end{align*}
Remember that our bound above are derived under assumptions
$\Ns\geq\frac{\|\deltaS\|_2^2}{2\sigma_{\varepsilon}^2}$, $\gamma_b\geq
|\Ze|_{(\Nb)}-\Ze_{(1)}$ and $r<\Nb-K_{\gamma}$. We consider the following
``good'' set $\mathcal{A}$ that satisfies 
\begin{align*}
&\Ns\geq\frac{\|\deltaS\|_2^2}{2\sigma_{\varepsilon}^2}, \gamma_b\geq
|\Ze|_{(\Nb)}-\Ze_{(1)}, r<\Nb-K_{\gamma}\\
&\|\deltaS\|_2^2\le C\sigma_{\cS}^2\cvsigma,K_{\gamma}\le (1+\consp)p_{\gamma}\Nb,%
U_{(r+1)}\geq\frac{r}{2\Nb}.
\end{align*} 
Since $\|\deltaS\|_2^2\le C\sigma_{\cS}^2\cvsigma$ with probability at least
$1-\frac{C}{\sqrt{\Ns}}$, it is easy to prove $\Ns\geq 
\|\deltaS\|_2^2/2\sigma_{\varepsilon}^2$ with probability at least
$1-\frac{C}{\sqrt{\Ns}}$ by taking $\Ns\geq
C\sigma_{\cS}^2\cvsigma/2\sigma_{\varepsilon}^2$. Similarly, since~\eqref{eq:Kgamma1}
satisfies with probability at least $1-e^{-\frac{\consp^2 p_{\gamma}\Nb}{8}}$, we also know that
$r<\Nb-K_{\gamma}$ with probability at least $1-e^{-\frac{\consp^2 p_{\gamma}\Nb}{8}}$ by taking
$r\le\{1-(1+\consp)p_{\gamma}\}\Nb$. Since
$\Ze_i,i=1,2,\ldots,\Nb$ are normally distributed, it is easy to know that
$|\Ze|_{(\Nb)}-\Ze_{(1)}\le 4\sigma_{\Ze}\sqrt{\log\Nb}$ with
probability at least $1-2\Nb^{-1}$. Therefore, $\gamma_b\geq
|\Ze|_{(\Nb)}-\Ze_{(1)}$ with probability at least $1-2\Nb^{-1}$ if
$\gamma_b\gtrsim\sqrt{\log\Nb}$. Due to Lemma A.1.1
in~\cite{oliveira2025finite}, we also know that $\forall\cvsigma_1>0$, with
probability at least $1-e^{-2\cvsigma_1}$, we have that 
\begin{align*}
U_{(r+1)}\geq\frac{(\sqrt{r+1}-\sqrt{\cvsigma_1})^2}{\Nb}.
\end{align*}
Taking $\cvsigma_1=\frac{r+1}{(2\log\Nb)^2}$, we have that with probability at least
$1-e^{-(r+1)/\{2(\log\Nb)^2\}}$, %
\begin{equation*}
U_{(r+1)}\geq\left\{1+\frac{1}{4(\log\Nb)^2}-\frac{1}{\log\Nb}\right\}\frac{r+1}{r}\frac{r}{\Nb}
\geq\left(1-\frac{1}{\log\Nb}\right)\frac{r}{\Nb}.
\end{equation*}
Therefore, with probability at least
$1-\frac{C}{\sqrt{\Ns}}-\frac{3}{\Nb}-e^{-\frac{\consp^2p_{\gamma}\Nb}{8}}-e^{-(r+1)/\{2(\log\Nb)^2\}}-2e^{-\frac{\cvsigma^2}{2}}-\cvsigma^{-1}$,
we have that 
\begin{align*}
\frac{1}{r}\left|\sum_{k=1}^r\varepsilon_{[k]}\tx_{[k]}\tp\bu\right|
\lesssim\frac{\sigma_{\cS}\sqrt{\cvsigma}}{\sqrt{\Ns}}Q\left\{\left(1-\frac{1}{\log\Nb}\right)\frac{r}{\Nb}\right\}
+\frac{\sqrt{\cvsigma}\{(\cvsigma^2\log\Nb)\maxs (p_{\gamma}\Nb)\}}{\sqrt{r}}.
\end{align*} 
\end{proof}

Now, we prove Theorem~\ref{thm:dtrans} based on Lemma~\ref{lem:nobias2},
Lemma~\ref{lem:concox}, and Lemma~\ref{lem:deteps}.

\begin{proof}[\textbf{Proof of Theorem~\ref{thm:dtrans} (I)}]

Remember that for target guided method the weighting matrix $\W_{\cBs}=\I_{\Nb}$
and therefore, we know that according to~\eqref{eq:fpgamma}, we know that
$\hgamma_{\cBs}$ is the fix point of equation:
\begin{align}\label{eq:fpgamma2}
\hgamma_{\cBs}
=\Theta\left\{(\I-\Xbs\V^{-1}\Xbs\tp)\ybs
+\Xbs\V^{-1}\Xbs\tp\hgamma_{\cBs}-\Xbs\V^{-1}\Xs\tp\ys;\lambda\right\},
\end{align}
and $\hbeta=\V^{-1}\Xs\tp\ys+\V^{-1}\Xbs\tp(\ybs-\hgamma_{\cBs})$, where
$\cBs=\cBd$. Since 
\begin{align*}
&\Xbs\V^{-1}\Xs\tp\ys=\Xbs\V^{-1}\Xs\tp\Xs\hbeta_{\cS}\\
&=\Xbs\V^{-1}(\V-\Xbs\tp\Xbs)\hbeta_{\cS}
=(\I-\Xbs\V^{-1}\Xbs\tp)\Xbs\hbeta_{\cS},
\end{align*}
we have that 
\begin{equation*}
\hgamma_{\cBs}
=\Theta_1\left\{(\I-\Xbs\V^{-1}\Xbs\tp)(\ybs-\Xbs\hbeta_{\cS})
+\Xbs\V^{-1}\Xbs\tp\hgamma_{\cBs};\lambda\right\}.
\end{equation*}
Now, we prove that the all the eigen values of matrix $\Xbs\V^{-1}\Xbs\tp$ are in
the range of $[0,1]$. It is easy to observe that $\Xbs\V^{-1}\Xbs\tp$ is
non-negative and thus we only prove that the eigen values of
$\Xbs\V^{-1}\Xbs\tp$ are less than 1. We know that, by definition, the eigen
values of $\Xbs\V^{-1}\Xbs\tp$ are the roots of 
\begin{align*}
&det\left\{\lambda\I_{\Nb}-\Xbs\V^{-1}\Xbs\tp\right\}
=\lambda^{\Nb-d}det\left\{\lambda\I_d-\V^{-1}\Xbs\tp\Xbs\right\}\\
&=\lambda^{\Nb-d}(\lambda-1)^{d-\Ns}det\left\{(\lambda-1)\I_{\Ns}+\Xs\V^{-1}\Xs\tp\right\},
\end{align*}
which is strictly positive if $\lambda>0$ and therefore, all the eigen values of
matrix $\Xbs\V^{-1}\Xbs\tp$ are in the range of $[0,1]$. This also implies that
all the eigen values of matrix $\I-\Xbs\V^{-1}\Xbs\tp$ are in the range of
$[0,1]$ and thus $\|\I-\Xbs\V^{-1}\Xbs\tp\|_2\le 1$. Since
\begin{align*}
&\|(\I-\Xbs\V^{-1}\Xbs\tp)(\ybs-\Xbs\hbeta_{\cS})\|_{\infty}
\le\|\I-\Xbs\V^{-1}\Xbs\tp\|_{\infty}\|\ybs-\Xbs\hbeta_{\cS}\|_{\infty}\\
&\le\sqrt{d}\|\I-\Xbs\V^{-1}\Xbs\tp\|_2\|\ybs-\Xbs\hbeta_{\cS}\|_{\infty}
\le\sqrt{d}\|\ybs-\Xbs\hbeta_{\cS}\|_{\infty}=\sqrt{d}|\Zg|_{(r)}.
\end{align*}
Applying Lemma~\ref{lem:nobias2} and Corollary~\ref{cor:orderstat}, we know that
on $\cBd$, $\gamma_{[1]}=\gamma_{[2]}=\cdots=\gamma_{[r]}=0$ and
$|\Zg|_{(r)}\lesssim\sqrt{\log\Nb}$ with probability
at least
$1-Crp_{\gamma}e^{\frac{-\gamma_b^2}{4\sigma_{\varepsilon}^2}}-Ce^{-\{(p_{\gamma}\Nb)\mins\log\Nb\}}$,
Due to the definition of $\Theta_1(\cdot)$, we know that if
$\lambda>\frac{|\Zg|_{(r)}}{\sqrt{d}}$, we have that $\hgamma_{\cBs}=\0$ is the fix
point of~\eqref{eq:fpgamma2}. Therefore, since $\lambda\gtrsim\sqrt{\log\Nb}$,
we know that $\lambda\gtrsim\frac{|\Zg|_{(r)}}{\sqrt{d}}$ with probability at
least
$1-Crp_{\gamma}e^{\frac{-\gamma_b^2}{4\sigma_{\varepsilon}^2}}-Ce^{-\{(p_{\gamma}\Nb)\mins\log\Nb\}}$
and 
\begin{align*}
&\hbeta_{\dtm}=\V^{-1}\Xs\tp\ys+\V^{-1}\Xbs\tp(\ybs-\hgamma_{\cBs})\\
&=\left(\sumcs\x_t\x_t\tp+\sum_{k=1}^r\x_{[k]}\x_{[k]}\tp\right)^{-1}
\left(\sumcs\x_t\x_t\tp\tbeta+\sumcs\varepsilon_t\x_t
+\sum_{k=1}^r\x_{[k]}\x_{[k]}\tp\tbeta+\sum_{k=1}^r\varepsilon_{[k]}\x_{[k]}\right),
\end{align*}
which implies 
\begin{align*}
\bSigma^{\frac{1}{2}}\hbeta-\bSigma^{\frac{1}{2}}\tbeta
=\left(\sumcs\tx_t\tx_t\tp+\sum_{k=1}^r\tx_{[k]}\tx_{[k]}\tp\right)^{-1}
\left(\sumcs\varepsilon_t\tx_t+\sum_{k=1}^r\varepsilon_{[k]}\tx_{[k]}\right).
\end{align*}
Due to Lemma~\ref{lem:lemofS} and Lemma~\ref{lem:concox}, we know that with
probability at least
$1-Crp_{\gamma}e^{-\frac{\gamma_b^2}{4\sigma_{\varepsilon}^2}}
-Ce^{-\{(p_{\gamma}\Nb)\mins\log\Nb\}}-2e^{-Cr^{1-\alpha_3}}-\frac{C}{\sqrt{\Ns}}$,
we have that 
\begin{equation*}
\V=\frac{1}{N}\left(\sumcs\tx_t\tx_t\tp
+\sum_{i=1}^r\tx_{[i]}\tx_{[i]}\tp\right)\geq\frac{1}{4},
\end{equation*}
where $N=\Ns+\Nb$. Now, combining Lemma~\ref{lem:deteps}, we complete the proof.

\end{proof}

\begin{proof}[\textbf{Proof of Theorem~\ref{thm:dtrans} (II)}]

Due to Lemma~\ref{lem:nobias2}, we have that with probability
at least $1-Crp_{\gamma}e^{\frac{-\gamma_b^2}{4\sigma_{\varepsilon}^2}}-Ce^{-\{(p_{\gamma}\Nb)\mins\log\Nb\}}$,
$\gamma_{[1]}=\gamma_{[2]}=\cdots=\gamma_{[r]}=0$.
From the proof of Theorem~\ref{thm:rtrans2} (see Section~\ref{sec:prf-rtrans2}),
we can know that for $\hbetad$, we have that
\begin{align*}
\bSigma^{\frac{1}{2}}\hbetad-\bSigma^{\frac{1}{2}}\tbeta%
=\left(\sumcs\tx_t\tx_t\tp
+\frac{\lambda}{1+\lambda}\sum_{i=1}^r\tx_{[i]}\tx_{[i]}\tp\right)^{-1}
\left(\sumcs\varepsilon_t\tx_t
+\frac{\lambda}{1+\lambda}\sum_{i=1}^r\varepsilon_{[i]}\tx_{[i]}\right).
\end{align*}
Due to Lemma~\ref{lem:lemofS} and Lemma~\ref{lem:concox}, we know that with
probability at least
$1-Crp_{\gamma}e^{-\frac{\gamma_b^2}{4\sigma_{\varepsilon}^2}}
-Ce^{-\{(p_{\gamma}\Nb)\mins\log\Nb\}}-2e^{-Cr^{1-\alpha_3}}-\frac{C}{\sqrt{\Ns}}$,
we have that
\begin{equation*}
\V_{\lambda}=\frac{1}{N_{f\lambda}}\left(\sumcs\tx_t\tx_t\tp
+\frac{\lambda}{1+\lambda}\sum_{i=1}^r\tx_{[i]}\tx_{[i]}\tp\right)\geq\frac{1}{4},
\end{equation*}
where $N_{f\lambda}=\Ns+c_{\lambda}\}r$ and
$c_{\lambda}=\lambda/(1+\lambda)$. Now, combining Lemma~\ref{lem:deteps}, we
complete the proof.

\end{proof}

\begin{proof}[\textbf{Proof of Corollary~\ref{cor:orderimprove}}]
Due to Proposition 4.1.1 in \cite{oliveira2025finite}, and the definition in of
$Q(b)$ in the proof of Lemma~\ref{lem:chisq}, we have that $\forall b>0$
\begin{align*}
  Q(b)=1-\mu^{(0,b)}=|\mu^{(0,b)}-\mu|
  &=-\frac{\int_0^1\{F^{-1}(u)-\mu\}I(u<b)\dd u}{1-b}\\
  &\le\Exp(|X_1-\mu|^q)\frac{(1-b)}{b}(1-b)^{-\frac{1}{q}}.
\end{align*}
Therefore, the result holds by plugging in $b=(1-\frac{1}{\log\Nb})\frac{r}{\Nb}$, which implies, 
\begin{align*}
Q(b)\lesssim\frac{1-\left(1-\frac{1}{\log\Nb}\right)\frac{r}{\Nb}}{
\left(1-\frac{1}{\log\Nb}\right)\frac{r}{\Nb}}\left\{1-\left(1-\frac{1}{\log\Nb}\right)\frac{r}{\Nb}\right\}^{-\frac{1}{q}}
=\left(\Nb-r+\frac{r}{\log\Nb}\right)^{1-\frac{1}{q}}\frac{\Nb^{\frac{1}{q}}}{r-\frac{r}{\log\Nb}}.
\end{align*}

\end{proof}

\subsection{Calculation of degree of freedom in Section~\ref{sec:compissue}}

For $\ell_2$ penalty, denoting $\V_{\cS}=\X_{\cS}\tp\X_{\cS}$,
$\V_{\cBs}=\X_{\cBs}\tp\W_{\cBs}\X_{\cBs}$, $c_{\lambda}=\lambda/(1+\lambda)$,
$\V_{\lambda}=\V_{\cS}+c_{\lambda}\V_{\cBs}$,  we have that
$\hbeta=\V_{\lambda}^{-1}\left(\Xs\tp\ys+c_{\lambda}\Xbs\tp\W_{\cBs}\ybs\right)$ and
$\hgamma_{\cBs}=(1-c_{\lambda})(\ybs-\Xbs\hbeta)$.  
Therefore, We have that
\begin{align*}
&\begin{pmatrix}
\hat{\bm{y}}_{\cS}\\
\hat{\bm{y}}_{\cBs}\\
\end{pmatrix}
=\bm{M} 
\begin{pmatrix}
\ys\\
\ybs\\
\end{pmatrix},\\
&\text{ where }
\bm{M}
=
\begin{pmatrix}
\Xs\V_{\lambda}^{-1}\Xs\tp
& c_{\lambda}\Xs\V_{\lambda}^{-1}\Xbs\tp\W_{\cBs}\\
c_{\lambda}\W_{\cBs}\Xb\V_{\lambda}^{-1}\Xs\tp
& c_{\lambda}\Xbs\V_{\lambda}^{-1}c_{\lambda}\Xbs\tp\W_{\cBs}+(1-c_{\lambda})\I_{\Nbs}
\end{pmatrix}.
\end{align*}
There for the degree of freedom is defined as 
\begin{align*}
&\dfr%
=\text{tr}\left(\V_{\lambda}^{-1}\bm{V}_{\cS}\right)
+c_{\lambda}\text{tr}\left\{\V_{\lambda}^{-1}
c_{\lambda}\bm{V}_{\cBs}\right\}+(1-c_{\lambda})\text{tr}(\I_{\Nbs})\\
&=(1-c_{\lambda})\Nbs+c_{\lambda}d
+(1-c_{\lambda})\text{tr}\left(\V_{\lambda}^{-1}\V_{\cS}\right)
\end{align*}
Since the trace of $\V_{\lambda}^{-1}\V_{\cS}$ equals the trace of
$\left(\I_d+c_{\lambda}\V_{\cS}^{-\frac{1}{2}}\V_{\cBs}\V_{\cS}^{-\frac{1}{2}}\right)^{-1}$,
denoting the eigen values of
$\bm{V}_{\cS}^{\frac{1}{2}}\bm{V}_{\cBs}^{-1}\bm{V}_{\cS}^{\frac{1}{2}}$ as $q_i,i=1,...,d$, we
have  
\begin{align*}
\dfr=\text{tr}(\bm{M})
&=(1-c_{\lambda})\Nb+c_{\lambda}d
+(1-c_{\lambda})\sum_{i=1}^d\frac{1}{1+c_{\lambda}\frac{1}{q_i}}\\
&=\frac{\Nb}{1+\lambda}+\frac{\lambda}{1+\lambda}d
+\sum_{i=1}^d\frac{q_i}{\lambda+(1+\lambda)q_i}.
\end{align*}

\subsection{Additional simulation results}\label{sec:addsimu}

In this section, we present additional simulation results mentioned in Section~\ref{sec:simu}.
We first present additional results for simulated data using $\ell_1$ penalization.
The eBias and eVar of
Case ``SP'' and Case ``HT'' are presented in the following Figure~\ref{fig:bias-var-apdx}.
\begin{figure}[H]
  \centering
  \begin{subfigure}{0.205\textwidth}
    \includegraphics[width=\textwidth]{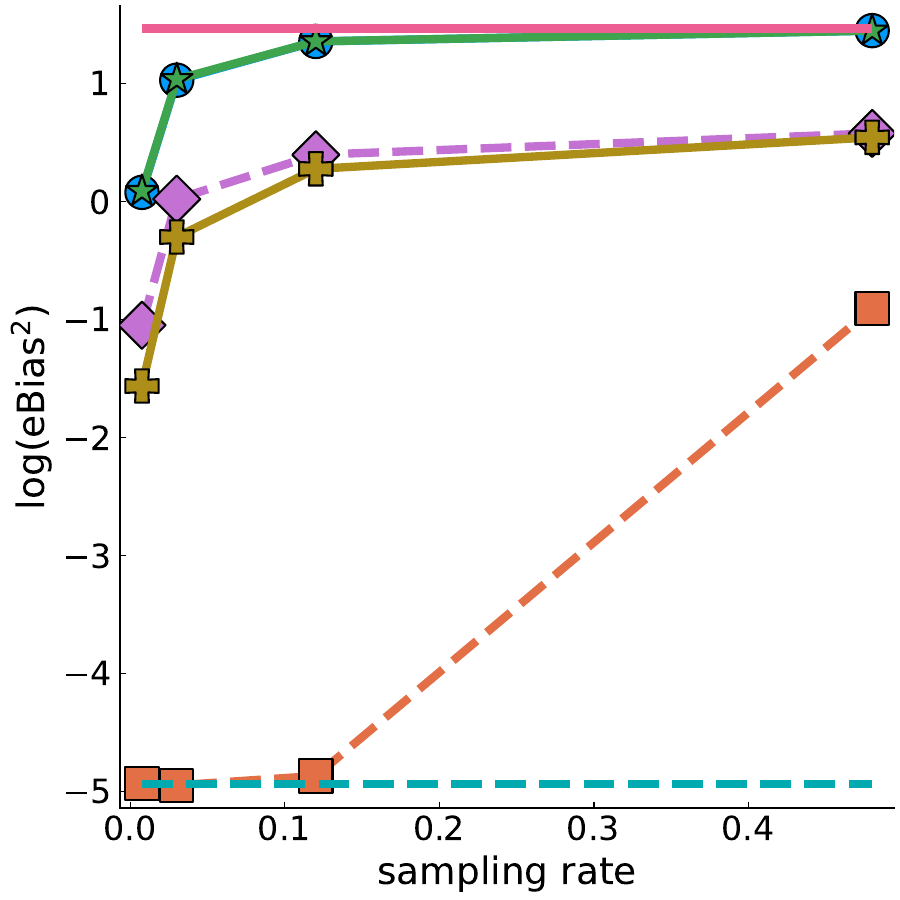}
    \caption*{eBias$^2$ of Figure~\ref{fig:mseL1}(\subref{sfig:a})}
  \end{subfigure}
  \begin{subfigure}{0.205\textwidth}
    \includegraphics[width=\textwidth]{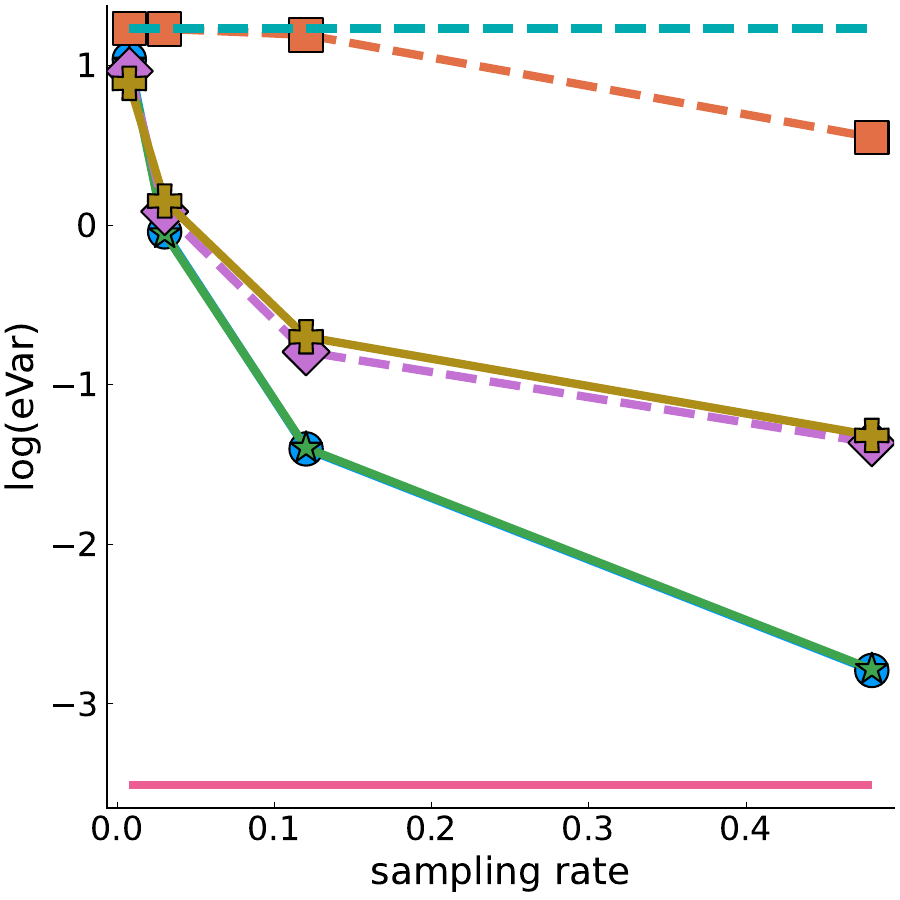}
    \caption*{eVar of Figure~\ref{fig:mseL1}(\subref{sfig:a})}
  \end{subfigure}
  \begin{subfigure}{0.215\textwidth}
    \includegraphics[width=\textwidth]{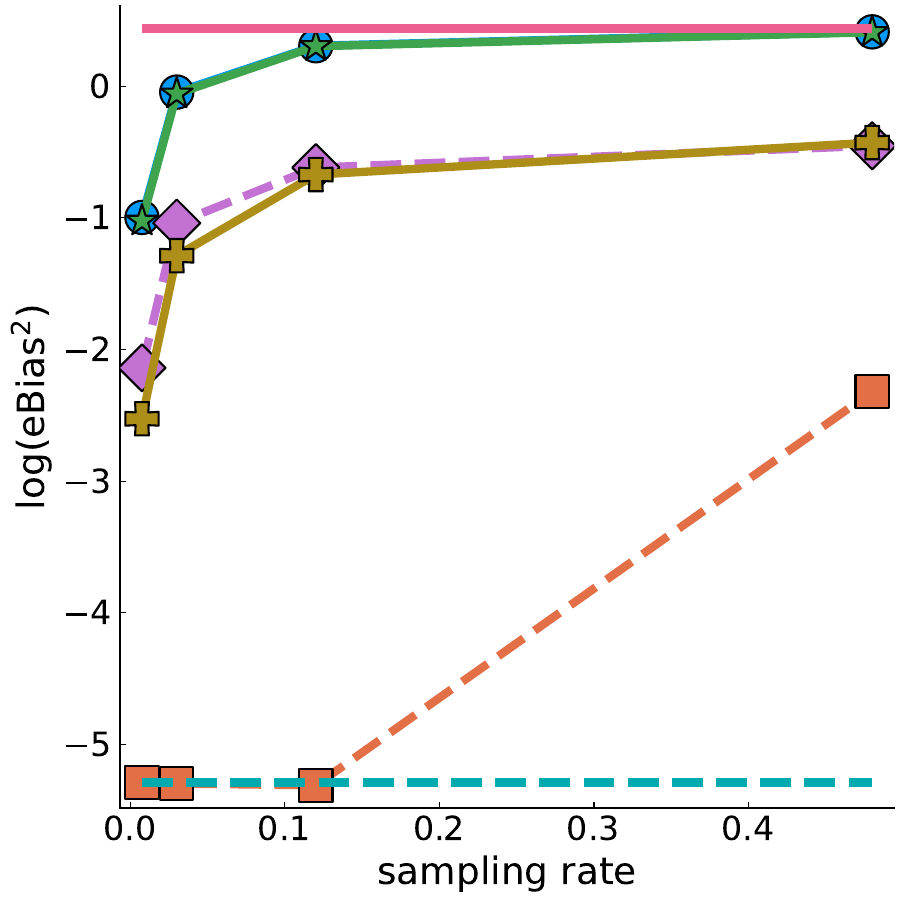}
    \caption*{eBias$^2$  of Figure~\ref{fig:mseL1}(\subref{sfig:b})}
  \end{subfigure}
  \begin{subfigure}{0.215\textwidth}
    \includegraphics[width=\textwidth]{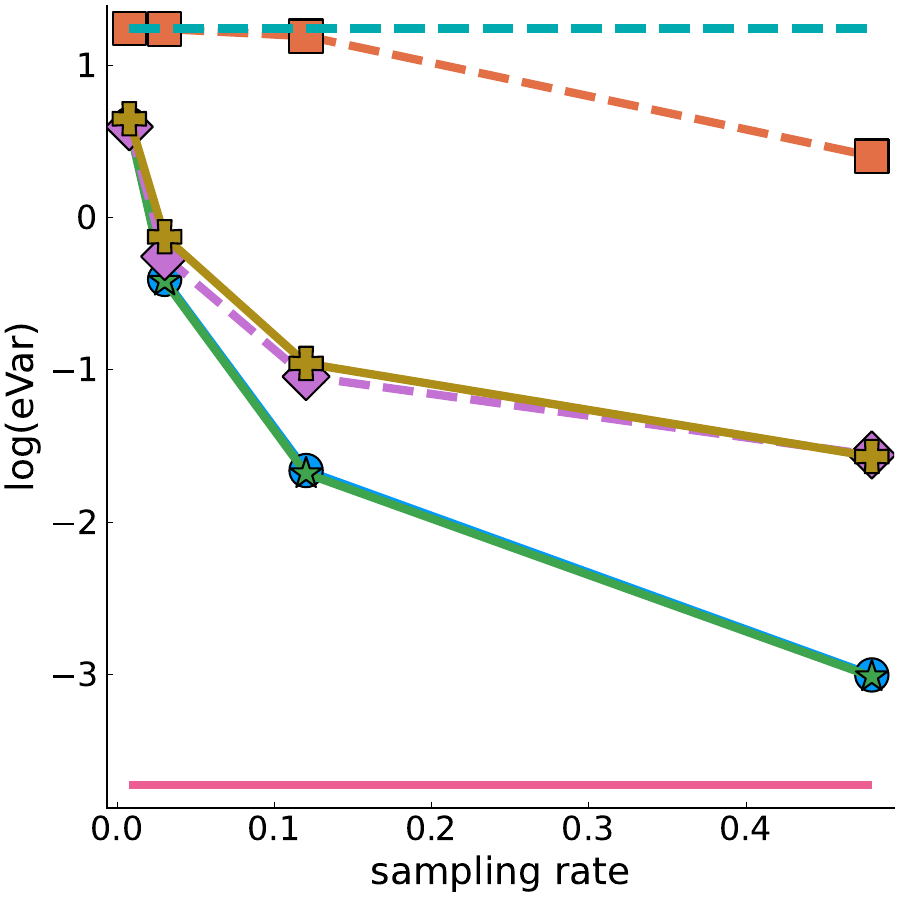}
    \caption*{eVar of Figure~\ref{fig:mseL1}(\subref{sfig:b})}
  \end{subfigure}
  \begin{subfigure}{0.10\textwidth}
    \includegraphics[width=\textwidth]{leg3.pdf}
  \end{subfigure}\\
  \begin{subfigure}{0.215\textwidth}
    \includegraphics[width=\textwidth]{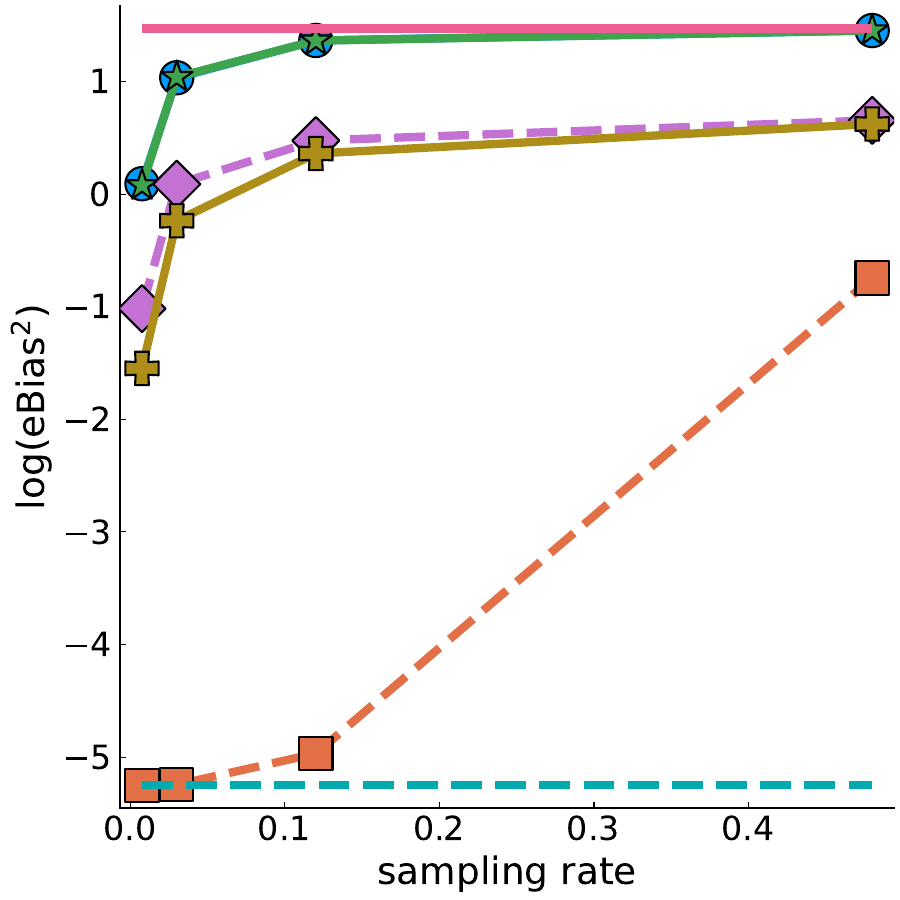}
    \caption*{eBias$^2$  of Figure~\ref{fig:mseL1}(\subref{sfig:d})}
  \end{subfigure}
    \begin{subfigure}{0.205\textwidth}
    \includegraphics[width=\textwidth]{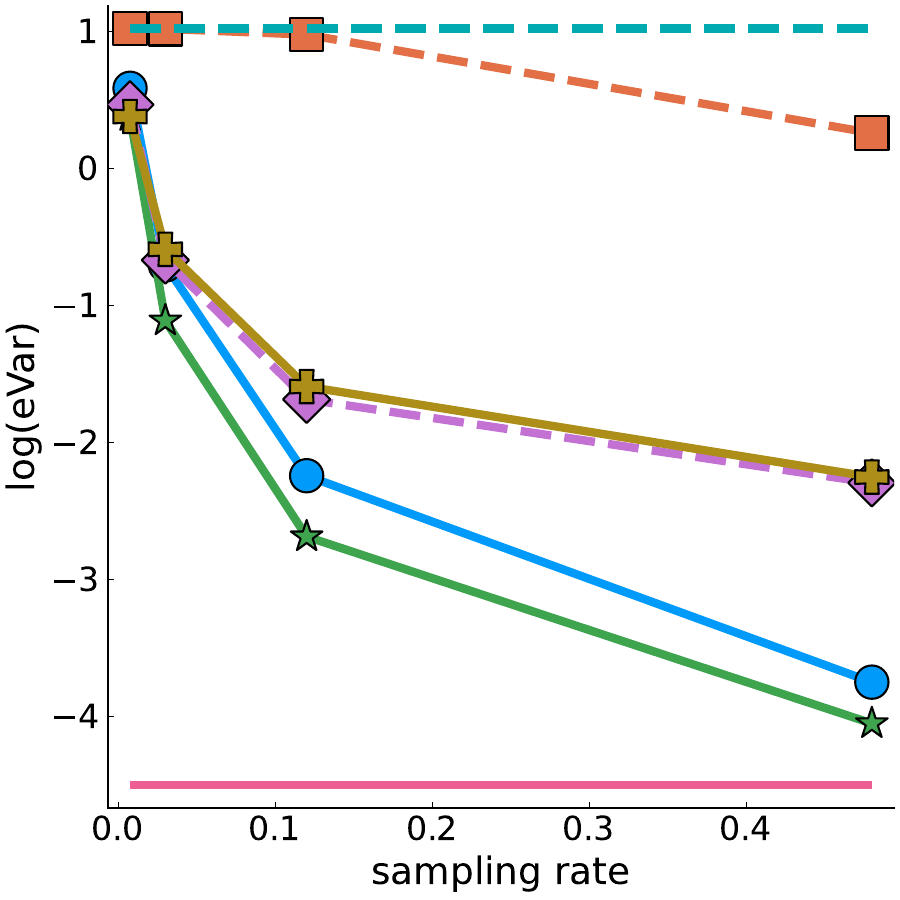}
    \caption*{eVar of Figure~\ref{fig:mseL1}(\subref{sfig:d})}
  \end{subfigure}
  \begin{subfigure}{0.215\textwidth}
    \includegraphics[width=\textwidth]{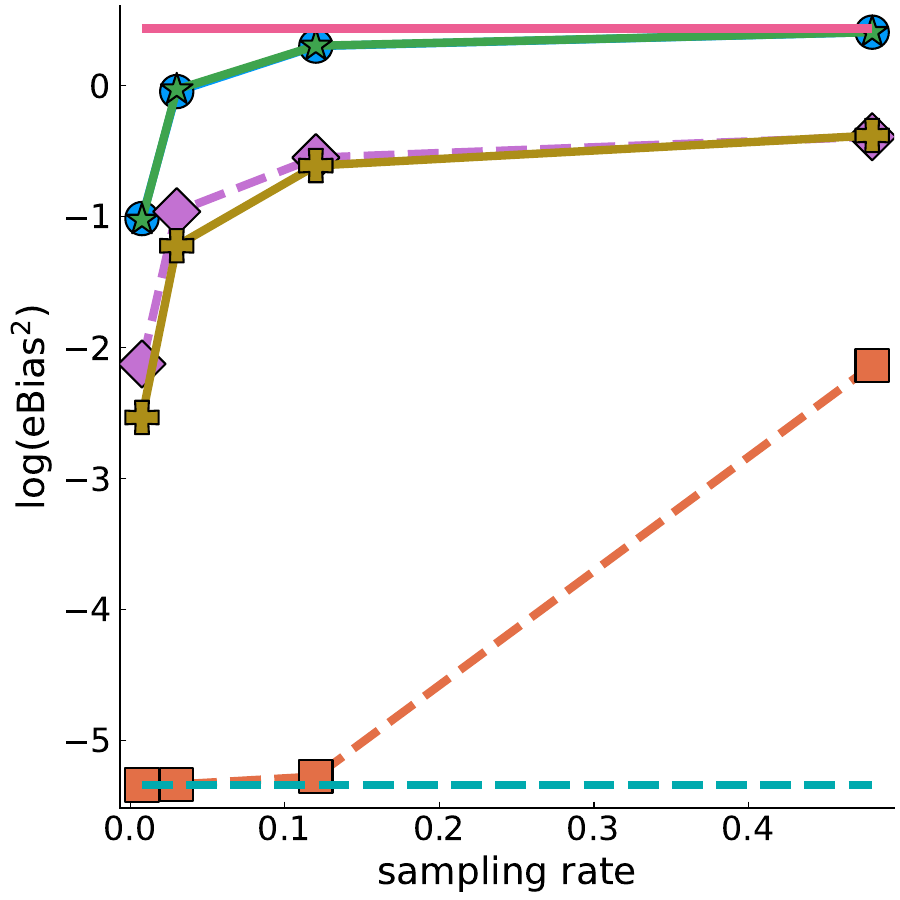}
    \caption*{eBias of Figure~\ref{fig:mseL1}(\subref{sfig:e})}
  \end{subfigure}
    \begin{subfigure}{0.215\textwidth}
    \includegraphics[width=\textwidth]{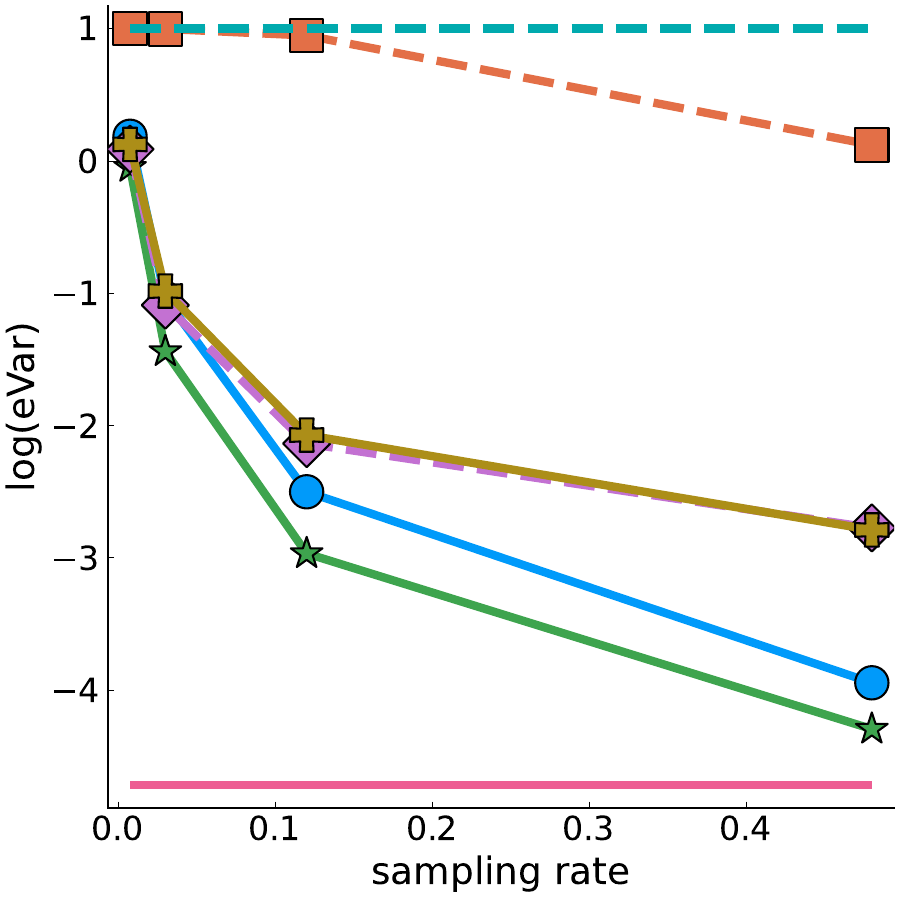}
    \caption*{eVar of Figure~\ref{fig:mseL1}(\subref{sfig:e})}
  \end{subfigure}
  \begin{subfigure}{0.10\textwidth}
    \includegraphics[width=\textwidth]{leg3.pdf}
  \end{subfigure}
  \caption{The eBias$^2$  and eVar of transfer learning estimators with $\ell_1$
    penalty for Case ``SP'' and Case ``HT''.}
  \label{fig:bias-var-apdx}
\end{figure}
Now, we present additional simulation results using $\ell_2$ penalty. The eMSEs are
presented in the following Figure~\ref{fig:mseL2}.
\begin{figure}[H]
  \centering 
  \begin{subfigure}{0.255\textwidth}
    \includegraphics[width=\textwidth]{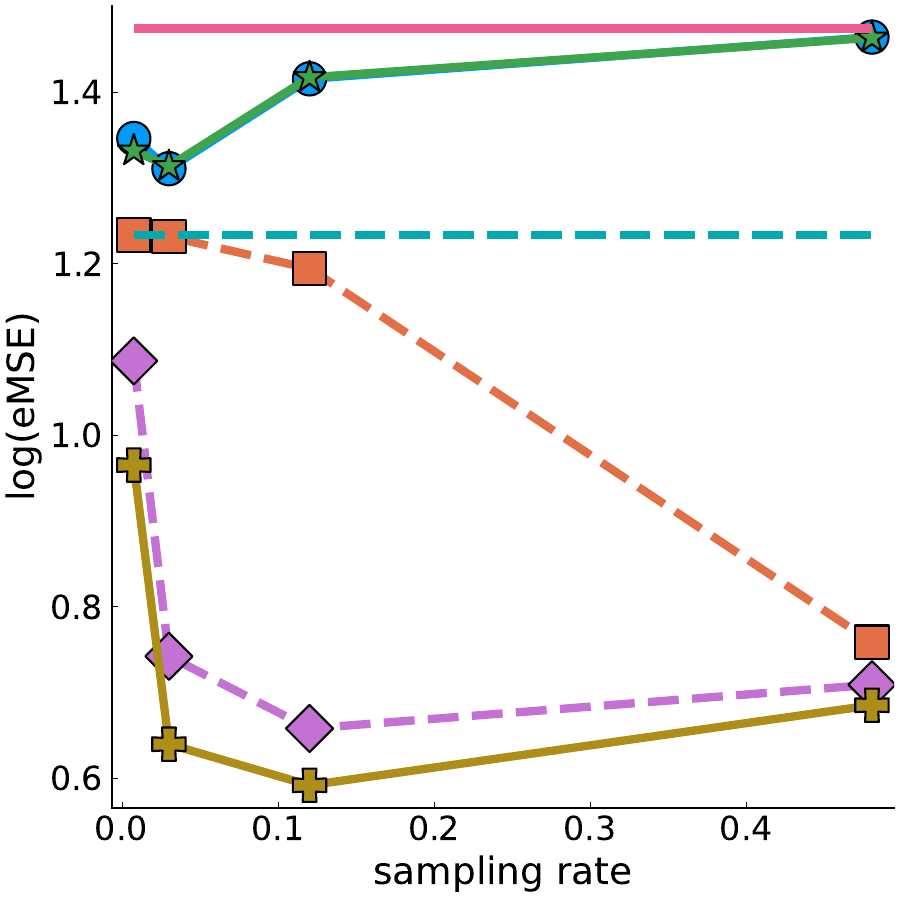}
    \caption{SP (light-tail $\x$)}
  \end{subfigure}
  \begin{subfigure}{0.255\textwidth}
    \includegraphics[width=\textwidth]{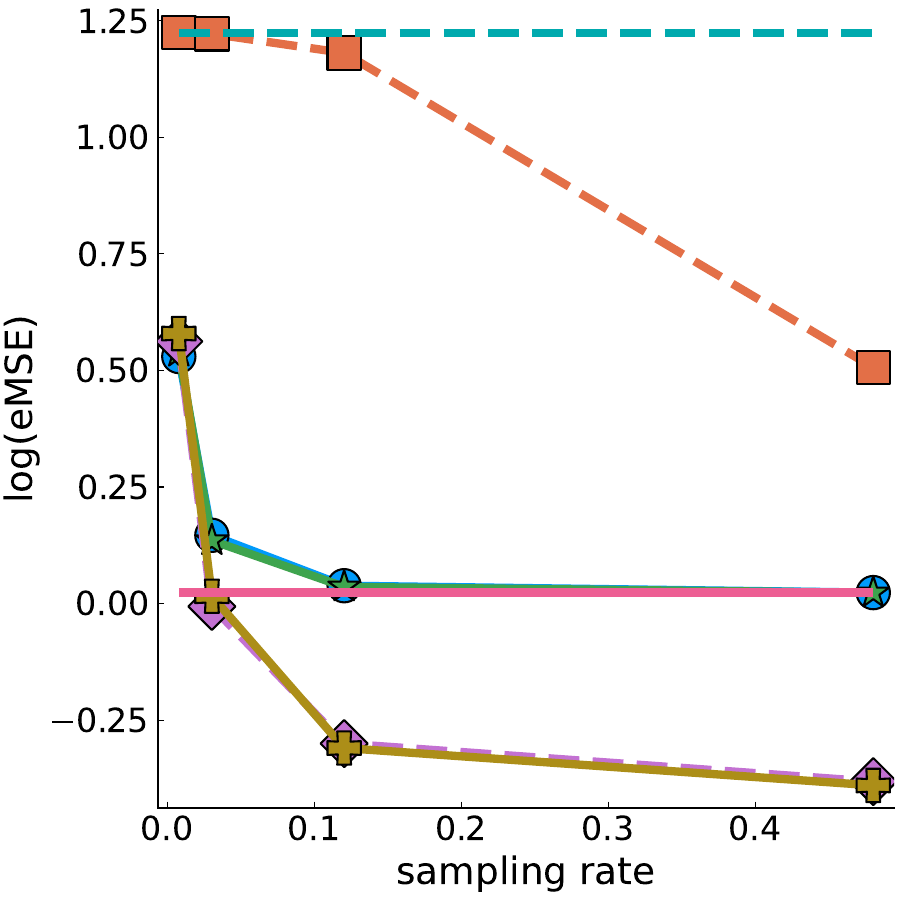}
    \caption{HL (light-tail $\x$)}
  \end{subfigure}
  \begin{subfigure}{0.255\textwidth}
    \includegraphics[width=\textwidth]{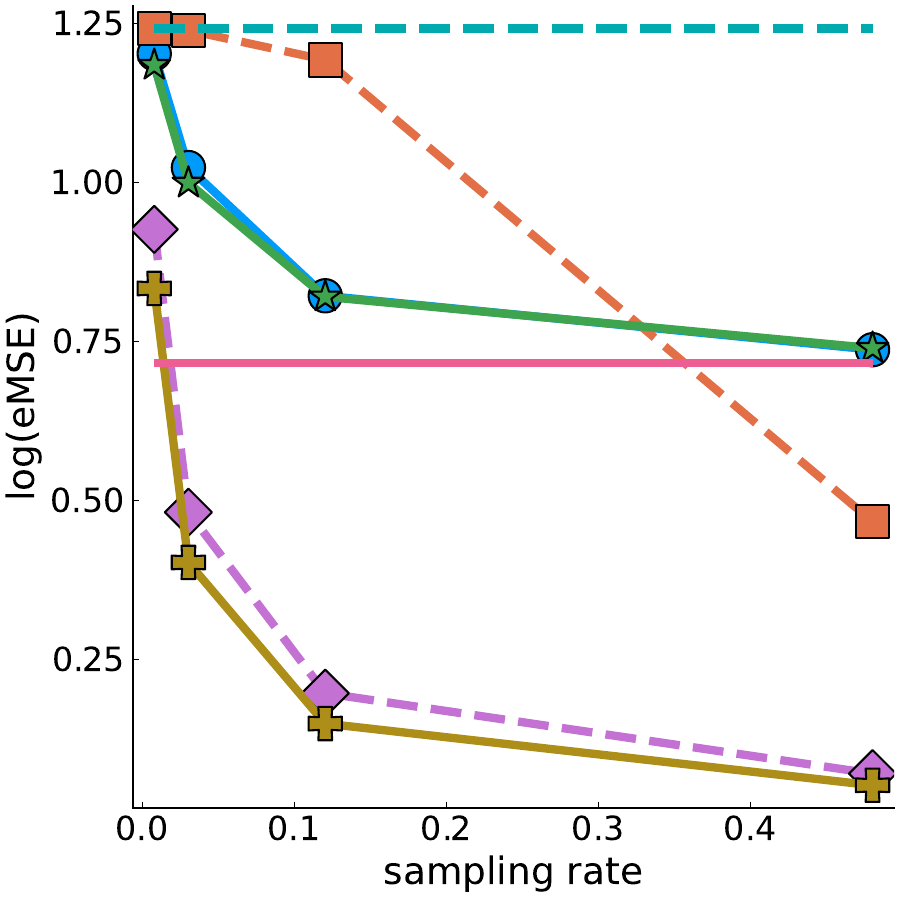}
    \caption{HT (light-tail $\x$)}
  \end{subfigure}
  \begin{subfigure}{0.1\textwidth}
    \includegraphics[width=\textwidth]{leg3.pdf}
  \end{subfigure}
  \begin{subfigure}{0.255\textwidth}
    \includegraphics[width=\textwidth]{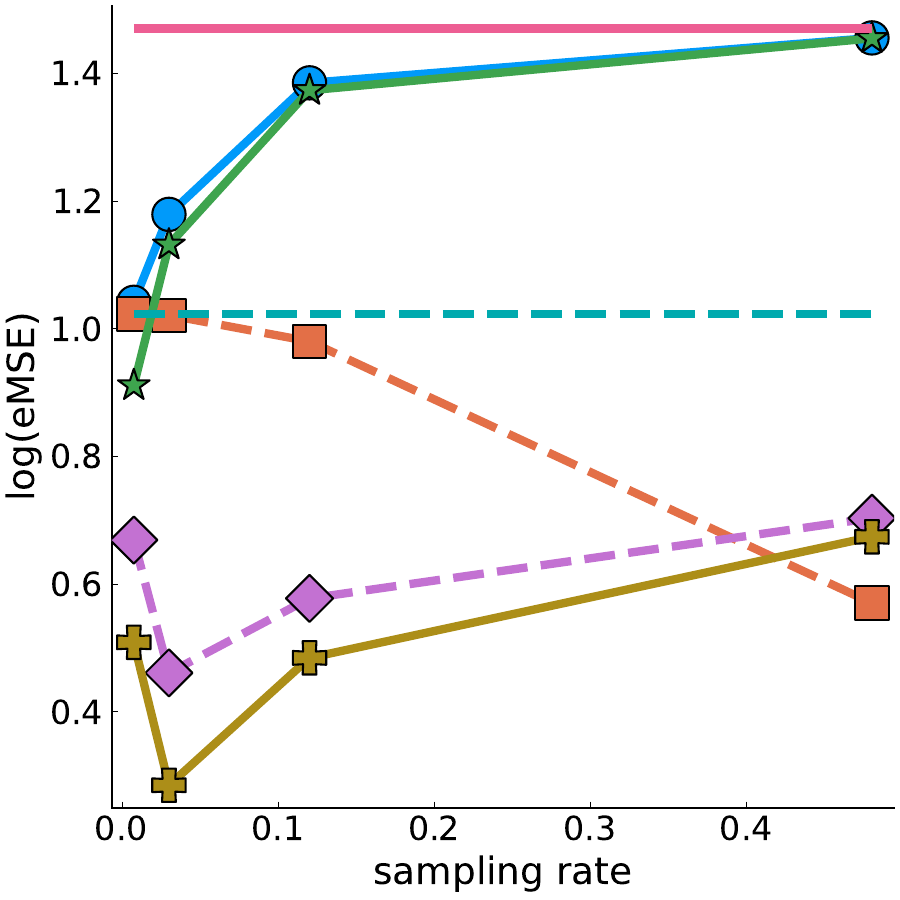}
    \caption{SP (heavy-tail $\x$)}
  \end{subfigure}
  \begin{subfigure}{0.255\textwidth}
    \includegraphics[width=\textwidth]{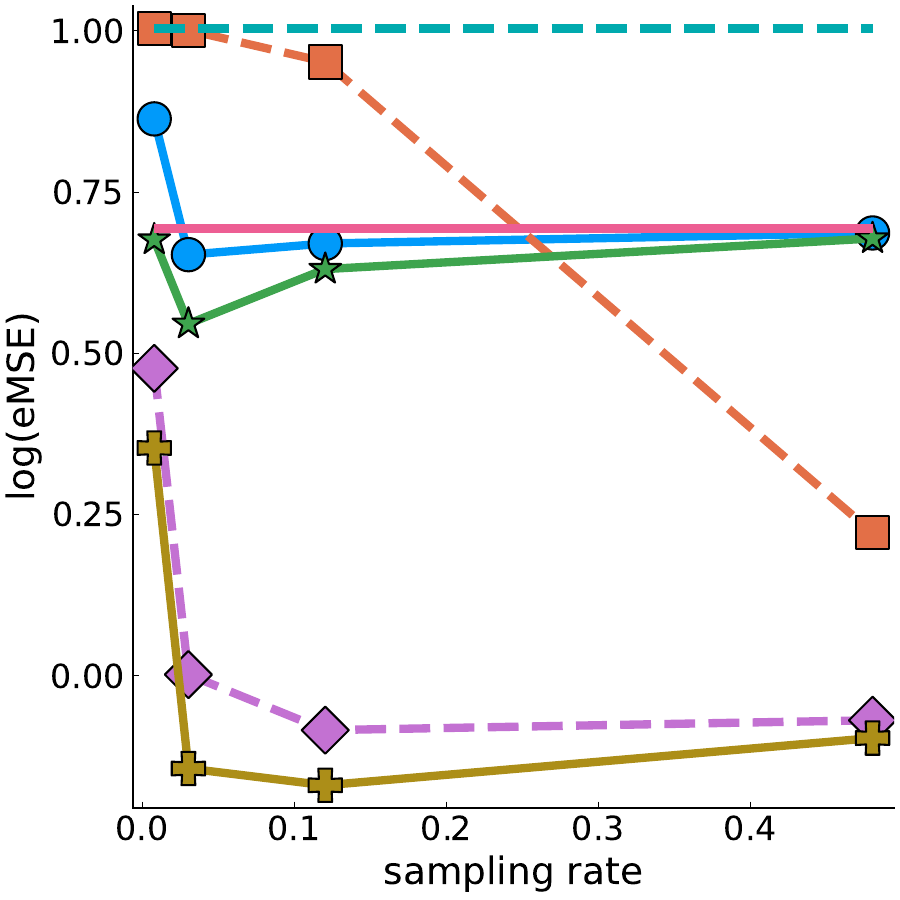}
    \caption{HT (heavy-tail $\x$)}
  \end{subfigure}
  \begin{subfigure}{0.1\textwidth}
    \includegraphics[width=\textwidth]{leg3.pdf}
  \end{subfigure}
  \caption{The eMSE of transfer learning estimators with $\ell_2$ penalty. The first row
    is for light-tailed $\x$ and the second row is for heavy-tailed $\x$.
    The dashed line represents the results of least
    square estimator ($\hbeta_{\cS}$) with only target data $\cS$. The solid line
    represents the results
    of the transfer learning estimator ($\hbeta_{\mathrm{full}}$) with full data $\cS\cup\cB$.}
  \label{fig:mseL2}
\end{figure}
Next, we present some additional simulation results of real data. As
we mentioned in Section~\ref{sec:realdata}, we also try combined estimators with
different combining proportions. Here, we present the results with
$c=0.5,0.7,0.9$.
\begin{table}[H]
\caption{Percentages improved on eMSPE of transfer learning estimators over
$\hbeta_{\cS}$.}
\label{tb:real-apx}
\centering
\begin{tabular}{l *{6}{w{c}{12mm}} }\hline
       & \multicolumn{6}{c}{50\%-50\% split of target data}\\
       & \multicolumn{2}{c}{$c=0.5$} & \multicolumn{2}{c}{$c=0.7$} & \multicolumn{2}{c}{$c=0.9$}\\\hline
  $r$ & $\hbeta_{\comdat}$ & $\hbeta_{\comest}$
  & $\hbeta_{\comdat}$ & $\hbeta_{\comest}$ & $\hbeta_{\comdat}$
  & $\hbeta_{\comest}$\\\hline
  $\Ns$  & $+1.01$ & $+1.05$ & $+0.60$ & $+0.69$ & $+0.19$ & $+0.24$ \\
  $2\Ns$ & $+1.00$ & $+1.29$ & $+0.60$ & $+0.96$ & $+0.19$ & $+0.39$ \\
  $4\Ns$ & $+0.82$ & $+1.47$ & $+0.47$ & $+1.22$ & $+0.14$ & $+0.58$ \\
  $8\Ns$ & $+0.56$ & $+1.17$ & $+0.30$ & $+1.23$ & $+0.09$ & $+0.70$ \\\hline
\end{tabular}
\begin{tabular}{l *{6}{w{c}{12mm}} }\hline
       & \multicolumn{6}{c}{80\%-20\% split of target data}\\
       & \multicolumn{2}{c}{$c=0.5$} & \multicolumn{2}{c}{$c=0.7$} & \multicolumn{2}{c}{$c=0.9$}\\\hline
  $r$ & $\hbeta_{\comdat}$ & $\hbeta_{\comest}$
  & $\hbeta_{\comdat}$ & $\hbeta_{\comest}$ & $\hbeta_{\comdat}$
  & $\hbeta_{\comest}$\\\hline
  $\Ns$  & $+0.69$ & $+0.69$ & $+0.46$ & $+0.48$ & $+0.16$ & $+0.19$ \\
  $2\Ns$ & $+0.65$ & $+0.88$ & $+0.41$ & $+0.65$ & $+0.14$ & $+0.28$ \\
  $4\Ns$ & $+0.53$ & $+0.83$ & $+0.32$ & $+0.74$ & $+0.11$ & $+0.33$ \\
  $8\Ns$ & $+0.37$ & $+0.44$ & $+0.22$ & $+0.69$ & $+0.07$ & $+0.43$ \\\hline
\end{tabular}
\end{table}

\bibliographystyle{natbib}
\bibliography{references}
\end{document}